\documentclass{article}

\usepackage[utf8]{inputenc} 
\usepackage[T1]{fontenc}    
\usepackage{url}            
\usepackage{booktabs}       
\usepackage{amsfonts}       
\usepackage{nicefrac}       
\usepackage{microtype}      
\usepackage{xcolor}         

\usepackage{graphicx}
\usepackage{amsmath,amssymb,amsfonts, bm, dsfont}
\usepackage{xcolor}

\DeclareMathAlphabet      {\mathbfit}{OML}{cmm}{b}{it}
\usepackage{amsthm} 

\usepackage{acronym} 

\usepackage[colorlinks]{hyperref}
\hypersetup{
    colorlinks=False,   
}

\usepackage{tikz}
\usetikzlibrary{math}
\usepackage{pgfplots}
\usepackage{subcaption} 

\usepackage{tkz-tab}
\usepackage{pgfplots}

\usepackage{tabularx} 
\usepackage{microtype}
\usepackage{graphicx}
\usepackage{subcaption}
\usepackage{booktabs} 

\usepackage{hyperref}

\usepackage[preprint]{icml2026}

\usepackage{natbib}
\usepackage{amsmath}
\usepackage{amssymb}
\usepackage{mathtools}
\usepackage{amsthm}

\usepackage[capitalize,noabbrev]{cleveref}

\theoremstyle{plain}
\newtheorem{theorem}{Theorem}[section]
\newtheorem{proposition}[theorem]{Proposition}

\newtheorem{corollary}[theorem]{Corollary}
\theoremstyle{definition}
\newtheorem{definition}[theorem]{Definition}
\newtheorem{assumption}[theorem]{Assumption}
\theoremstyle{remark}
\newtheorem{remark}[theorem]{Remark}

\usepackage[textsize=tiny]{todonotes}

\icmltitlerunning{Learning High-Dimensional Parity Functions with Product Networks using Gradient Descent}

\begin{document}
\twocolumn[
  \icmltitle{Learning High-Dimensional Parity Functions with Product Networks using Gradient Descent}



  \icmlsetsymbol{equal}{*}

  \begin{icmlauthorlist}
    \icmlauthor{Guillaume Larue}{ora}
    \icmlauthor{Louis-Adrien Dufrène}{ora}
    \icmlauthor{Quentin Lampin}{ora}
    \icmlauthor{Hadi Ghauch}{tp}
    \icmlauthor{Ghaya Rekaya}{tp}
  \end{icmlauthorlist}

  \icmlaffiliation{ora}{Orange Research, Meylan, France}
  \icmlaffiliation{tp}{Télécom Paris, Institut Polytechnique de Paris, Palaiseau, France}

  \icmlcorrespondingauthor{Guillaume Larue}{guillaume.larue@orange.com}
  \icmlcorrespondingauthor{Louis-Adrien Dufrène}{louisadrien.dufrene@orange.com}

  \icmlkeywords{Parity Learning, Product Neural Networks, Data Sparsity, Boolean Functions, Convergence Guarantees, Sample Complexity}

  \vskip 0.3in
]



\printAffiliationsAndNotice{}  

\begin{abstract}
    Parity functions are fundamental Boolean operations with critical applications across machine learning, cryptography, and error correction. Yet, learning high-dimensional parity functions poses significant challenges: in a general setting, standard neural network architectures typically require exponential sample complexity, making gradient-based optimization intractable for large number of inputs $N$. We demonstrate that compact product-based neural architectures combined with stochastic data sparsity (Bernoulli inputs with $p_e \leq 1/N$) and appropriate hyperparameter choice enable efficient parity learning, with theoretical guarantees of convergence. Experiments validate our theory across dimensions up to $N = 100{,}000$, with empirical evidence showing optimal hyperparameter choices for $p_e$ and learning rate $\alpha$, as well as polynomial complexity scaling laws. This work establishes fundamental connections between architectural inductive bias and data sparsity, opening new possibilities for neural arithmetic, structured reasoning, binary neural networks, and machine learning applied to automated protocol discovery.
\end{abstract}

\section{Introduction}  
    Parity functions are a canonical family of Boolean operations with deep connections to machine learning, cryptography, coding theory, and digital signal processing. Their deceptively simple structure has posed fundamental challenges to neural computation since the earliest days of the field. The classic 2-bit XOR problem famously demonstrated that single-layer perceptrons cannot achieve linear separability for even the simplest parity, requiring at least two layers, a limitation that contributed to the first "AI Winter" following the influential critique of \cite{PerceptronsAnIntroductiontoComputationalGeometry}. While architectural solutions to small parity problems have long been known \cite{ParitywithTwoLayerFeedforwardNets}, scaling to high-dimensional parity functions introduces new fundamental challenges.

    The subset parity problem involves learning a parity function on an unknown $k$-sized subset of $N$ inputs from labeled training data using gradient descent \cite{ProvableLimitationsofDeepLearning}. One key challenge is identifying which variables participate in the parity computation, a problem that scales exponentially with dimension $N$ in the absence of appropriate inductive biases. In a general setting, neural architectures trained with gradient-based methods require exponentially many samples to learn high-dimensional parities, as optimization often stalls due to the absence of useful intermediate signals to exploit \cite{ProvableLimitationsofDeepLearning, StatisticalQueriesandStatisticalAlgorithmsFoundationsandApplications}.

    Parity problems naturally arise in a  wide range of settings. Within machine learning, parity-like computations appear in arithmetic circuits and neural modules designed for algorithmic generalization \cite{NeuralArithmeticLogicUnits, NeuralArithmeticUnits}, structured probabilistic models like Sum-Product Networks \cite{SPN,SumProductNetworksASurvey}. Parities are central in cryptographic constructions, telecommunications, digital signal processing \cite{DigitalSignalProcessingPrinciplesAlgorithmsandApplications} and finite field arithmetic. 

    Among these applications, neural decoders for error-correcting codes provided the initial motivation for this work: recent results have demonstrated remarkable empirical success in learning decoder architectures via gradient descent \cite{DeepLearningMethodsforImprovedDecodingofLinearCodes,NeuralBeliefPropagationAutoEncoderforLinearBlockCodeDesign,LearningLinearBlockCodesWithGradientQuantization, FactorGraphOptimizationofErrorCorrectingCodesforBeliefPropagationDecoding}, despite the prevalence of parity functions. This raises a fundamental question that extends well beyond coding theory: under what conditions can gradient descent efficiently learn parities? Our work addresses this question in a general setting, proving that product-based architectures become provably learnable via standard gradient descent when inputs satisfy a sparsity condition naturally arising in many practical systems, including error-correcting codes.

    \subsection{Related Work}
        $N$-bit parity functions serve as a canonical benchmark in neural network research, exposing fundamental questions about architectural expressiveness and optimization.
    
        From an architectural perspective, the representability question has been extensively studied: can neural architectures express parity functions, and what constraints govern the size and depth required to compute $N$-bit parity problems? It is well established that single-layer perceptrons cannot solve even 2-bit XOR, requiring at least depth-2 architectures \cite{PerceptronsAnIntroductiontoComputationalGeometry,ParitywithTwoLayerFeedforwardNets}. A straightforward extension for $N$-bit parity involves stacking $N$ layers with skip connections, implementing the recursive structure $y = \bigoplus_{i=1}^{N}{x_{i}}= x_{1}\oplus(x_{2}\oplus(...\oplus(x_{N-1}\oplus x_N)...))$. Alternative solutions include cascading structures, hierarchical trees \cite{SolvingParityNProblemsWithFeedForwardNeuralNetworks, SolvingNbitParityProblemUsingNN, ParitywithTwoLayerFeedforwardNets}, but all exhibit fundamental depth/width trade-offs when relying on additive networks with conventional activations (ReLU, sigmoid, etc.). 

        A crucial distinction emerges between what can be represented and what can be learned: architectural expressiveness does not guarantee learnability \cite{AreEfficientDeepRepresentationsLearnable}. While manual weight configuration can provably implement specific subset parity functions within these architectures \cite{SolvingParityNProblemsWithFeedForwardNeuralNetworks, NBitParityNeuralNetworksNewSolutionsBasedonLinearProgramming}, learning these weights through gradient descent presents fundamental challenges in particular regarding sample complexity and generalization capabilities \cite{AreEfficientDeepRepresentationsLearnable,AMathematicalModelforCurriculumLearningforParities,LearningHighDegreeParitiesTheCrucialRoleoftheInitialization,ProvableAdvantageofCurriculumLearningonParityTargetswithMixedInputs,LearningParitieswithNeuralNetworks}. This representation-learnability gap highlights that architectural capacity alone is insufficient. The optimization landscape, initialization strategies, and training dynamics all determine whether a representable function can be successfully learned.

        While parity functions are efficiently learnable using algorithms like Gaussian elimination (PAC-learnable) \cite{StatisticalQueriesandStatisticalAlgorithmsFoundationsandApplications}, gradient descent faces fundamental obstacles. 

        The core difficulty stems from parity functions' inability to be approximated by linear combinations of lower-order components, eliminating the incremental learning pathways that gradient descent typically exploits. This creates an optimization landscape where intermediate representations provide no useful signal toward the target function, forcing algorithms into an all-or-nothing learning regime fundamentally at odds with gradient-based optimization's incremental nature. While deep networks can represent parity effectively, standard gradient descent from random initialization on uniform data often fails to discover these representations, succeeding only when initialized extremely close to optimal solutions  \cite{AreEfficientDeepRepresentationsLearnable}. Within the Statistical Query framework, parity learning requires exponential sample complexity, as parity functions on $\{0, 1\}^N$ have SQ-DIM = $2^N$ \cite{StatisticalQueriesandStatisticalAlgorithmsFoundationsandApplications}.
    
        Some solutions have been proposed to overcome gradient descent's limitations on parity learning. One approach leverages specialized data distributions: sparse parity functions can be learned when the data distribution is uniform everywhere except on the parity's support, where all bits are perfectly correlated \cite{LearningParitieswithNeuralNetworks}. Curriculum learning strategies enable polynomial-time learning by progressively increasing complexity, i.e. first training on sparse inputs, then transitioning to dense uniform inputs \cite{AMathematicalModelforCurriculumLearningforParities, ProvableAdvantageofCurriculumLearningonParityTargetswithMixedInputs}. Careful initialization strategies can be crucial: SGD succeeds on MLPs for full or near-full parities when using Rademacher initialization \cite{LearningHighDegreeParitiesTheCrucialRoleoftheInitialization}. These approaches demonstrate that strategic control of data distribution, training progression, and initialization can enable gradient-based parity learning on additive networks.

    \subsection{Our Work}
        In this work, we focus on a compact \textbf{product-based architecture} combined with \textbf{stochastic dataset sparsity} to enable arbitrary subset parity learning in \textbf{high-dimensional spaces}. Product-based structures inherently create an inductive bias that provides straightforward \textbf{generalization} and thus improved \textbf{sample efficiency} through direct architectural alignment with parity computation. 

        Parity functions can be directly expressed as products of bipolar inputs, making product-based architectures a natural choice. Without this structural alignment, standard MLPs trained under conventional methods face a fundamental trade-off: either the model fails to generalize outside its training set, or the training set must be exponentially large relative to problem dimensionality \cite{LearningParitieswithNeuralNetworks,StatisticalQueriesandStatisticalAlgorithmsFoundationsandApplications}. Product-based architectures circumvent this limitation by embedding parity's mathematical structure directly into a single neuron per parity subset, as shown later.

\begin{figure*}[t]
    \begin{subfigure}[t]{125pt}
        \centering
        \resizebox{75pt}{75pt}{
            \begin{tikzpicture}
                \draw[thick,->] (0,0) -- (0,6.5);
                \draw[thick,->] (0,0) -- (6.5,0);
                \draw[thick] (-0.2,4) -- (+0.2,4);
                \draw[thick] (4,-0.2) -- (4,+0.2);
                \node [font=\Huge] at (4,-1) {$1$};
                \node [font=\Huge] at (-1,4) {$1$};
                \node [font=\Huge] at (-0.5,-0.5) {$0$};
                \node [font=\Huge] at (6.5,-1) {$x_{1}$};
                \node [font=\Huge] at (-1,6.5) {$x_{2}$};
                
                \draw[fill=white] (0,0) circle (0.25);
                \draw[fill=black] (0,4) circle (0.25);
                \draw[fill=black] (4,0) circle (0.25);
                \draw[fill=white] (4,4) circle (0.25);
                
                \draw[thick,dotted] (0,6) -- (6,0);
                \draw[thick,dotted] (0,2) -- (2,0);
                \draw[thick,dashed,red] (0.1,6.1) -- (6.1,0.1);
                \draw[thick,dashed,red] (0.1,2.1) -- (2.1,0.1);
            \end{tikzpicture}
            }
        \caption{Complete dataset (4/4 samples).} 
        \label{fig:2-XOR-1}
    \end{subfigure}
    \hfill
    \begin{subfigure}[t]{125pt}
        \centering
        \resizebox{75pt}{75pt}{
        \begin{tikzpicture}
            \draw[thick,->] (0,0) -- (0,6.5);
            \draw[thick,->] (0,0) -- (6.5,0);
            \draw[thick] (-0.2,4) -- (+0.2,4);
            \draw[thick] (4,-0.2) -- (4,+0.2);
            \node [font=\Huge] at (4,-1) {$1$};
            \node [font=\Huge] at (-1,4) {$1$};
            \node [font=\Huge] at (-0.5,-0.5) {$0$};
            \node [font=\Huge] at (6.5,-1) {$x_{1}$};
            \node [font=\Huge] at (-1,6.5) {$x_{2}$};
            
            \draw[fill=white] (0,0) circle (0.25);
            \draw[fill=black] (4,0) circle (0.25);
            \draw[fill=white] (4,4) circle (0.25);
            
            \draw[thick,dotted] (1.5,0) -- (5.5,4);
            \draw[thick,dashed,red] (0.1,6.1) -- (6.1,0.1);
            \draw[thick,dashed,red] (0.1,2.1) -- (2.1,0.1);
        \end{tikzpicture}
        }
    \caption{Partial dataset (3/4 samples).}

    \label{fig:2-XOR-2}
    \end{subfigure}\hfill
    \begin{subfigure}[t]{125pt}
        \centering
        \resizebox{75pt}{75pt}{
        \begin{tikzpicture}
            \draw[thick,->] (0,0) -- (0,6.5);
            \draw[thick,->] (0,0) -- (6.5,0);
            \draw[thick] (-0.2,4) -- (+0.2,4);
            \draw[thick] (4,-0.2) -- (4,+0.2);
            \node [font=\Huge]at (4,-1) {$1$};
            \node [font=\Huge]at (-1,4) {$1$};
            \node [font=\Huge]at (-0.5,-0.5) {$0$};
            \node [font=\Huge]at (6.5,-1) {$x_{1}$};
            \node [font=\Huge]at (-1,6.5) {$x_{2}$};
            
            \draw[fill=black] (0,4) circle (0.25);
            \draw[fill=black] (4,0) circle (0.25);
            
            \draw[thick,dashed,red] (0.1,6.1) -- (6.1,0.1);
            \draw[thick,dashed,red] (0.1,2.1) -- (2.1,0.1);
        \end{tikzpicture}
        }
    \caption{Partial dataset (2/4 samples). }

    \label{fig:2-XOR-3}
    \end{subfigure}\hfill

    \vspace{0.5cm}
    \caption{\textbf{2-XOR learning scenarios:} The axes represent input values $\mathbfit{x}=\{x_1,x_2\}$ and circle colors denote labels $y$ (white: 0, black: 1).  \textbf{(a)} Training success guarantees test generalization with learned separating boundaries (dotted - black lines) solving the true task (dashed - red lines). \textbf{(b)} Learned boundary separates training data but fails on true XOR task. \textbf{(c)} General model may learn trivial "always 1" solution while a structured XOR model would have sufficient information for full parametrization. This sample efficiency advantage becomes critical for N-XOR problems where the gap between structured and unstructured data requirements grows exponentially.}

    \label{fig:xor_scenarios}

\end{figure*}
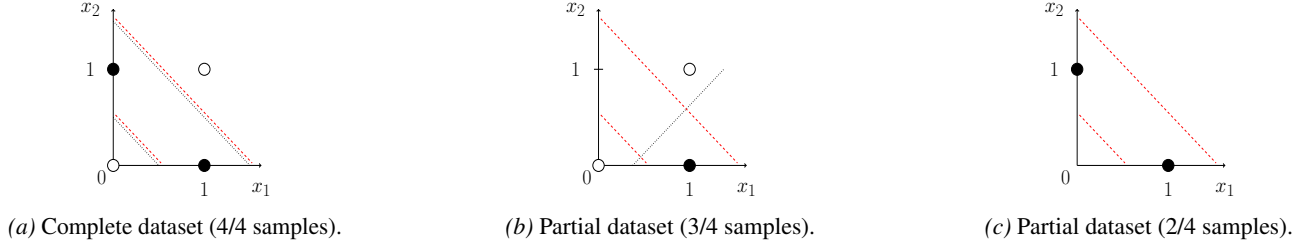

        Figure~\ref{fig:xor_scenarios} illustrates the fundamental generalization challenge faced by standard architectures when learning parity functions from incomplete datasets. Standard MLPs fail to generalize when trained on subsets of the full truth table due to parity functions' non-linear separability. Without access to all $2^N$ possible input combinations, these architectures may find decision boundaries that separate the training data but fail to capture the underlying parity logic. For instance, in the 2-bit XOR case, an MLP trained on only 3 out of 4 possible examples can learn a separator that works on the training set but misclassifies the held-out example (Figure~\ref{fig:xor_scenarios}.b). This problem becomes prohibitively severe as the dimensionality increases, requiring exhaustive datasets that scale exponentially with problem size. In contrast, product-based architectures can identify the parity function even from minimal training data (e.g., the $N$ basis vectors as in Figure~\ref{fig:xor_scenarios}.c), as  identifying the subset's support enables full generalization across the truth table.

        While product-based architectures provide significant advantages in generalization, they do not eliminate the inherent challenges of subset parity learning. Under uniform input distribution, the fundamental difficulty of identifying the correct subset of relevant variables from exponentially many possibilities remains. Moreover, product operations introduce numerical complexities that can complicate learning. The non-linear nature of products can lead to optimization challenges including non-convex loss landscapes, reduced linear separability, and potential gradient instability. 

        In addition to architectural inductive bias, which aids generalization, we leverage inductive bias on input data distribution to facilitate parity subset identification.  We employ stochastic per-example sparsity throughout training, demonstrating that gradient descent on product-based networks can successfully identify the target subset without requiring explicit curriculum schedules. A key practical advantage of our stochastic sparsity approach is its modular applicability: unlike curriculum methods that typically operate on raw inputs, our technique can be integrated at any layer within a complex architecture. By applying sparsity regularization to intermediate feature representations that feed into product-based layers, we enable parity learning on learned features rather than just input variables, allowing flexible integration of XOR operations throughout deep networks.

        Hence, our method leverages two complementary inductive biases: \textbf{(1) architectural bias} through product-based structures enabling direct generalization once the support is identified, even from partial truth tables, without requiring exhaustive coverage; and \textbf{(2) data distribution bias} through stochastic sparsity that makes subset identification tractable and regularizes the optimization landscape, leading to near-convex dynamics under gradient descent.

    \subsection{Plan of the Paper \& Contributions} 
        This paper provides a comprehensive theoretical and empirical analysis of product-based architectures for subset parity learning under stochastic sparsity conditions. In \textbf{Section~\ref{section:problem_formulation}}, we formally define the subset parity learning problem and introduce an oracle framework enabling precise convergence analysis. In \textbf{Section~\ref{section:learning_parities_with_product_node}} we define product nodes and prove their training is generally non-convex and lacks linear separability, explaining optimization difficulties in standard settings. We define our XOR node as a special case of product units, thus inheriting its properties.  In \textbf{Section~\ref{section:learning_parities_under_stochastic_sparsity}}, we develop our main theoretical framework: we demonstrate how sparse input structures enable tractable learning by making product nodes strictly convex, introduce our stochastic sparsity model with Bernoulli probability $p_e$, establish key properties of XOR nodes with normally distributed weights, and prove guaranteed convergence to a bounded interval under bounded learning rate when $p_e = 1/N$. In \textbf{Section~\ref{section:experimental_validation}}, we provide empirical validation, demonstrating convergence properties for large dimensions up to $N = 100{,}000$.  We study weight distribution behaviors and characterize empirical boundaries where convergence fails, validating alignment with our theoretical predictions. We exhibit polynomial complexity scaling laws in $N$ primarily driven by the effective learning rate $\alpha p_e$, avoiding the exponential scaling that characterizes standard approaches to parity learning.

       
        \textbf{Main Contribution:} We establish a complete theoretical framework for provably efficient subset parity learning in high-dimensional spaces, combining product-based architectures with stochastic data sparsity. Our oracle-based analysis provides precise convergence conditions and hyperparameter bounds, while extensive experiments validate practical applicability for problems with thousands of dimensions\footnote{The source code to reproduce all listed experiments is available at: \url{https://github.com/Orange-OpenSource/learning-parities-with-product-networks}}.


\section{Problem Formulation}\label{section:problem_formulation}
    We consider a supervised learning setting with training set $\mathbb{S} = \{(\mathbfit{x}_m, y_m^{\mathrm{true}})\}_{m=1}^M$ where $\mathbfit{x}_m \in \mathbb{R}^N$ are input vectors and $y_m^{\mathrm{true}} \in \mathbb{R}$ are target outputs. Our model $f(\mathbfit{w}; \mathbfit{x}_m)$ is parameterized by weights $\mathbfit{w} \in \mathbb{R}^N$. We consider an oracle-based setup where the true labels are generated by applying the same function $f$ to the inputs using dedicated oracle parameters $\mathbfit{w}^{\mathrm{true}}$, unknown to the model to be trained: 
    \begin{equation}
        y_m^{\mathrm{true}} = f(\mathbfit{w}^{\mathrm{true}}; \mathbfit{x}_m), \quad \forall m \in \{1,\ldots,M\}
    \end{equation}
    This oracle paradigm allows controlled analysis of the optimization landscape, convergence to the true solution, and theoretical learnability guarantees under specific conditions. The learning objective is to minimize the empirical risk using gradient descent:
    \begin{equation}
        \hat{\ell}(\mathbfit{w}) = \frac{1}{M} \sum_{m=1}^M l(f(\mathbfit{w}; \mathbfit{x}_m), y_m^{\mathrm{true}}) 
    \end{equation}
    
    We focus on learning XOR functions defined, given binary\footnote{Although oracle weights and inputs are typically binary ($\mathbfit{w}^{\mathrm{true}},\mathbfit{x}_m  \in \{0,1\}^N$), model weights $\mathbfit{w}\in\mathbb{R}^N$ are real-valued, and expected to converge to the discrete values $\{0,1\}^N \subset \mathbb{R}^N$.}. inputs $\mathbfit{x}_m \in \{0,1\}^N$ and weights $\mathbfit{w}^{\mathrm{true}} \in \{0,1\}^N$, as:
    
    \begin{equation}
        y_m = f(\mathbfit{w}; \mathbfit{x}_m) = \bigoplus_{i=1}^{N} w_i x_{m,i}
    \end{equation}
    
    where $\bigoplus$ denotes XOR operation (addition in $\mathbb{F}_2$).
    


\section{Parity Functions \& Product Nodes}\label{section:learning_parities_with_product_node}
    \subsection{On the Inductive Bias of Product Node}
        $N$-bit parity presents two fundamental challenges to be addressed for successful gradient descent based learning:
        
        \textbf{Subset identification:} determining which of the $N$ input variables participate in the parity. Assuming inputs are uniformly distributed, both additive and product-based networks face exponential complexity in identifying the correct subset of relevant variables through gradient descent. Data distribution bias can help mitigate this aspect.

        \textbf{Generalization:} correctly computing parity logic on all $2^N$ input combinations. Even when subset identification is solved (e.g., via specific input distributions), different architectures may exhibit fundamentally different generalization capabilities due to distinct architectural inductive biases. 

        To illustrate this distinction, consider 2-XOR scenario with partial training data: $\{(\mathbfit{x}, y^{\mathrm{true}}):(1,0) \rightarrow 1, (0,1) \rightarrow 1\}$ (Figure~\ref{fig:xor_scenarios}.c). An additive network could converge to the "always 1" solution, satisfying training data but failing on test cases $\{(0,0) \rightarrow 0, (1,1) \rightarrow 0\}$, highlighting the need for exponentially large training sets. In contrast, a product-based XOR node\footnote{Using bipolar transformation $\{0,1\} \to \{+1,-1\}$ where XOR corresponds to product operation. This assumes all variables participate in parity; complex scenarios follow in later sections.} has a unique solution (both weights = 1) that automatically generalizes correctly, demonstrating structural inductive bias aligned with XOR logic.

        Hence, while specific input distributions can universally solve subset identification (with polynomial sample complexity), well crafted product architectures inherently solve boolean generalization through structural inductive bias. 
    
    \subsection{The Optimization Challenge of Product Node}
        The architectural advantage of product nodes comes at the cost of increased optimization complexity. While additive nodes of the form $\sum w_ix_i$ ensure convex loss landscapes enabling straightforward gradient descent convergence, naive product nodes of the form $\prod w_ix_i$ introduce fundamental optimization challenges that complicate learning.

        They are non-linearly separable (as evidenced by the gradient equation in Prop.~\ref{prop:product_gradient}) and create non-convex optimization landscapes (Prop.~\ref{prop:product_nonconvex}) with many degenerate solutions. Their product structure causes numerical instabilities, including vanishing and exploding gradient that can severely impede training. See Appendix~\ref{app:product_node_analysis} for detailed analysis of such product nodes behaviors.   

        A critical issue plaguing naive product nodes is the vanishing output problem: any zero weight causes the entire product to collapse to zero, preventing the node from selective input discarding. This restricts the node's ability to learn meaningful representations. To address this vanishing behavior, we draw inspiration from the concept of neutral elements and approaches in \cite{NeuralArithmeticUnits,NeuralArithmeticLogicUnits}. We employ throughout this paper a modified product node of the form:
        \begin{equation}
            \prod_{i=1}^N \Big(w_ix_i + (1-w_i)\Big)
        \end{equation}

        In this formulation, null weights contribute a factor of one rather than forcing the entire product to zero. This modification preserves the multiplicative structure while eliminating pathological vanishing behavior. Crucially, it enables subset selection by setting irrelevant inputs' weights to zero. Detailed analysis of this node is provided in Appendix~\ref{app:product_with_neutral_element_node_analysis}.

        XOR is commonly represented as a product over bipolar inputs $\{-1,+1\}$, providing a natural connection between product nodes and parity computation. Using our product formulation with neutral elements, we can express parity as:

        \begin{equation}\label{eq:xor}
            \bigoplus_{i=1}^{N} w_ix_{i}= \frac{1}{2}\Bigg[1-\prod_{i=1}^N\bigg(w_i(1-2x_i)+(1-w_i)\bigg)\Bigg]
        \end{equation}
        
        When $w_i,x_i \in \{0,1\}$, this exactly computes the parity of the subset of inputs where $w_i=1$. We refer the reader to Appendix~\ref{app:xor_node_analysis} for detailed node analysis.
    

\section{Learning Parities under Sparsity}\label{section:learning_parities_under_stochastic_sparsity}
    
    \textit{The following Sections provide a proof sketch that guides the reader through the logical chain establishing convergence. Key ideas and intermediate results are presented here while complete proofs are provided in Appendices.}

    \subsection{Sparsity-Induced Convexity}\label{subsection:deterministic_sparsity}
        While the proposed neutral element approach does not solve non-convexity in the general case (Prop.~\ref{prop:product_node_with_ne_non_convexity}), deterministic unit-sparse input distributions (\emph{i.e.} each input vector has at most one non-zero entry) create a linearly separable  optimization landscape (sparsity eliminates multiplicative interactions between variables) ensuring convexity (Prop.~\ref{prop:sparse_convexity}). In this regime, the expected gradient becomes proportional to the $L1$ distance to target parameters, \emph{i.e.} $\mathbb{E}\left\{\nabla_\mathbfit{w}\hat{\ell}_{\overset{\bullet}{\Pi}}(\mathbfit{w})\right\} \propto(\mathbfit{w} - \mathbfit{w}^\mathrm{true})$, similar to additive nodes.

        The XOR operator built using this product structure naturally inherits this convexity property under one-hot input distributions (the binary case of unit-sparse input distributions). Each one-hot input vector provides an isolated "probe" of a single variable's parity contribution, if any (Prop.~\ref{prop:xor_sparse_data}).

        This sparse regime provides dual benefits: it reduces subset identification complexity from $\mathcal{O}(2^N)$ to $\mathcal{O}(N)$ samples while transforming the non-convex problem into a convex one. The product architecture then enables automatic Boolean generalization without additional training, achieving polynomial-time learning.

        Exact unit sparsity is rarely encountered in practical scenarios, limiting the applicability of this theoretical result. On the contrary, stochastic sparsity is commonly found in applications such as error-correcting codes in communication systems\footnote{When learning Linear Block Codes or their decoders, neural networks are typically trained on the zero codeword perturbed by sparse errors under controlled noise levels (by code linearity, this generalizes to all codewords). A typical 5G LDPC code has block size $\approx 8{,}000$ bits, but decoders process local check nodes of degree $d_c \approx 10$. Even at a pessimistic raw bit error rate (BER) of $10^{-1}$, each check node observes on average $1$ erroneous bit. Under realistic BER $\approx 10^{-3}$, the sparsity becomes $p_e \approx 10^{-3}$, ensuring unit-sparse error regimes for operations up to $N \approx 1{,}000$.} \cite{LearningLinearBlockCodesWithGradientQuantization, DeepLearningMethodsforImprovedDecodingofLinearCodes, FactorGraphOptimizationofErrorCorrectingCodesforBeliefPropagationDecoding}. This raises a critical question: can we maintain convergence guarantees under such sparsity conditions?

        To address this, we analyze parity learning under stochastic sparsity where each input vector $\mathbfit{x}_m$ is generated by independently sampling inputs according to a Bernoulli distribution $x_{m,i} \sim B(p_e)$ for all $i \in \{1,\ldots,N\}$, where the probability  $p_e$ controls the expected sparsity level. We outline three distinct regimes: \textbf{ultra-sparse} ($p_e \ll 1/N$, mostly zero vectors), \textbf{unit-sparse} ($p_e = 1/N$, one active bit on average), and \textbf{dense} ($p_e \gg 1/N$, multiple active bits are common).

    \subsection{Normally Distributed Model Weights}\label{subsection:normally_distributed_weights}
        To analyze XOR node optimization dynamics under gradient descent, we adopt a distributional perspective. Let $D[k]$ denote the model's weight distribution and $D_0[k]$, $D_1[k]$, the weight distributions associated with oracle target values of 0 and 1, respectively, at training step $k$.

        When weight distributions are Gaussian, i.e., $D_0[k]\sim\mathcal{N}(\mu_0[k],\sigma_0^2[k])$ and $D_1[k]\sim\mathcal{N}(\mu_1[k],\sigma_1^2[k])$, the distributional framework exhibits two key properties (experimentally validated in Section~\ref{subsubsection:validation_gaussian_model}): 

        \textbf{Gaussian Preservation:} Under our setting, gradient descent updates constitute affine transformations, ensuring that Gaussian distributions remain Gaussian across training steps (Prop.~\ref{prop:gaussian_conservation}).

        \textbf{Symmetry Preservation:} If the distributions are symmetric w.r.t. $1/2$ at step $k$, i.e. $\mu_0[k] = 1 - \mu_1[k]$ and $\sigma_0^2[k] = \sigma_1^2[k]$, this symmetry is preserved at step $k+1$ (Prop.~\ref{prop:symmetry_conservation}).
        
        Proper initialization (Gaussian distribution $\mathcal{N}(1/2, \sigma^2[0])$) ensures these properties hold from the start, enabling analytical simplifications. Gradient updates become deterministic transformations of distributional moments $(\mu[k], \sigma^2[k])$ and symmetric coupling allows tracking a single distribution $(\mu[k], \sigma^2[k]) := (\mu_0[k], \sigma_0^2[k])$ without loss of generality, since $(\mu_1[k], \sigma_1^2[k]) = (1-\mu_0[k], \sigma_0^2[k])$. This reduces the analysis from tracking $N$ individual weights to two distributional parameters. Moreover, symmetry renders the system invariant to the proportion $p_w$ of oracle weights equal to 1, ensuring the behavior is independent of the target parity's structure (see Remark~\ref{remark:gaussian_init}). Notably, these arguments require no knowledge of the oracle weight assignments.
        

    \subsection{Convergence Analysis under Unit Sparsity}\label{subsection:convergence_analysis_under_unit_sparsity}
        We now analyze convergence behavior under stochastic unit sparsity, where $p_e=1/N$ ensures an average of one active bit per input vector. By the symmetry established in Prop.~\ref{prop:symmetry_conservation}, we analyze only distribution $D_0[k]$ using notation $(\mu[k], \sigma^2[k]) := (\mu_0[k], \sigma_0^2[k])$, with $D_1[k]$ behavior following by symmetry. We define convergence of weights to their true target values as:
        \begin{equation}
            \lim_{k\to\infty} \mu[k] = 0\\ \quad \text{and}\quad \lim_{k\to\infty} \sigma^2[k] = 0\\
        \end{equation}

        Our convergence analysis operates within a bounded domain defined by $\mu[k] \in [-1/4, 1/2]$ and $\sigma^2[k] \in (0, 1/4]$ (Assumption~\ref{assumption:system_assumption}). We initialize the system with a symmetric Gaussian distribution centered at $\mu[0] = 1/2$ and variance $\sigma^2[0] = 1/4$, which ensures both the symmetry condition and places the system within our analysis domain (Remark~\ref{remark:gaussian_init}). Under bounded domain assumptions, we express the gradient descent weight update using series expansion (Prop.~\ref{prop:exp_approx_when_pe_1_N} and Remark~\ref{remark:exp_approx_update_rules}). The resulting update's affine form $w_i[k+1] = m[k]w_i[k] + c[k]$ enables direct distributional analysis. Applying the transformation $aX+b \sim \mathcal{N}(a\mu+b,a^2\sigma^2)$ for $X \sim \mathcal{N}(\mu,\sigma^2)$ yields $\mu[k+1] = m[k]\mu[k] + c[k]$ and $\sigma^2[k+1] = m^2[k]\sigma^2[k]$.
                      
        This separation enables independent analysis: variance exhibits simple geometric decay when $|m[k]| < 1$  (see Section~\ref{subsubsection:convergence_analysis_variance}), while mean convergence involves more complex dynamics where both contracting factor $m[k]$ and bias term $c[k]$ evolve with the system state  (see Section~\ref{subsubsection:convergence_analysis_mean}). 

        \textit{\textbf{Analytical Methodology:} In upcoming sections, using our series expansions, we establish worst-case bounds on truncation errors over our assumed domain, then derive conditions ensuring certain properties under these bounds.} 
        
        \subsubsection{Convergence Analysis - Variance}\label{subsubsection:convergence_analysis_variance}
        
            We first analyze convergence of variance $\sigma^2[k]$ under Assumptions~\ref{assumption:system_assumption}. Since variance evolves as $\sigma^2[k+1] = m^2[k]\sigma^2[k]$, convergence is guaranteed when $|m[k]| < 1$ for all $k$, ensuring geometric decay toward zero. This condition holds when the learning rate satisfies: 
            \begin{equation}
                \alpha < \alpha_0 = Ne^{-\frac{72N-9}{32N-72}} \quad \text{(Prop.~\ref{prop:sigma_convergence})}
            \end{equation}   

            A stricter condition $\alpha < \alpha_1 = \frac{\alpha_0}{2}$ ensures $0 < m[k] < 1$, guaranteeing the affine transformation exhibits no sign reversals (Prop.~\ref{prop:monotonic_constraint}). While variance convergence only requires $|m[k]| < 1$, the condition $0 < m[k] < 1$ is crucial for mean convergence analysis, as it ensures the distance to the instantaneous fixed point $\overline{\mu}[k] = \frac{c[k]}{1-m[k]}$ decreases monotonically without overshooting.
                   
    \subsubsection{Convergence Analysis - Mean}\label{subsubsection:convergence_analysis_mean}

        Having established variance convergence, we now analyze convergence of the mean $\mu[k]$.
        
        Since gradient descent typically moves in small incremental updates across the parameter space, we first establish in Prop.~\ref{prop:bounded_update_steps} that update steps $\delta[k] = \mu[k+1] - \mu[k]$ are uniformly bounded across our working domain by $\delta_{\max}$.
        
        The mean dynamics are governed by convergence toward instantaneous fixed points $\overline{\mu}[k] = \frac{c[k]}{1-m[k]}$ that vary with system state. In Prop.~\ref{prop:envelope_characterization} we bound the instantaneous fixed point within $\overline{\mu}_{\min}[k] \leq \overline{\mu}[k] \leq \overline{\mu}_{\max}[k]$. In Prop.~\ref{prop:bounded_convergence_interval}, we analyze where these envelope bounds intersect the identity line $\mu[k+1] = \mu[k]$. This yields the bounds $\phi_{\min}'$ and $\phi_{\max}'$ that contain these intersections, thereby bounding the convergence region where the true equilibrium $\overline{\mu}[k]=\mu[k]$ must lie. Using envelope functions monotonicity, we establish directional convergence: trajectories with $\mu[k] < \phi_{\min}'$ are driven upward, while those with $\mu[k] > \phi_{\max}'$ are driven downward, ensuring systematic attraction to $[\phi_{\min}', \phi_{\max}']$. Under bounded updates, we prove convergence to the padded interval $[\phi_{\min}' - \delta_{\max}, \phi_{\max}' + \delta_{\max}]$, when $N\geq18$ and $\alpha<\alpha_2$, where $\delta_{\max} = \mathcal{O}(\alpha/N)$ can be made arbitrarily small by reducing $\alpha$. In the limit $\alpha \to 0$, convergence is guaranteed within $[\phi_{\min}', \phi_{\max}']$, where the interval size scales as $\mathcal{O}(1/N)$.
 
        We verify forward invariance of system assumptions through two approaches: requiring $N \geq 43$ with $\alpha < \alpha_1$ (Prop.~\ref{prop:envelope_invariance}), or the relaxed condition $N \geq 8$ under sufficiently small $\alpha$ ensuring bounded updates (Prop.~\ref{prop:relaxed_stability}).
  
    \subsubsection{Closure}

        Having established the individual components of our convergence analysis, we now synthesize these results to provide our convergence proof for parity function learning. Under our proposed symmetric initialization $\mu[0] = 1/2$, $\sigma^2[0] = 1/4$, we have demonstrated:

        \textbf{Variance geometric convergence}: $\lim_{k\to\infty} \sigma^2[k] = 0$ when $\alpha < \alpha_0$ (Prop.~\ref{prop:sigma_convergence})

        \textbf{Mean bounded convergence}: $\mu[k]$ converges to interval $[\phi_{\min}' - \delta_{\max}, \phi_{\max}' + \delta_{\max}]$ when $N \geq 18$ and $\alpha < \alpha_2$ (Prop.~\ref{prop:bounded_convergence_interval})
               
        Convergence interval size scales as $\mathcal{O}(1/N) + \mathcal{O}(\alpha/N)$, providing arbitrarily tight convergence as $N \to \infty$. The complete closure of the demonstration is provided in Prop.~\ref{prop:complete_convergence_closure}.

    \begin{figure*}[htbp]
        \centerline{\includegraphics[width=470pt]{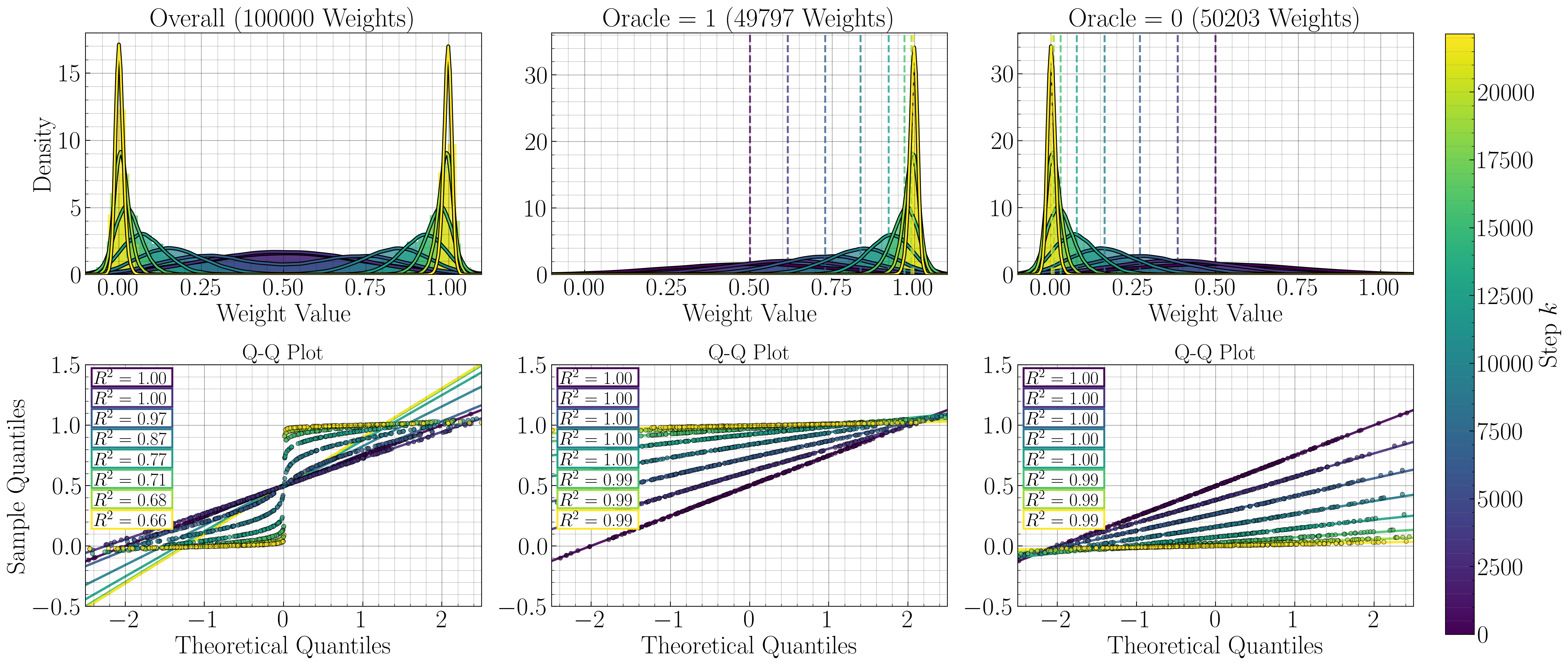}}
        \caption{Validation of the symmetric Gaussian properties. \textbf{Top row}: Weight distribution snapshots during training (purple: early, yellow: late) for all weights (left), target-1 set (center), and target-0 set (right). \textbf{Bottom row}: Corresponding Q-Q plots against theoretical Gaussian distributions. While combined distribution becomes bimodal, each set in isolation maintains Gaussianity throughout training, and converge symmetrically to targets. \textit{Conditions}: $N=100{,}000$, $M=1{,}000$, $\alpha=10$, $p_e=1/N$, $p_w=1/2$, $P=1$, $S=1{,}000{,}000$}
        \label{fig:gaussian_conservation}
    \end{figure*}
    
    \section{Experimental validation}\label{section:experimental_validation}
        \subsection{System model \& Implementation}

            We now evaluate our theoretical derivation through experiments. A model and oracle are defined each composed of $P$ parallel (independant) XOR units implementing product equation \eqref{eq:xor} (equivalently Appendix~\ref{app:xor_node_analysis}-\eqref{eq:xor_product}). We define $\mathbfit{W} \in \mathbb{R}^{N \times P}$ and $\mathbfit{W}^\mathrm{true} \in \mathbb{F}_2^{N \times P}$ the $N\times P$ parameter matrices of model and oracle whose columns represent the parameter vector of one of the $P$ XOR units. This setting allows efficient study of XOR node statistical behavior by executing many nodes in parallel. Sparse input vectors $\mathbf{b}_m \in \mathbb{F}_2^N$ for $m \in \{1,...,M\}$ are generated by independently setting each component to 1 with probability $p_e$ (Bernoulli process). MSE loss is used to match previous derivations. To ensure consistent learning dynamics, the loss function is normalized by sample count $M$ only (not by the number of nodes $P$), ensuring gradient magnitude for each node remains independent of $P$:
            \begin{equation}
                \hat{\ell}(\mathbfit{w}) = \frac{1}{M}\sum_{m=1}^{M} \sum_{p=1}^P (y_{m,p} - y_{m,p}^\mathrm{true})^2
            \end{equation}

            A standard SGD optimizer, without momentum, and a learning rate $\alpha$ is used with batch of size $M$.

    \subsection{Studies \& Results} 
        \subsubsection{Validation of the Gaussian Model}\label{subsubsection:validation_gaussian_model}
            
            Our theoretical analysis assumes weight distributions $D_0[k]$ and $D_1[k]$ (associated with oracle targets 0 and 1) remain Gaussian and symmetric around $1/2$ throughout training (Propositions~\ref{prop:gaussian_conservation} and~\ref{prop:symmetry_conservation}). We experimentally validate these claims using a single XOR node ($P=1$) with $N = 100{,}000$ to avoid averaging effects across nodes while ensuring sufficient sample size for reliable distributional analysis.

            Figure~\ref{fig:gaussian_conservation} shows weight distribution evolution through detailed distributional analysis with Q-Q plots, demonstrating symmetric weight evolution around $\mu = 1/2$ that confirms Propositions~\ref{prop:symmetry_conservation} and~\ref{prop:gaussian_conservation}. While combined weights become bimodal as they converge to distinct targets 0 and 1, each set of weights considered in isolation maintains perfect Gaussianity and symmetry throughout training. 

            Proposition~\ref{prop:symmetry_conservation} predicts learning independence from oracle weight proportion $p_w$. Experimental validation across all $p_w$ values (Appendix~\ref{appendix:pw_independence}) confirms this theoretical prediction.   

            These results demonstrate successful parity learning across all target densities, confirming practical applicability of our theoretical framework at high dimensions ($N=100{,}000$).

            \begin{figure}[htbp]
                \centerline{\includegraphics[width=225pt]{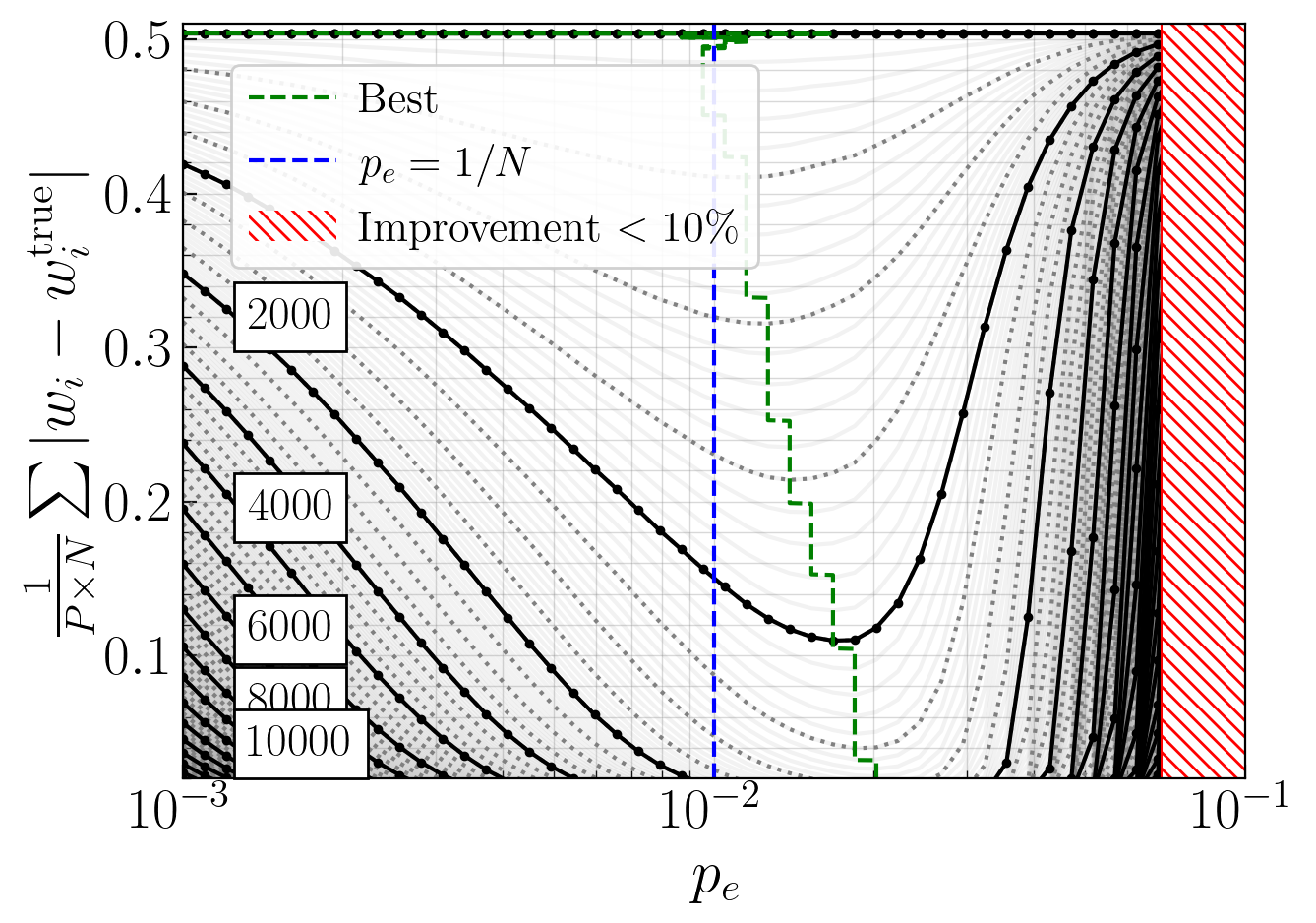}}
                \caption{Convergence analysis across sparsity regimes with constant $p_e$. Iso-step curves show training steps to reach convergence levels for fixed $p_e$ values. \textbf{Blue}: theoretical baseline $p_e=1/N$. \textbf{Green}: iso-curve minima trajectory showing best constant-$p_e$ shifts to $p_e \approx 2/N$ for complete training. \textbf{Red: } $p_e>p_e^\mathrm{lim}$, less than $10\%$ distance reduction in $S$ steps. \textit{Conditions}: $N=100$, $P=1000$, $\alpha=0.1$, $M=100$, $p_w=0.5$, $S=25{,}000$.}
                \label{fig:convergence_as_a_function_of_p_e}
            \end{figure}

        \subsubsection{Impact of Dataset Sparsity $p_e$}\label{subsubsection:impact_of_data_sparsity}
            
            We investigate the impact of dataset sparsity $p_e$ on training efficiency under constant $p_e$ values throughout training. We observe convergence by measuring the average $L_1$ norm between model and oracle weights and stop training when the norm is below $1\%$ or after a maximum of $S$ steps.
            
            Figure~\ref{fig:convergence_as_a_function_of_p_e} exhibits convergence behavior for $N=100$ across $p_e \in [10^{-3}, 10^{-1}]$, revealing the previously outlined regimes: \textbf{ultra-sparse} ($p_e \ll 1/N$, slow convergence, $\nabla \ell(\mathbfit{w})[0]\propto p_e$), \textbf{unit sparsity} ($p_e = 1/N$, theoretical baseline, $\max_{p_e}\nabla \ell(\mathbfit{w})[0]$), and \textbf{dense} ($p_e \gg 1/N$, exponential slowdown, $\nabla \ell(\mathbfit{w})[0] \propto(1-p_e)^{N-1}$) (Prop.~\ref{prop:optimal_error_probability}).
            
            While constant $p_e = 1/N$ guarantees convergence, the empirical optimum at $p_e^* \approx 2/N$ suggests slightly denser inputs can accelerate training. This reveals a trade-off: $p_e = 1/N$ enables fastest initialization but slower completion, while $p_e \approx 2/N$ has slower start but faster overall convergence. The lack of convergence in 25,000 steps for  $ p_e \geq p_e^{\text{lim}} \approx 7/N$ shows that dense initialization, from the start is detrimental. 
            These results support the idea of adaptive sparsity policies that gradually increase density during training, consistent with curriculum learning approaches \cite{ProvableAdvantageofCurriculumLearningonParityTargetswithMixedInputs,AMathematicalModelforCurriculumLearningforParities}.
            
            \begin{figure}[htbp]
                \centerline{\includegraphics[width=225pt]{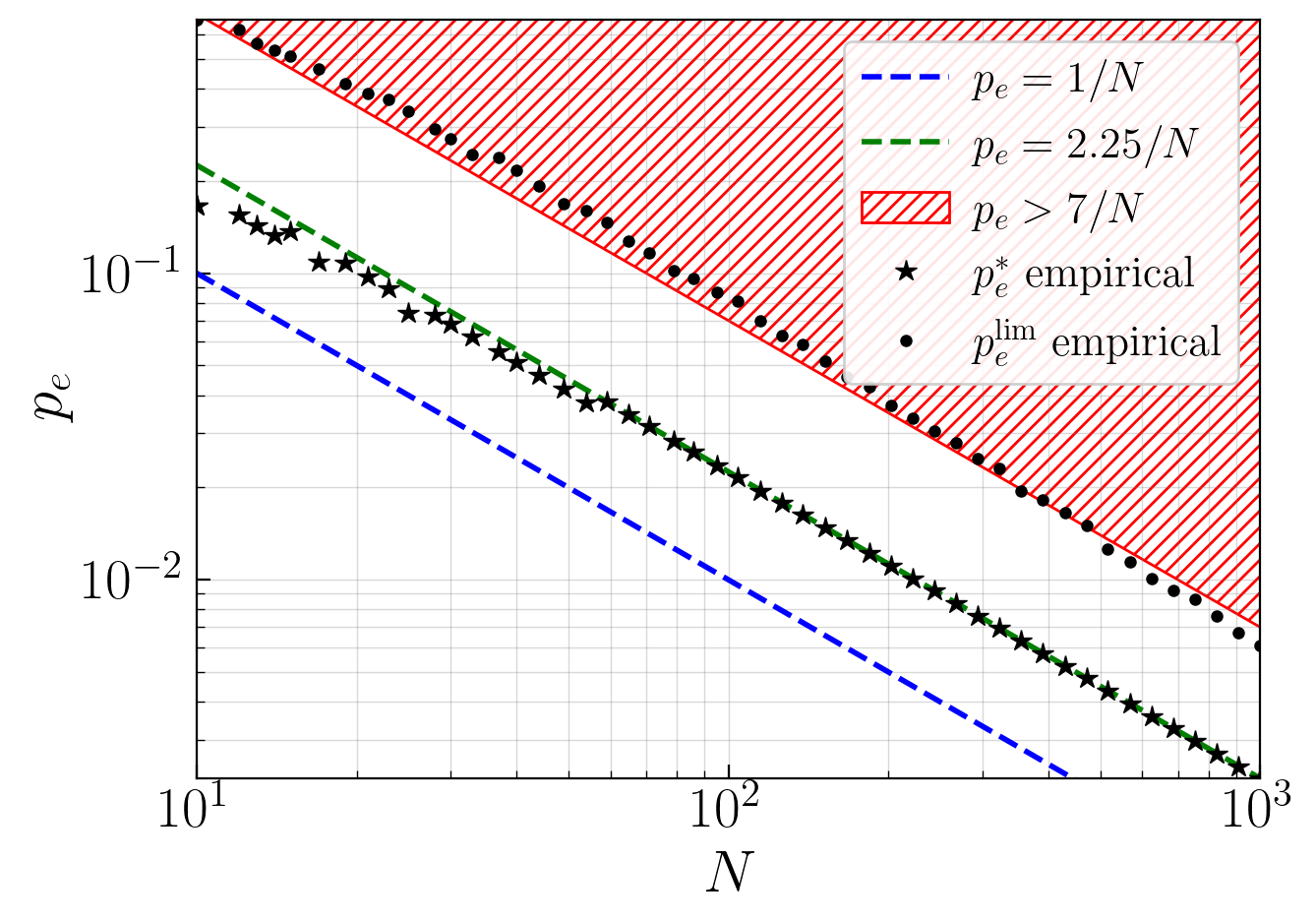}}
                \caption{Scaling of optimal and limit $p_e$. \textbf{Green}: empirical optimum $p_e^* \approx 2.25/N$. \textbf{Red}: convergence limit $p_e^{\text{lim}} \approx 7/N$. Consistent scaling relationships across problem sizes. \textit{Conditions}: $P=\lfloor1000/N\rfloor$, $\alpha=0.1$, $M=100$, $p_w=0.5$, $S=25{,}000$.}
                \label{fig:optimal_and_limit_p_e_for different_N}
            \end{figure}
            
            Figure~\ref{fig:optimal_and_limit_p_e_for different_N} extends the analysis of Figure~\ref{fig:convergence_as_a_function_of_p_e} by identifying $p_e^*$ and $p_e^{\text{lim}}$ for $N \in [10, 1000]$, providing insights regarding the optimal sparsity for a given size. The analysis confirms scaling relationships  $\propto 1/N$ with $p_e^* \approx 2.25/N$ and $p_e^{\text{lim}} \approx 7/N$ across several orders of magnitude. Notably, while our theory guarantees that $p_e \leq 1/N$ ensures convergence, experiments show that convergence is observed in practice for significantly denser regimes.

        \subsubsection{Impact of the Learning Rate $\alpha$}\label{subsubsection:impact_of_learning_rate}
            We study learning rate impact on convergence and verify theoretical bounds from Propositions~\ref{prop:sigma_convergence}, \ref{prop:monotonic_constraint}, and \ref{prop:bounded_convergence_interval}. Our analysis reveals that convergence time decreases linearly with $\alpha$ until a critical threshold $\alpha^{\text{lim}}$ where training fails, with optimal performance achieved at the maximum stable learning rate $\alpha^*$ just below this threshold (see Appendix~\ref{appendix:learning_rate_analysis}). Figure~\ref{fig:optimal_and_limit_alpha_for_different_N} reports the measured values of $\alpha^{\text{lim}}$ and $\alpha^*$ and provides evidence of a linear relationship  $\alpha^{\text{lim}} \approx N/2$ across $N \in [10, 1000]$ that aligns with theory: since effective learning rate scales as $\alpha p_e = \alpha/N$ (when $p_e = 1/N$), maintaining constant effective learning requires $\alpha \propto N$, as hinted by Gradient Eq.~\eqref{eq:symmetric_gradient}. Large batch sizes ensure stable statistical behavior, which we posit becomes increasingly important for the higher learning rates permitted as $\alpha^{\text{lim}}$ grows with larger $N$. Extended analysis in Appendix~\ref{app:sample_complexity_convergence_speed} confirms that convergence speed depends on the effective learning rate $\alpha p_e$, demonstrating scale-invariant training when hyper-parameters are properly scaled across problem sizes.
            
            We observe that the theoretical bounds correctly predict the convergence region: all rates below the most restrictive bound, $\alpha_2$, achieve convergence, with $\alpha^{\text{lim}} > \alpha_2$ confirming theoretical validity. 
               
            \begin{figure}[htbp]
                \centerline{\includegraphics[width=225pt]{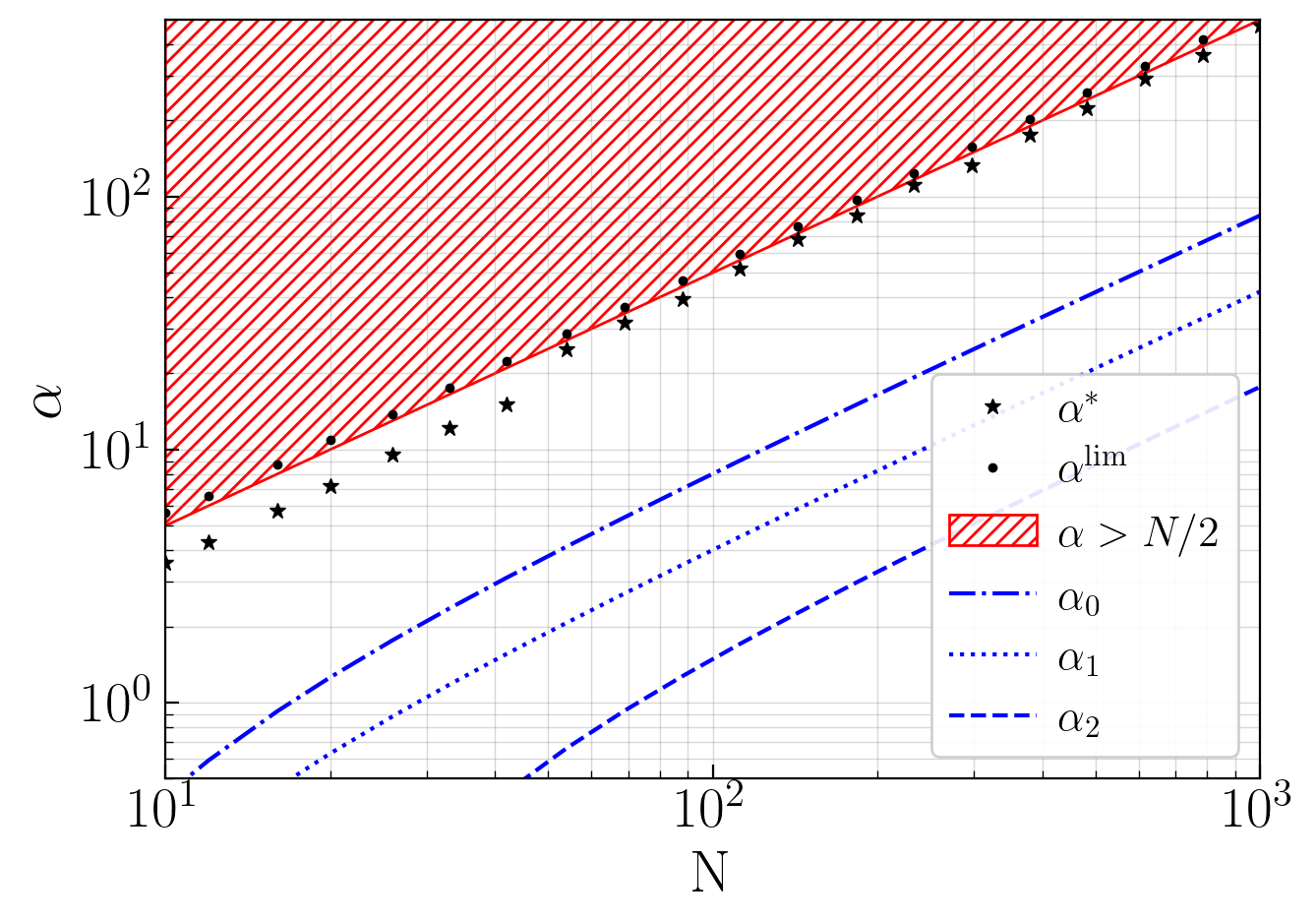}}
                \caption{Scaling of learning rate limits across problem sizes. Empirical limit $\alpha^{\text{lim}} \approx N/2$ shows linear scaling. Theoretical bounds scale consistently, with $\alpha_2 \approx \alpha^{\text{lim}}/20$ providing conservative but valid guarantees. Large batch ensures stable dynamics by maintaining consistent proportions of informative samples per batch (see Appendix~\ref{app:sample_complexity_convergence_speed}). \textit{Conditions}: $P=\lfloor1000/N\rfloor$, $p_e=1/N$, $M=50{,}000$, $p_w=0.5$, $S=30{,}000$.}
                \label{fig:optimal_and_limit_alpha_for_different_N}
            \end{figure}

        \subsubsection{Generalization \& Comparison with MLP}
            \begin{figure*}[htbp]
                \centerline{\includegraphics[width=470pt]{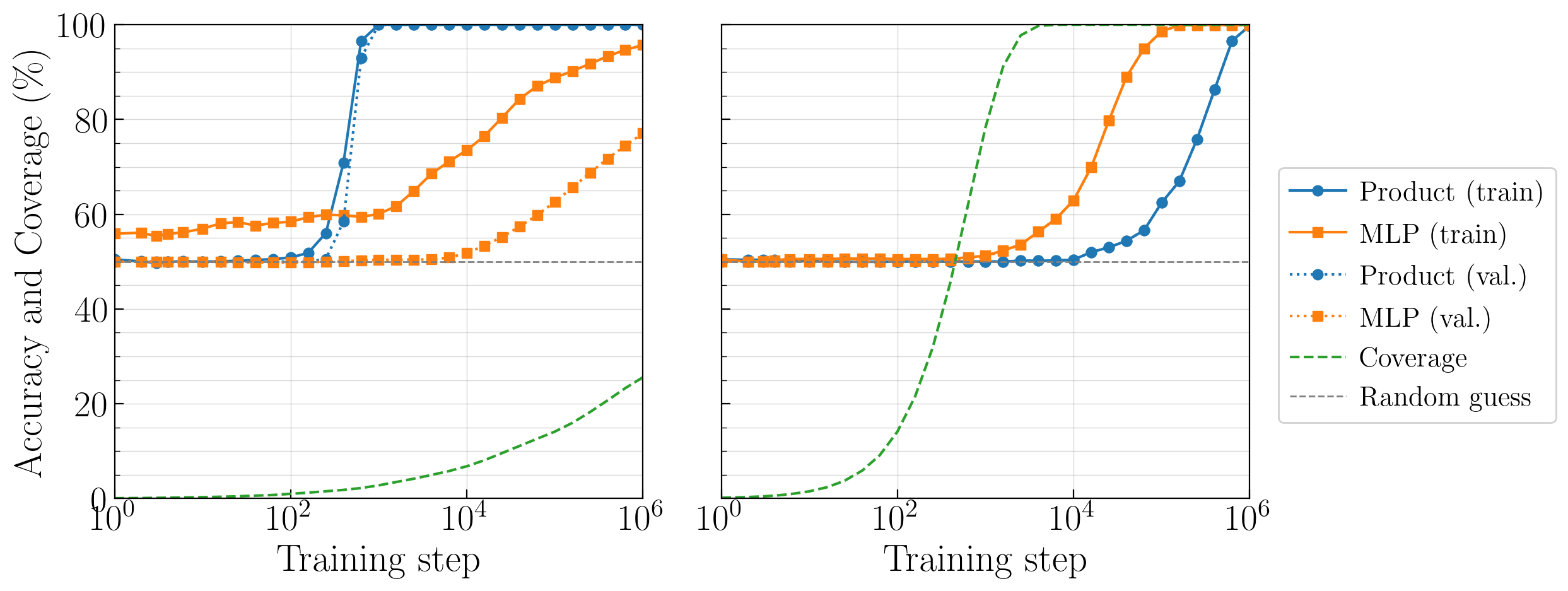}}
                \caption{Comparison of product node and MLP training dynamics under sparse and uniform sampling regimes. \textbf{(Left) Sparse regime:} training samples follow the unit-sparsity distribution with $p_e = 1/N$. \textbf{(Right) Uniform regime:} training samples are drawn uniformly from the complete truth table of size $2^N$. In both panels, training accuracy is computed over all truth table entries seen at least once since the beginning of training, validation accuracy over the complete truth table, and coverage denotes the fraction of the truth table observed during training. \textit{Conditions}: $N=16$, $P=250$, $\alpha=0.02$, $M=100$, $p_w=0.5$, $S=1{,}000{,}000$.}
                \label{fig:comparison_with_MLP}
            \end{figure*}
            A central claim of this work is that the product node's inductive bias enables generalization from sparse observations, a property standard MLPs lack without additional assumptions (see Figure~\ref{fig:xor_scenarios}). To test this, we compare our product node against the MLP architecture from~\cite{LearningHighDegreeParitiesTheCrucialRoleoftheInitialization} — three hidden layers $(512, 512, 64)$ with ReLU activations\footnote{Product node has $N = 16$ params. vs $\approx300k$ for the MLP.} — under identical conditions: plain SGD, $N = 16$ ($2^{16} = 65{,}536$ truth table entries), averaged over $250$ independent models. Figure~\ref{fig:comparison_with_MLP} reports training accuracy, validation accuracy (full truth table), and coverage (proportion of the truth table seen during training), under two regimes.
            
            Under the \textbf{sparse regime} ($p_e = 1/N$), the product node converges to $100\%$ validation accuracy within $1{,}000$ steps at less than $2\%$ truth table coverage, confirming that identifying the support suffices for full generalization across all $2^N$ inputs, by algebraic construction. The MLP instead overfits: training accuracy quickly rises while validation accuracy stagnates near $50\%$ (random guess), only starting to improve once a large fraction of the truth table is observed.
            
            Under \textbf{uniform sampling} of the truth table, the MLP eventually converges but only after exhaustive truth table coverage, suggesting that the MLP requires exhaustive distributional exposure to generalize. The product node converges more slowly, consistent with sparsity acting as a key enabler of efficient optimization. 
            
            The gap between the $\mathcal{O}(N)$ sparse samples sufficient for the product node and the $\mathcal{O}(2^N)$ samples required by the MLP widens exponentially with dimension, becoming quickly intractable for large $N$.


\section{Conclusion}

    This work demonstrates that high‑dimensional subset parity learning, long considered intractable for gradient‑based methods under standard architectures and data distributions, can be efficiently learned by combining architectural and data‑distribution biases.

    Architecturally, we introduce a compact product‑based XOR node with a neutral element that matches the algebraic structure of parity. Supposing proper parity subset identification, this node generalizes by construction across the full truth table, avoiding additive networks' exponential sample requirements in the general setting. On the data side, we consider Bernoulli distributed inputs which, combined with symmetric Gaussian weight distributions, yield affine dynamics preserving Gaussianity and symmetry. We show that $p_e = 1/N$ maximizes initial gradient signal and therefore analyze optimization under this stochastic unit-sparsity regime.

    Within this framework, we prove that for learning rates $\alpha < \alpha_2$ where $\alpha_2=\mathcal{O}(N)$, $N\geq18$, weight variances converge geometrically to zero while means converge to a bounded interval around the target values. The convergence interval scales as $\mathcal{O}(1/N) + \mathcal{O}(\alpha/N)$, vanishing as $N \to \infty$. Thus, gradient descent reliably recovers the underlying subset parity. 
        
    Experiments up to $N = 100{,}000$ validate the theory, showing preservation of Gaussian weight families, invariance to target weight proportions, and empirical scaling laws: $\alpha^{\mathrm{lim}}, \alpha^\ast \propto N$ and $p_e^{\mathrm{lim}}, p_e^\ast \propto 1/N$. While larger problems require increased sparsity, maintaining constant effective learning rate $\alpha p_e \propto C$ ensures consistent convergence speed across sizes.    
    
    Beyond parities and direct applications to error correcting codes, these results suggest a broader principle: for multiplicative or Boolean computations, product‑based architectures with sparse, structured inputs can transform traditionally difficult problems into ones amenable to standard gradient methods, enabling scalable integration of such computational capacities into deep learning architectures.
    

\clearpage
\newpage


\section*{Acknowledgments}

\noindent\textit{This work is dedicated to the memory of Eric Barelli, colleague and friend.}

\section*{Impact Statement}
    This paper presents work whose goal is to advance the field of machine learning. There are many potential societal consequences of our work, none of which we feel must be specifically highlighted here.

\bibliography{references}

@INPROCEEDINGS{
    SolvingParityNProblemsWithFeedForwardNeuralNetworks,  
    author={Wilamowski, B.M. and Hunter, D. and Malinowski, A.},  
    booktitle={Proceedings of the International Joint Conference on Neural Networks},   
    title={Solving {P}arity-{N} {P}roblems with {F}eedforward {N}eural {N}etworks},   
    year={2003},  
    doi={10.1109/IJCNN.2003.1223966}
}

@ARTICLE{SolvingNbitParityProblemUsingNN,
    title = {Solving the {N}-bit {P}arity {P}roblem using {N}eural {N}etworks},
    journal = {Neural Networks},
    volume = {12},
    number = {9},
    pages = {1321-1323},
    year = {1999},
    doi = {10.1016/S0893-6080(99)00069-6},
    author = {Myron E Hohil and Derong Liu and Stanley H Smith},
}

@INPROCEEDINGS{SPN,  
    author={Poon, Hoifung and Domingos, Pedro},  
    booktitle={2011 IEEE International Conference on Computer Vision Workshops},   
    title={Sum-{P}roduct {N}etworks: {A} {N}ew {D}eep {A}rchitecture},   
    year={2011},  
    volume={},  
    number={},  
    doi={10.1109/ICCVW.2011.6130310}
}

@INPROCEEDINGS{BlindNeuralBeliefPropagation,  
    author={Larue, Guillaume and Dufrene, Louis-Adrien and Lampin, Quentin and Chollet, Paul and Ghauch, Hadi and Rekaya, Ghaya},  
    booktitle={2021 Joint European Conference on Networks and Communications   6G Summit (EuCNC/6G Summit)},   
    title={Blind {N}eural {B}elief {P}ropagation {D}ecoder for {L}inear {B}lock {C}odes},   
    year={2021},  
    volume={},  number={},  
    pages={106-111},  
    doi={10.1109/EuCNC/6GSummit51104.2021.9482479}
}

@ARTICLE{LearningLinearBlockCodesWithGradientQuantization,
  author={Dufrène, Louis-Adrien and Lampin, Quentin and Larue, Guillaume},
  journal={IEEE Transactions on Communications}, 
  title={Learning {L}inear {B}lock {C}odes {W}ith {G}radient {Q}uantization}, 
  year={2025},
  volume={73},
  number={12},
  pages={13103-13116},
  doi={10.1109/TCOMM.2025.3615681}}

@ARTICLE{NeuralBeliefPropagationAutoEncoderforLinearBlockCodeDesign,
  author={Larue, Guillaume and Dufrene, Louis-Adrien and Lampin, Quentin and Ghauch, Hadi and Othman, Ghaya Rekaya-Ben},
  journal={IEEE Transactions on Communications}, 
  title={Neural {B}elief {P}ropagation {A}uto-{E}ncoder for {L}inear {B}lock {C}ode {D}esign}, 
  year={2022},
  volume={70},
  number={11},
  pages={7250-7264},
  doi={10.1109/TCOMM.2022.3208331}}

@misc{StatisticalQueriesandStatisticalAlgorithmsFoundationsandApplications,
  title={Statistical {Q}ueries and {S}tatistical {A}lgorithms: {F}oundations and {A}pplications},
  author={Reyzin, Lev},
  year={2020},
  url={https://arxiv.org/abs/2004.00557}
}

@inproceedings{LearningHighDegreeParitiesTheCrucialRoleoftheInitialization,
    title={Learning {H}igh-{D}egree {P}arities: {T}he {C}rucial {R}ole of the {I}nitialization},
    author={Emmanuel Abbe and Elisabetta Cornacchia and Jan H{\k{a}}z{\l}a and Donald Kougang-Yombi},
    booktitle={The Thirteenth International Conference on Learning Representations},
    year={2025},
    url={https://openreview.net/forum?id=OuNIWgGGif}
}

@inproceedings{
    ProvableAdvantageofCurriculumLearningonParityTargetswithMixedInputs,
    title={Provable {A}dvantage of {C}urriculum {L}earning on {P}arity {T}argets with {M}ixed {I}nputs},
    author={Emmanuel Abbe and Elisabetta Cornacchia and Aryo Lotfi},
    booktitle={Thirty-seventh Conference on Neural Information Processing Systems},
    year={2023},
    url={https://openreview.net/forum?id=9Ihu0VBOTq}
}

@InProceedings{AMathematicalModelforCurriculumLearningforParities,
  title = 	 {A {M}athematical {M}odel for {C}urriculum {L}earning for {P}arities},
  author =       {Cornacchia, Elisabetta and Mossel, Elchanan},
  booktitle = 	 {Proceedings of the 40th International Conference on Machine Learning},
  year = 	 {2023},
  url = 	 {https://proceedings.mlr.press/v202/cornacchia23a.html},
}

@inproceedings{
    NeuralArithmeticUnits,
    title={Neural {A}rithmetic {U}nits},
    author={Andreas Madsen and Alexander Rosenberg Johansen},
    booktitle={International Conference on Learning Representations},
    year={2020},
    url={https://openreview.net/forum?id=H1gNOeHKPS}
}

@inproceedings{NeuralArithmeticLogicUnits,
 author = {Trask, Andrew and Hill, Felix and Reed, Scott E and Rae, Jack and Dyer, Chris and Blunsom, Phil},
 booktitle = {Advances in Neural Information Processing Systems},
 title = {{Neural Arithmetic Logic Units}},
 url = {https://proceedings.neurips.cc/paper_files/paper/2018/file/0e64a7b00c83e3d22ce6b3acf2c582b6-Paper.pdf},
 volume = {31},
 year = {2018}
}

@inproceedings{LearningParitieswithNeuralNetworks,
 author = {Daniely, Amit and Malach, Eran},
 booktitle = {Advances in {N}eural {I}nformation {P}rocessing {S}ystems},
 title = {Learning {P}arities with {N}eural {N}etworks},
 url = {https://proceedings.neurips.cc/paper_files/paper/2020/file/eaae5e04a259d09af85c108fe4d7dd0c-Paper.pdf},
 year = {2020}
}

@inproceedings{
  AreEfficientDeepRepresentationsLearnable,
  title={Are {E}fficient {D}eep {R}epresentations {L}earnable?},
  author={Maxwell Nye and Andrew Saxe},
  booktitle={International Conference on Learning Representations Workshop},
  year={2018},
  url={https://openreview.net/forum?id=B1HI4FyvM}
}

@book{PerceptronsAnIntroductiontoComputationalGeometry,
  title={Perceptrons: {A}n {I}ntroduction to {C}omputational {G}eometry},
  author={Minsky, Marvin and Papert, Seymour},
  year={1969},
  publisher={MIT Press},
  address={Cambridge, MA}
}

@article{NBitParityNeuralNetworksNewSolutionsBasedonLinearProgramming,
    title   = {N-bit {P}arity {N}eural {N}etworks: {N}ew {S}olutions based on {L}inear {P}rogramming},
    author  = {Liu, Derong and Hohil, Myron E. and Smith, Stanley H.},
    journal = {Neurocomputing},
    volume  = {48},
    number  = {1--4},
    pages   = {477--488},
    year    = {2002},
    doi     = {10.1016/S0925-2312(01)00612-9}
    }

@article{ParitywithTwoLayerFeedforwardNets,
    title   = {Parity with {T}wo {L}ayer {F}eedforward {N}ets},
    author  = {Minor, J. M.},
    journal = {Neural Networks},
    volume  = {6},
    number  = {5},
    pages   = {705--707},
    year    = {1993},
    doi     = {10.1016/S0893-6080(05)80114-5}
    }

@ARTICLE{DeepLearningMethodsforImprovedDecodingofLinearCodes,
  author={Nachmani, Eliya and Marciano, Elad and Lugosch, Loren and Gross, Warren J. and Burshtein, David and Be’ery, Yair},
  journal={IEEE Journal of Selected Topics in Signal Processing}, 
  title={Deep {L}earning {M}ethods for {I}mproved {D}ecoding of {L}inear {C}odes}, 
  year={2018},
  volume={12},
  number={1},
  pages={119-131},
  doi={10.1109/JSTSP.2017.2788405}}

@ARTICLE{SumProductNetworksASurvey,
    author={Sanchez-Cauce, Raquel and Paris, Iago and Diez, Francisco Javier},
    journal={IEEE Transactions on Pattern Analysis \& Machine Intelligence},
    title={{Sum-Product Networks: A Survey}},
    year={2022},
    volume={44},
    number={07},
    pages={3821-3839},
    doi={10.1109/TPAMI.2021.3061898},
    publisher={IEEE Computer Society},
}

@misc{ProvableLimitationsofDeepLearning,
      title={{P}rovable {L}imitations of {D}eep {L}earning}, 
      author={Emmanuel Abbe and Colin Sandon},
      year={2019},
      eprint={1812.06369},
      archivePrefix={arXiv},
      primaryClass={cs.LG},
      url={https://arxiv.org/abs/1812.06369}, 
}

@book{DigitalSignalProcessingPrinciplesAlgorithmsandApplications,
  title={{D}igital {S}ignal {P}rocessing: {P}rinciples, {A}lgorithms, and {A}pplications},
  author={John G. Proakis and Dimitris G. Manolakis},
  year={1992}
}

@misc{FactorGraphOptimizationofErrorCorrectingCodesforBeliefPropagationDecoding,
      title={{F}actor {G}raph {O}ptimization of {E}rror-{C}orrecting {C}odes for {B}elief {P}ropagation {D}ecoding}, 
      author={Yoni Choukroun and Lior Wolf},
      year={2024},
      eprint={2406.12900},
      archivePrefix={arXiv},
      primaryClass={cs.IT},
      url={https://arxiv.org/abs/2406.12900}, 
}

\newpage
\appendix
\onecolumn


\section*{Appendix Guide}
\addcontentsline{toc}{section}{Appendix Guide}

The main paper presents a self-contained proof sketch providing the key ideas and logical chain establishing convergence. The appendices complement this by providing complete proofs for readers seeking detailed mathematical results, but are not required for a general understanding of the results. The appendices follow a logical progression establishing why and how parities become learnable under stochastic sparsity and product-based architectures. 

\textbf{Appendix~\ref{app:notations}} (p. \pageref{app:notations}) defines all notations.

\textit{Readers primarily interested in the convergence proof may skip directly to Appendix~\ref{app:probabilistic_sparsity_analysis}.}

\textbf{Appendices~\ref{app:sum_node_analysis} 
and~\ref{app:product_node_analysis}} 
(p.~\pageref{app:sum_node_analysis}--\pageref{app:product_with_neutral_element_node_analysis}) 
analyze sum nodes (convex baseline) and naive product nodes (non-convex) in a pedagogical progression intended to build intuition and motivate the architectural choices made in the convergence proof. These appendices are not required for the convergence result and may be skipped by readers already familiar with these architectures.

\textbf{Appendices~\ref{app:product_with_neutral_element_node_analysis} 
and~\ref{app:xor_node_analysis}} 
(p.~\pageref{app:product_with_neutral_element_node_analysis}--\pageref{app:probabilistic_sparsity_analysis}) 
introduce and analyze the core architectural building block studied throughout this paper: product nodes with neutral element, and their special case XOR nodes. A key property established here is that while these nodes are non-convex in general, they become strictly convex under sparse inputs (Prop.~\ref{prop:sparse_convexity}), with XOR nodes inheriting this result (Prop.~\ref{prop:xor_sparse_data}).

\textbf{Appendix~\ref{app:probabilistic_sparsity_analysis}} (p. \pageref{app:probabilistic_sparsity_analysis}--\pageref{app:exp_results}) is the central appendix, establishing convergence under stochastic sparsity. In particular, the reader will find:
\begin{enumerate}
    \item Expected gradient under Bernoulli inputs and optimal sparsity (Props.~\ref{prop:xor_node_gradient_stochastic_sparsity}--\ref{prop:optimal_error_probability});
    \item Gaussian distributional framework introducing preservation of Gaussianity and symmetry of weights 
    (Props.~\ref{prop:gaussian_gradient_expression}--\ref{prop:symmetry_conservation});
    \item Variance convergence and sufficient learning rate conditions (Props.~\ref{prop:sigma_convergence}--\ref{prop:monotonic_constraint});
    \item Mean convergence through fixed point analysis ensuring domain invariance and bounded convergence interval 
    (Props.~\ref{prop:envelope_characterization}--\ref{prop:bounded_convergence_interval});
    \item Closure providing the synthesis of all components into the complete convergence guarantee (Prop.~\ref{prop:complete_convergence_closure}).
\end{enumerate}

\textbf{Appendix~\ref{app:exp_results}} (p. \pageref{app:exp_results}--\pageref{app:supporting_results}) provides supplementary experimental results complementary to that of Section~\ref{section:experimental_validation}. 

\textbf{Appendix~\ref{app:supporting_results}} (p. \pageref{app:supporting_results}--\pageref{app:end}) provides the mathematical foundations used in the convergence proof, notably the exponential approximation (Prop.~\ref{prop:exp_approx}) and variable-coefficient fixed point dynamics (Thm.~\ref{thm:complete_variable_dynamics}) referenced throughout 
Appendix~\ref{app:probabilistic_sparsity_analysis}.

\section{Notations}\label{app:notations}

    \begin{table}[h]
        \begin{tabularx}{\columnwidth}{c|X}
            $s$ & Scalar \\
            $\mathbfit{v}$ & Vector (column) \\
            $\mathbfit{M}$ & Matrix (Denominator layout is considered for vector/matrix operations)  \\
            $\mathrm{s}$ & Random scalar variable \\
            $\mathbf{v}$ & Random vector variable (column) \\
            $\mathbf{M}$ & Random matrix variable \\
            $\mathbfit{M}^T$ & Transpose of matrix or vector\\
            $\mathbfit{M} \succ 0$ & Matrix $\mathbfit{M}$ is Positive Definite (PD)\\
            $\mathbfit{M} \succeq 0$ & Matrix $\mathbfit{M}$ is Positive Semi-Definite (PSD)\\
            $\mathbfit{I}_n$ & Identity matrix with $n$ rows and $n$ columns. When not specified, $n$ is implied by context \\
            $\mathbf{0}$ and $\mathbf{1}$ & denotes the all-zero and all-one vectors. More generally, all bold face number denotes a vector (or any n-dimensional tensor) whose value is constant across all index (\textit{e.g.} $\mathbf{2}$ is the all-two vector)\\
            $\mathbf{e}^{(k)}$ & The k-th standard basis vector, \textit{i.e.} a one hot vector with a one at index $k$\\
            $\mathrm{s} \sim D$ & Random variable $\mathrm{s}$ follows distribution $D$ \\
            $\mathbb{E}_{\mathrm{x}\sim D}\left\{\mathrm{x}\right\}$ & Expectation of the random variable $\mathrm{x}$ following distribution $D$\\             
            $\mathbb{S}$ & A set \\

        \end{tabularx}
    \end{table}

    \clearpage
    \newpage
    \section{Analysis of Sum Node}\label{app:sum_node_analysis}
        We recall the considered supervised learning setting with training set $\mathbb{S} = \{(\mathbfit{x}_m, y_m^{\mathrm{true}})\}_{m=1}^M$ where $\mathbfit{x}_m \in \mathbb{R}^N$ are input vectors and $y_m^{\mathrm{true}} \in \mathbb{R}$ are target outputs. Our model $f(\mathbfit{w}; \mathbfit{x}_m)$ is parameterized by weights $\mathbfit{w} \in \mathbb{R}^N$. We consider an oracle-based setup where the true labels are generated by applying the same function $f$ to the inputs using dedicated oracle parameters $\mathbfit{w}^{\mathrm{true}}$, unknown to the model to be trained:
        \begin{equation}
            y_m^{\mathrm{true}} = f(\mathbfit{w}^{\mathrm{true}}; \mathbfit{x}_m), \quad \forall m \in \{1,\ldots,M\}
        \end{equation}
        
        The learning objective is to minimize the empirical risk:
        \begin{equation}
            \hat{\ell}(\mathbfit{w}) = \frac{1}{M} \sum_{m=1}^M l(f(\mathbfit{w}; \mathbfit{x}_m), y_m^{\mathrm{true}})
        \end{equation}

        In this section, we consider the scenario depicted in Fig.\ref{fig:sum_unit}: a single neuron performs a weighted sum of its input. This scenario serves as a starting point to introduce both the notations and analytical methodology.
        
        \begin{figure}[htbp]
    \centerline{
        \begin{tikzpicture}
            \tikzmath{
                function PlotArrow (\ArrowStartX, \ArrowStartY, \CircleCenterX, \CircleCenterY, \CircleRadius) {
                    \ArrowEndX = \CircleCenterX - \CircleRadius*cos(atan((\ArrowStartY -  \CircleCenterY)/(\ArrowStartX - \CircleCenterX)));
                    \ArrowEndY = \CircleCenterY - \CircleRadius*sin(atan((\ArrowStartY - \CircleCenterY)/(\ArrowStartX - \CircleCenterX)));
                    {\draw[thick, -latex] (\ArrowStartX,\ArrowStartY) -- (\ArrowEndX,\ArrowEndY);
                    };
                };
                \CircleRadius1 = 0.5;
                \CircleCenterX1 = 0;
                \CircleCenterY1 = 0;
                \ArrowStartX1 = -3;
                \ArrowStartY1 = +2;
                \WeightPositionX1 = (\ArrowStartX1+\CircleCenterX1)/2;
                \WeightPositionY1 = (\ArrowStartY1+\CircleCenterY1)/2;
                PlotArrow(\ArrowStartX1, \ArrowStartY1, \CircleCenterX1, \CircleCenterY1, \CircleRadius1);
                \ArrowStartX2 = -3;
                \ArrowStartY2 = +1;
                \WeightPositionX2 = (\ArrowStartX2+\CircleCenterX1)/2;
                \WeightPositionY2 = (\ArrowStartY2+\CircleCenterY1)/2;
                PlotArrow(\ArrowStartX2, \ArrowStartY2, \CircleCenterX1, \CircleCenterY1, \CircleRadius1);
                \ArrowStartX3 = -3;
                \ArrowStartY3 = -2;
                \WeightPositionX3 = (\ArrowStartX3+\CircleCenterX1)/2;
                \WeightPositionY3 = (\ArrowStartY3+\CircleCenterY1)/2;
                PlotArrow(\ArrowStartX3, \ArrowStartY3, \CircleCenterX1, \CircleCenterY1, \CircleRadius1);
            };
            \draw[thick, fill=yellow!5] (\CircleCenterX1,\CircleCenterY1) circle (\CircleRadius1);
            
            \draw[thick] (\CircleCenterX1 - \CircleRadius1,\CircleCenterY1) -- (\CircleCenterX1 + \CircleRadius1,\CircleCenterY1);
            
            \draw[thick] (\CircleCenterX1,\CircleCenterY1-\CircleRadius1) -- (\CircleCenterX1,\CircleCenterY1+\CircleRadius1);
            
            \node at (\CircleCenterX1,\CircleCenterY1 + 2*\CircleRadius1) {SUM};
            \draw[thick,-latex] (\CircleCenterX1 + \CircleRadius1,\CircleCenterY1) -- (\CircleCenterX1 + \CircleRadius1 + 1,\CircleCenterY1);

            \node at (\CircleCenterX1 + \CircleRadius1 + 3,\CircleCenterY1) {$y_m = f(\mathbfit{w};\mathbfit{x}_m)$};
            
            \node at (\CircleCenterX1 + \CircleRadius1 + 3,\ArrowStartY3-1) {$y_m \in \mathbb{R}$};
            
            \node at (\CircleCenterX1 + \CircleRadius1 + 3,\ArrowStartY3) {$y_m = \sum_{i=1}^N w_ix_{m,i}$};
            
            \node at (\ArrowStartX1-0.5,\ArrowStartY1+0.2) {$x_{m,1}$};
            \node at (\ArrowStartX2-0.5,\ArrowStartY2+0.2) {$x_{m,2}$};
            \node at (\ArrowStartX3-0.5,\ArrowStartY3-0.2) {$x_{m,N}$};
            \node at (\ArrowStartX3-0.5,\CircleCenterY1) {$\vdots$};
            \node at (\ArrowStartX3-0.5,\ArrowStartY3-1) {$\mathbfit{x}_m \in \mathbb{R}^N$};
            
            \node at (\WeightPositionX1,\WeightPositionY1+0.3) {$w_{1}$};
            \node at (\WeightPositionX2,\WeightPositionY2+0.2) {$w_{2}$};
            \node at (\WeightPositionX3,\WeightPositionY3-0.3) {$w_{N}$};
            \node at (\WeightPositionX3,\ArrowStartY3-1) {$\mathbfit{w} \in \mathbb{R}^N$};
            
        \end{tikzpicture}
    }
    \caption{The sum unit}
    \label{fig:sum_unit}
\end{figure}
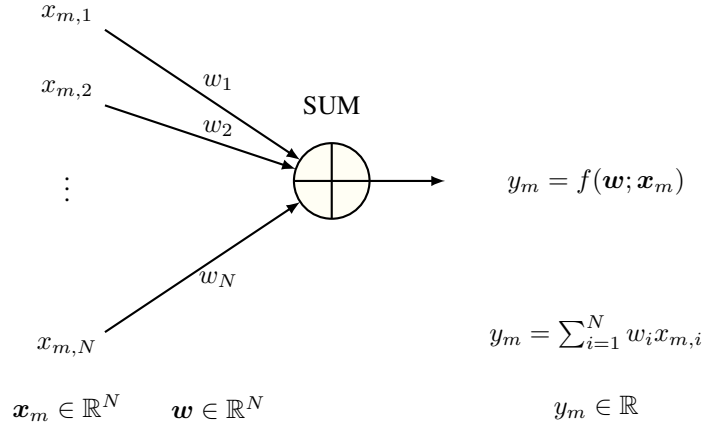
        
        The input is composed of vectors $\mathbfit{x}_m$ of dimension $N$, with $m\in [1,M]$. The neuron has a corresponding number of parameters, or weights, denoted in vector form $\mathbfit{w}$. In this configuration, the single neuron model performs a weighted sum of the components of $\mathbfit{x}_m$, denoted $x_{m,i}$ with $i\in[1,N]$:
        \begin{equation}
            y_m = f(\mathbfit{w};\mathbf{x}_m) = \sum_{i=1}^N w_ix_{m,i} \quad \forall m \in[1,M]
        \end{equation}

        Let us consider the task of learning the parameters $\mathbfit{w}$ given pairs of inputs and outputs ($\mathbfit{x_{m}}$, $y_m$). We consider the (unregularized) empirical risk with $\ell_2$ loss:
        \begin{equation}
            \hat{\ell}(\mathbfit{w}) = \frac{1}{M} \sum_{m=1}^M (y_m-y_m^{\mathrm{true}})^2
        \end{equation}

        \begin{proposition}[Sum Node - Convexity of Empirical Risk]\label{prop:empirical_convex}
            For a finite dataset $\mathbf{X} \in \mathbb{R}^{M \times N}$ with full column rank, the unregularized empirical risk
            $$\hat{\ell}(\mathbfit{w}) = \frac{1}{M}||\mathbfit{X}\mathbfit{w} - \mathbfit{X}\mathbfit{w}^{\mathrm{true}}||_2^2$$
            is strongly convex with unique global minimum at $\mathbfit{w}^{\mathrm{true}}$.
        \end{proposition}
        
        \begin{proof}
            The empirical risk can be written as: 
            \begin{equation}
                \hat{\ell}(\mathbfit{w}) = \frac{1}{M}[\mathbfit{w}^T\mathbfit{X}^T\mathbfit{X}\mathbfit{w} - 2\mathbfit{w}^T\mathbfit{X}^T\mathbfit{X}\mathbfit{w}^{\mathrm{true}} + (\mathbfit{w}^{\mathrm{true}})^T\mathbfit{X}^T\mathbfit{X}\mathbfit{w}^{\mathrm{true}}]
            \end{equation}
            
            The gradient is:
            \begin{equation}\label{eq:sum_node_gradient}
                \nabla_{\mathbfit{w}} \hat{\ell}(\mathbfit{w}) = \frac{2}{M}\mathbfit{X}^T\mathbfit{X}\left(\mathbfit{w} - \mathbfit{w}^{\mathrm{true}}\right)\propto\left(\mathbfit{w} - \mathbfit{w}^{\mathrm{true}}\right)
            \end{equation}
            
            The Hessian is:
            \begin{equation}
                \nabla^2_{\mathbfit{w}} \hat{\ell}(\mathbfit{w}) = \frac{2}{M}\mathbfit{X}^T\mathbfit{X}\succ0
            \end{equation}
            
            Since $\mathbf{X}$ has full column rank ($\text{rank}(\mathbf{X}) = N$), the matrix $\mathbf{X}^T\mathbf{X}$ is positive definite. Therefore, $\nabla^2_{\mathbfit{w}} \hat{\ell}(\mathbfit{w}) \succ \mathbf{0}$, establishing strong convexity. Setting $\nabla_{\mathbfit{w}} \hat{\ell}(\mathbfit{w}) = 0$ yields $\mathbfit{w} = \mathbfit{w}^{\mathrm{true}}$ as the unique global minimum. For i.i.d. non-degenerate input distributions, this full column rank condition typically holds with high probability when $M \gg N$.
        \end{proof}


        \section{Analysis of Naive Product Node}\label{app:product_node_analysis}
        
            \begin{figure}[htbp]
        \centerline{
            \begin{tikzpicture}
                \tikzmath{
                    function PlotArrow (\ArrowStartX, \ArrowStartY, \CircleCenterX, \CircleCenterY, \CircleRadius) {
                        \ArrowEndX = \CircleCenterX - \CircleRadius*cos(atan((\ArrowStartY -  \CircleCenterY)/(\ArrowStartX - \CircleCenterX)));
                        \ArrowEndY = \CircleCenterY - \CircleRadius*sin(atan((\ArrowStartY - \CircleCenterY)/(\ArrowStartX - \CircleCenterX)));
                        {\draw[thick, -latex] (\ArrowStartX,\ArrowStartY) -- (\ArrowEndX,\ArrowEndY);
                        };
                    };
                    \CircleRadius1 = 0.5;
                    \CircleCenterX1 = 0;
                    \CircleCenterY1 = 0;
                    \ArrowStartX1 = -4;
                    \ArrowStartY1 = +2;
                    \WeightPositionX1 = (\ArrowStartX1+\CircleCenterX1)/2;
                    \WeightPositionY1 = (\ArrowStartY1+\CircleCenterY1)/2;
                    PlotArrow(\ArrowStartX1, \ArrowStartY1, \CircleCenterX1, \CircleCenterY1, \CircleRadius1);
                    \ArrowStartX2 = -4;
                    \ArrowStartY2 = +1;
                    \WeightPositionX2 = (\ArrowStartX2+\CircleCenterX1)/2;
                    \WeightPositionY2 = (\ArrowStartY2+\CircleCenterY1)/2;
                    PlotArrow(\ArrowStartX2, \ArrowStartY2, \CircleCenterX1, \CircleCenterY1, \CircleRadius1);
                    \ArrowStartX3 = -4;
                    \ArrowStartY3 = -2;
                    \WeightPositionX3 = (\ArrowStartX3+\CircleCenterX1)/2;
                    \WeightPositionY3 = (\ArrowStartY3+\CircleCenterY1)/2;
                    PlotArrow(\ArrowStartX3, \ArrowStartY3, \CircleCenterX1, \CircleCenterY1, \CircleRadius1);
                };
                \draw[thick, fill=red!6] (\CircleCenterX1,\CircleCenterY1) circle (\CircleRadius1);
                
                \draw[thick] (\CircleCenterX1 - 0.707*\CircleRadius1, \CircleCenterY1+0.707*\CircleRadius1) -- (\CircleCenterX1 + 0.707*\CircleRadius1, \CircleCenterY1-0.707*\CircleRadius1);
                
                \draw[thick] (\CircleCenterX1 - 0.707*\CircleRadius1, \CircleCenterY1-0.707*\CircleRadius1) -- (\CircleCenterX1 + 0.707*\CircleRadius1, \CircleCenterY1               +0.707*\CircleRadius1);
                
                \node at (\CircleCenterX1,\CircleCenterY1 + 2*\CircleRadius1) {PRODUCT};
                \draw[thick,-latex] (\CircleCenterX1 + \CircleRadius1,\CircleCenterY1) -- (\CircleCenterX1 + \CircleRadius1 + 2,\CircleCenterY1);

                \node at (\CircleCenterX1 + \CircleRadius1 + 4,\CircleCenterY1) {$y_m = f(\mathbfit{w};\mathbfit{x}_m)$};
                
                \node at (\CircleCenterX1 + \CircleRadius1 + 4,\ArrowStartY3-1) {$y_m \in \mathbb{R}$};
                
                \node at (\CircleCenterX1 + \CircleRadius1 + 4,\ArrowStartY3) {$y_m = \prod_{i=1}^N w_ix_{m,i}$};
                
                \node at (\ArrowStartX1-0.5,\ArrowStartY1+0.2) {$x_{m,1}$};
                \node at (\ArrowStartX2-0.5,\ArrowStartY2+0.2) {$x_{m,2}$};
                \node at (\ArrowStartX3-0.5,\ArrowStartY3-0.2) {$x_{m,N}$};
                \node at (\ArrowStartX3-0.5,\CircleCenterY1) {$\vdots$};
                \node at (\ArrowStartX3-0.5,\ArrowStartY3-1) {$\mathbfit{x}_m \in \mathbb{R}^N$};
                
                \node at (\WeightPositionX1,\WeightPositionY1+0.3) {$w_{1}$};
                \node at (\WeightPositionX2,\WeightPositionY2+0.2) {$w_{2}$};
                \node at (\WeightPositionX3,\WeightPositionY3-0.3) {$w_{N}$};
                \node at (\WeightPositionX3,\ArrowStartY3-1) {$\mathbfit{w} \in \mathbb{R}^N$};
                
            \end{tikzpicture}
        }
        \caption{The naive product unit}
        \label{fig:prod_unit}
    \end{figure}
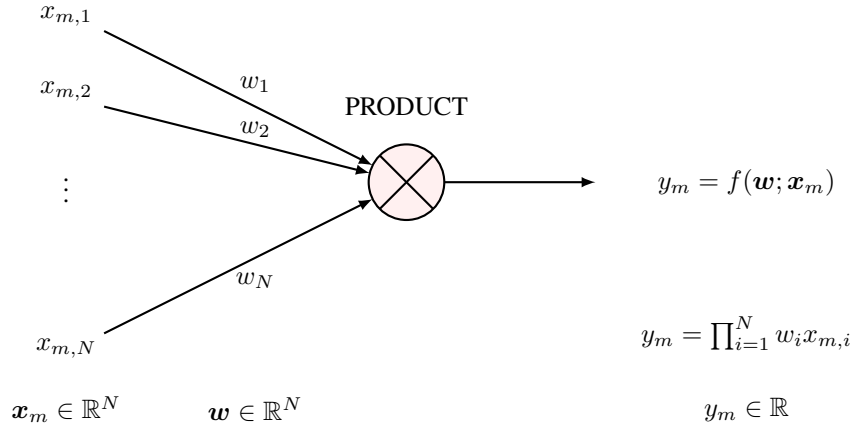
    
            Having established the convex nature of sum nodes, we now analyze the significantly more complex case of product nodes in the same general setting. This analysis reveals fundamental optimization challenges that motivate our structured approach. The naive product node computes a weighted product of inputs:
            \begin{equation} \label{eq:NP}
                y_m = f(\mathbfit{w};\mathbfit{x}_m) =\prod_{i=1}^{N} \big(w_i x_{m,i}\big) =  \Bigg(\prod_{i=1}^{N} w_i\Bigg) \times \Bigg( \prod_{i=1}^{N} x_{m,i}\Bigg)  \quad \forall m \in\{1,\ldots,M\}
            \end{equation}

            The empirical risk is:
            \begin{equation} \label{eq:loss_NP_scalar}
                \hat{\ell}_{\Pi}(\mathbfit{w}) = \frac{1}{M}\sum_{m=1}^{M}(y_m - y_m^{\mathrm{true}})^2
            \end{equation}
            
            The optimization problem is formulated as:
            \begin{align}
                (P_{\Pi}): \quad \hat{\mathbfit{w}}_{\Pi} := \arg\min_{\mathbfit{w} \in \mathbb{R}^N} \hat{\ell}_{\Pi}(\mathbfit{w})
            \end{align}
            
            \begin{proposition}[Gradient of Product Node Loss]\label{prop:product_gradient}
                For the naive product node with empirical risk $\hat{\ell}_{\Pi}(\mathbfit{w})$, the j-th component of the gradient $\nabla_{\mathbf{w}} \hat{\ell}_{\Pi}(\mathbf{w})$  is:
                \begin{equation}
                    \begin{aligned}
                        \frac{\partial \hat{\ell}_{\Pi}(\mathbfit{w}) }{\partial w_j} &= 2D\times \bigg(\prod_{\substack{i=1 \\ i \neq j}}^{N} w_i \bigg)\times\bigg(\prod_{i=1}^{N} w_i-\prod_{i=1}^{N} w_i^\mathrm{true}\bigg)\\
                    \end{aligned}
                \end{equation}
                where $D=\frac{1}{M}\sum_{m=1}^M\bigg( \prod_{i=1}^{N} x_{m,i}^2\bigg)$ 
            \end{proposition}
            
            \begin{proof}   
                The empirical risk with $\ell_2$ loss can be written as:
                \begin{equation}
                    \begin{aligned}
                        \hat{\ell}_{\Pi} (\mathbfit{w}) &= \frac{1}{M}\sum_{m=1}^M\Bigg[y_m-y_m^\mathrm{true}\Bigg]^2\\
                        &= \frac{1}{M}\sum_{m=1}^M\Bigg[\bigg(\prod_{i=1}^{N} w_i\bigg) \times \bigg( \prod_{i=1}^{N} x_{m,i}\bigg)-\bigg(\prod_{i=1}^{N} w_i^\mathrm{true}\bigg) \times \bigg( \prod_{i=1}^{N} x_{m,i}\bigg)\Bigg]^2\\
                        &= \frac{1}{M}\sum_{m=1}^M\Bigg[\bigg(\prod_{i=1}^{N} w_i-\prod_{i=1}^{N} w_i^\mathrm{true}\bigg) \times \bigg( \prod_{i=1}^{N} x_{m,i}\bigg)\Bigg]^2\\
                        &= \frac{1}{M}\sum_{m=1}^M\bigg( \prod_{i=1}^{N} x_{m,i}\bigg)^2 \times \bigg(\prod_{i=1}^{N} w_i-\prod_{i=1}^{N} w_i^\mathrm{true}\bigg)^2\\
                        &= \bigg(\prod_{i=1}^{N} w_i-\prod_{i=1}^{N} w_i^\mathrm{true}\bigg)^2 \times \frac{1}{M}\sum_{m=1}^M\bigg( \prod_{i=1}^{N} x_{m,i}^2\bigg)\\
                        &= \bigg(\prod_{i=1}^{N} w_i-\prod_{i=1}^{N} w_i^\mathrm{true}\bigg)^2 \times D\\
                    \end{aligned}
                \end{equation}
                where $D=\frac{1}{M}\sum_{m=1}^M\bigg( \prod_{i=1}^{N} x_{m,i}^2\bigg)$ is dataset dependant with, for i.i.d. variables:

                \begin{equation}
                    \mathbb{E}_{\mathrm{x}\sim X}\left\{D\right\}=\mathbb{E}_{\mathrm{x}\sim X}\left\{\frac{1}{M}\sum_{m=1}^M\bigg( \prod_{i=1}^{N} \mathrm{x}_{m,i}^2\bigg)\right\}=\mathbb{E}_{\mathrm{x}\sim X}\left\{\prod_{i=1}^{N} \mathrm{x}_{i}^2\right\}=\prod_{i=1}^{N} \left(\mathbb{E}_{\mathrm{x}\sim X}\left\{\mathrm{x}_{i}^2\right\}\right)=\mathbb{E}_{\mathrm{x}\sim X}^N\left\{\mathrm{x}^2\right\}
                \end{equation}

                Therefore, the j-th component of the gradient $\nabla_{\mathbf{w}} \hat{\ell}_{\Pi}(\mathbf{w})$ is:
                \begin{equation}
                    \begin{aligned}
                        \frac{\partial \hat{\ell}_{\Pi}(\mathbfit{w}) }{\partial w_j} &= 2D\times \bigg(\prod_{\substack{i=1 \\ i \neq j}}^{N} w_i \bigg)\times\bigg(\prod_{i=1}^{N} w_i-\prod_{i=1}^{N} w_i^\mathrm{true}\bigg)\\
                    \end{aligned}
                \end{equation}    
            \end{proof}
            \begin{proposition}[Hessian of Product Node Loss]\label{prop:product_hessian}
                
                The Hessian matrix of the empirical risk $\hat{\ell}_{\Pi}(\mathbfit{w})$ has entries:
                \begin{equation}
                    [\mathbfit{H}]_{j,l} = \frac{\partial^2 \hat{\ell}_{\Pi}(\mathbfit{w})}{\partial w_j \partial w_l} = 2D \times 
                    \left\{
                    \begin{aligned}
                        &\prod_{\substack{i=1 \\ i \neq j}}^{N} w_i^2 & \text{if } j = l \\
                        &\left(\prod_{\substack{i=1 \\ i \neq j \neq l}}^{N} w_i\right) \times \left[2\times\prod_{i=1}^{N} w_i - \prod_{i=1}^{N} w_i^{\mathrm{true}}\right] & \text{if } j \neq l
                    \end{aligned}
                    \right.
                \end{equation}
                where $D = \frac{1}{M}\sum_{m=1}^M \prod_{i=1}^N x_{m,i}^2$.
            \end{proposition}
            
            \begin{proof}
                The Hessian of the empirical risk is defined as:
                \begin{equation}
                    \mathbfit{H} = \nabla_{\mathbfit{w}}^2 \hat{\ell}_{\Pi}(\mathbfit{w}) = \frac{\partial}{\partial \mathbfit{w}}\nabla_{\mathbfit{w}} \hat{\ell}_{\Pi}(\mathbfit{w})
                \end{equation}
                
                Given the j-th component of the gradient (Proposition~\ref{prop:product_gradient}):
                \begin{equation}
                    \frac{\partial \hat{\ell}_{\Pi}(\mathbfit{w})}{\partial w_j} = 2D \times \left(\prod_{\substack{i=1 \\ i \neq j}}^{N} w_i\right) \times \left(\prod_{i=1}^{N} w_i - \prod_{i=1}^{N} w_i^{\mathrm{true}}\right)
                \end{equation}
                
                The Hessian entries are given by $[\mathbfit{H}]_{j,l} = \frac{\partial^2 \hat{\ell}_{\Pi}(\mathbfit{w})}{\partial w_j \partial w_l}$. Two cases arise:
                
                \textbf{Diagonal entries} ($j = l$):
                \begin{equation}
                    \frac{\partial^2 \hat{\ell}_{\Pi}(\mathbfit{w})}{\partial w_j^2} = \frac{\partial}{\partial w_j}\left[\frac{\partial \hat{\ell}_{\Pi}(\mathbfit{w})}{\partial w_j}\right]=2D \times \prod_{\substack{i=1 \\ i \neq j}}^{N} w_i^2
                \end{equation}
                
                \textbf{Off-diagonal entries }($j \neq l$):
                \begin{equation}
                    \begin{aligned}
                        \frac{\partial^2 \hat{\ell}_{\Pi}(\mathbfit{w})}{\partial w_j \partial w_l} &= \frac{\partial}{\partial w_l}\left[\frac{\partial \hat{\ell}_{\Pi}(\mathbfit{w})}{\partial w_j}\right] \\
                        &= \frac{\partial}{\partial w_l}\left[2D \times \left(\prod_{\substack{i=1 \\ i \neq j}}^{N} w_i\right) \times \left(\prod_{i=1}^{N} w_i - \prod_{i=1}^{N} w_i^{\mathrm{true}}\right)\right]\\
                        &= 2D\times \frac{\partial}{\partial w_l}\left[w_j\times \prod_{\substack{i=1 \\ i \neq j}}^{N} w_i^2 - \left(\prod_{\substack{i=1 \\ i \neq j}}^{N} w_i\right)\times\left(\prod_{i=1}^{N} w_i^{\mathrm{true}}\right)\right]\\
                        &= 2D\times \left[2w_jw_l\times \prod_{\substack{i=1 \\ i \neq j \neq l}}^{N} w_i^2 - \left(\prod_{\substack{i=1 \\ i \neq j \neq l}}^{N} w_i\right)\times\left(\prod_{i=1}^{N} w_i^{\mathrm{true}}\right)\right]\\
                        &= 2D\times \left(\prod_{\substack{i=1 \\ i \neq j \neq l}}^{N} w_i\right)\times \left[2w_jw_l\times \prod_{\substack{i=1 \\ i \neq j \neq l}}^{N} w_i - \prod_{i=1}^{N} w_i^{\mathrm{true}}\right]\\
                        &=2D\times \left(\prod_{\substack{i=1 \\ i \neq j \neq l}}^{N} w_i\right)\times \left[2\times \prod_{i=1}^{N} w_i - \prod_{i=1}^{N} w_i^{\mathrm{true}}\right]\\
                    \end{aligned}                    
                \end{equation}
            \end{proof}

            \begin{proposition}[Non-Convexity of Product Node]\label{prop:product_nonconvex}
                The naive product node optimization problem is non-convex. There exist parameter configurations where the Hessian matrix is not positive semi-definite.
            \end{proposition}
            
            \begin{proof}            
                Recall that $\mathbfit{H} \in \mathbb{R}^{N \times N}$ is PSD (PD respectively), iff $\forall \mathbfit{q} \in \mathbb{R}^N \setminus \{\mathbf{0}\}$ the following condition is satisfied: $\mathbfit{q}^T \mathbfit{H} \mathbfit{q} \geq 0$ ($\mathbfit{q}^T \mathbfit{H} \mathbfit{q} > 0$ respectively). Applying this condition to the above Hessian matrix, $\nabla_\mathbfit{w}^2\hat{\ell}_{\Pi_r}(\mathbfit{w})$ in Proposition~\ref{prop:product_hessian}, reveals that the PSD condition is not always satisfied.
      
                We demonstrate non-convexity by constructing a counterexample. Let $w_1 = 0, w_i = 1 \forall i \in \{2,...,N\}, w_i^\mathrm{true}=1 \forall i \in \{1,...,N\}$. Equivalently we can define $\mathbfit{w}=\mathbf{1}-\mathbfit{e}^{(1)}$ and $\mathbfit{w}^\mathrm{true}=\mathbf{1}$. We recall the Hessian matrix of the empirical risk of the naive product node (Proposition~\ref{prop:product_hessian}):

                \begin{equation}
                    [\mathbfit{H}]_{j,l} = 2D \times 
                    \left\{
                    \begin{aligned}
                        &\prod_{\substack{i=1 \\ i \neq j}}^{N} w_i^2 & \text{if } j = l \\
                        &\left(\prod_{\substack{i=1 \\ i \neq j \neq l}}^{N} w_i\right) \times \left[2\times\prod_{i=1}^{N} w_i - \prod_{i=1}^{N} w_i^{\mathrm{true}}\right] & \text{if } j \neq l
                    \end{aligned}
                    \right.
                \end{equation}

                where $\prod_{i=1}^{N} w_i=0$ and $\prod_{i=1}^{N} w_i^{\mathrm{true}}=1$. Hence, the Hessian can be simplified to:
                \begin{equation}
                    [\mathbfit{H}]_{j,l} = 2D \times \left\{
                    \begin{aligned}
                        \prod_{\substack{i=1 \\ i \neq j}}^{N} w_i^2 & \text{if } j = l \\
                        -\prod_{\substack{i=1 \\ i \neq j \neq l}}^{N} w_i & \text{if } j \neq l
                    \end{aligned}
                    \right.
                \end{equation}
                
                Which has the following values:
                \begin{equation}
                    [\mathbfit{H}]_{j,l} = \frac{\partial^2 \hat{\ell}_{\Pi}(\mathbfit{w})}{\partial w_j \partial w_l} = 2D \times \begin{cases}
                        1 & \text{if } j = l = 1\\
                        0 & \text{if } j = l \neq 1\\
                        -1  & \text{if } j \neq l \text{ and at least one of } j,l \text{ equals } 1\\
                        0  & \text{if } j \neq l \neq 1\\
                    \end{cases}
                \end{equation}

                The resulting Hessian matrix can be compactly noted as:
                \begin{equation}
                    \mathbfit{H}=2D\times
                    \begin{pmatrix}
                        1 & -1 & -1 & \cdots & -1\\
                        -1 & 0 & 0 & \cdots & 0\\
                        -1 & 0 & 0 & \cdots & 0\\
                        \vdots & \vdots & \vdots & \ddots&\vdots\\
                        -1 & 0 & 0 & \cdots&0\\
                    \end{pmatrix}
                    =2D\left[\mathbfit{e}^{(1)}(\mathbfit{e}^{(1)})^T-(\mathbf{1}-\mathbfit{e}^{(1)})(\mathbfit{e}^{(1)})^T-\mathbfit{e}^{(1)}(\mathbf{1}-\mathbfit{e}^{(1)})^T\right]
                \end{equation}
                or equivalently,
                \begin{equation}
                    \mathbfit{H}=2D\times\left[3\mathbfit{e}^{(1)}(\mathbfit{e}^{(1)})^T-\mathbf{1}(\mathbfit{e}^{(1)})^T-\mathbfit{e}^{(1)}\mathbf{1}^T\right]
                \end{equation}

                Multiplying by the test vector $\mathbfit{q}$:
                \begin{equation} 
                    \begin{aligned}
                        \mathbfit{q}^T\nabla_\mathbfit{w}^2 \hat{\ell}_{\Pi} (\mathbfit{w})\mathbfit{q}
                        =&2D \times \mathbfit{q}^T\left[3\mathbfit{e}^{(1)}(\mathbfit{e}^{(1)})^T - \mathbf{1}\cdot(\mathbfit{e}^{(1)})^T - \mathbfit{e}^{(1)}\cdot\mathbf{1}^T \right]\mathbfit{q}\\
                        =&2D \times \left[3\mathbfit{q}^T\mathbfit{e}^{(1)}(\mathbfit{e}^{(1)})^T   - \mathbfit{q}^T\left(\mathbf{1}\cdot(\mathbfit{e}^{(1)})^T\right) - \mathbfit{q}^T\left(\mathbfit{e}^{(1)}\cdot\mathbf{1}^T\right)\right]\mathbfit{q}\\
                        =&2D \times \left[3q_1(\mathbfit{e}^{(1)})^T - \mathbfit{q}^T\left(\mathbf{1}\cdot(\mathbfit{e}^{(1)})^T\right) - q_1\mathbf{1}^T\right]\mathbfit{q}\\
                        =&2D \times \left[3q_1(\mathbfit{e}^{(1)})^T\mathbfit{q} - \mathbfit{q}^T\left(\mathbf{1}\cdot(\mathbfit{e}^{(1)})^T\right)\mathbfit{q} - q_1\mathbf{1}^T\mathbfit{q}\right]\\
                        =&2D \times \left[3q_1(\mathbfit{e}^{(1)})^T\mathbfit{q} - q_1\mathbfit{q}^T\mathbf{1} - q_1\mathbf{1}^T\mathbfit{q}\right]\\
                        =&2D \times q_1\left[3(\mathbfit{e}^{(1)})^T\mathbfit{q} - 2\mathbfit{q}^T\mathbf{1}\right]
                    \end{aligned}
                \end{equation}
                
               For the test vector $\mathbfit{q} = \mathbf{1}$:
                \begin{equation} 
                    \begin{aligned}
                        \mathbfit{q}^T\nabla_\mathbfit{w}^2 \hat{\ell}_{\Pi} (\mathbfit{w})\mathbfit{q}
                        =&2D\left[3 - 2N\right]\\
                        \mathbfit{q}^T\nabla_\mathbfit{w}^2 \hat{\ell}_{\Pi} (\mathbfit{w})\mathbfit{q}
                        <& 0 \forall N \geq 2\\
                    \end{aligned}
                \end{equation}
                As a result, the PSD conditions is not satisfied for certain instances of $\mathbfit{w},\mathbfit{w}^\mathrm{true}$. Thus, the naive product unit is not always convex in $\mathbfit{w}$.

                This counterexample is interesting from an optimization perspective: even when all but one parameter have reached their optimal values, the Hessian matrix is not positive definite, revealing fundamental non-convexity despite being arbitrarily close to the global solution.
            \end{proof}

        The analysis reveals that the naive product node presents fundamental optimization challenges:
        \begin{itemize}
            \item \textbf{Non-convexity and non-separability:} The $j$-th gradient component depends on all other parameters in $\mathbfit{w}$, creating multiple local minima and preventing guaranteed global optimization.
            \item \textbf{Multiple critical points:} Degenerate solutions where $w_i = 0$ for any $i$ are critical points, and infinitely many parameter configurations achieve identical loss values.
        \end{itemize}

        These severe limitations make naive product nodes impractical for direct use, motivating our investigation into whether specialized structural constraints and sparsity assumptions can render product-based architectures trainable.

    \section{Analysis of Product with Neutral Element Node} \label{app:product_with_neutral_element_node_analysis}

        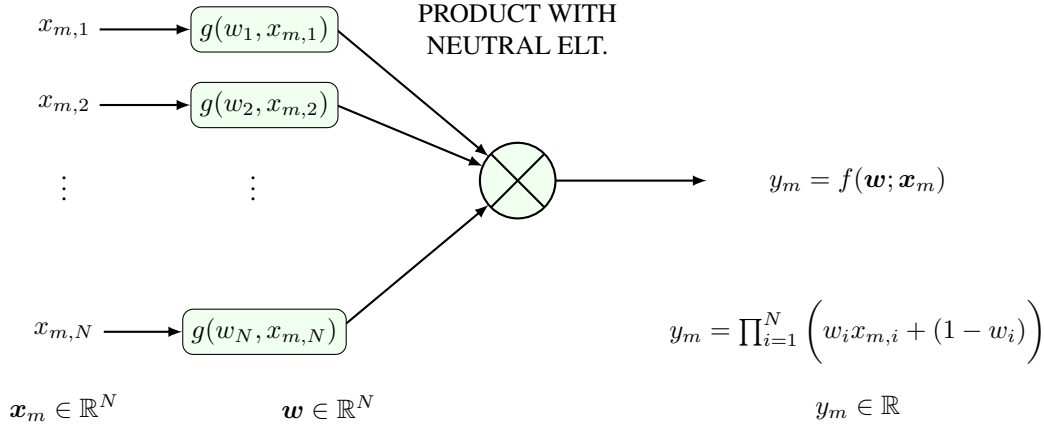
\begin{figure}[htbp]
    \centerline{
        \begin{tikzpicture}
            \tikzmath{
                function PlotArrow (\ArrowStartX, \ArrowStartY, \CircleCenterX, \CircleCenterY, \CircleRadius) {
                    \ArrowEndX = \CircleCenterX - \CircleRadius*cos(atan((\ArrowStartY -  \CircleCenterY)/(\ArrowStartX - \CircleCenterX)));
                    \ArrowEndY = \CircleCenterY - \CircleRadius*sin(atan((\ArrowStartY - \CircleCenterY)/(\ArrowStartX - \CircleCenterX)));
                    {\draw[thick, -latex] (\ArrowStartX,\ArrowStartY) -- (\ArrowEndX,\ArrowEndY);
                    };
                };
                \CircleRadius1 = 0.5;
                \CircleCenterX1 = 0;
                \CircleCenterY1 = 0;
                \ArrowStartX1 = -4;
                \ArrowStartY1 = +2;
                \WeightPositionX1 = (\ArrowStartX1+\CircleCenterX1)/2;
                \WeightPositionY1 = (\ArrowStartY1+\CircleCenterY1)/2;
                PlotArrow(\ArrowStartX1+1.6, \ArrowStartY1, \CircleCenterX1, \CircleCenterY1, \CircleRadius1);
                \ArrowStartX2 = -4;
                \ArrowStartY2 = +1;
                \WeightPositionX2 = (\ArrowStartX2+\CircleCenterX1)/2;
                \WeightPositionY2 = (\ArrowStartY2+\CircleCenterY1)/2;
                PlotArrow(\ArrowStartX2+1.6, \ArrowStartY2, \CircleCenterX1, \CircleCenterY1, \CircleRadius1);
                \ArrowStartX3 = -4;
                \ArrowStartY3 = -2;
                \WeightPositionX3 = (\ArrowStartX3+\CircleCenterX1)/2;
                \WeightPositionY3 = (\ArrowStartY3+\CircleCenterY1)/2;
                PlotArrow(\ArrowStartX3+1.6, \ArrowStartY3, \CircleCenterX1, \CircleCenterY1, \CircleRadius1);
            };
            \draw[thick, fill=green!6] (\CircleCenterX1,\CircleCenterY1) circle (\CircleRadius1);
            
            \draw[thick] (\CircleCenterX1 - 0.707*\CircleRadius1, \CircleCenterY1+0.707*\CircleRadius1) -- (\CircleCenterX1 + 0.707*\CircleRadius1, \CircleCenterY1-0.707*\CircleRadius1);
            
            \draw[thick] (\CircleCenterX1 - 0.707*\CircleRadius1, \CircleCenterY1-0.707*\CircleRadius1) -- (\CircleCenterX1 + 0.707*\CircleRadius1, \CircleCenterY1               +0.707*\CircleRadius1);
            
            \node [align=center] at (\CircleCenterX1,\CircleCenterY1 + 4*\CircleRadius1 ) {PRODUCT WITH\\NEUTRAL ELT.};
            \draw[thick,-latex] (\CircleCenterX1 + \CircleRadius1,\CircleCenterY1) -- (\CircleCenterX1 + \CircleRadius1 + 2,\CircleCenterY1);

            \node at (\CircleCenterX1 + \CircleRadius1 + 4,\CircleCenterY1) {$y_m = f(\mathbfit{w};\mathbfit{x}_m)$};
            
            \node at (\CircleCenterX1 + \CircleRadius1 + 4,\ArrowStartY3-1) {$y_m \in \mathbb{R}$};
            
            \node at (\CircleCenterX1 + \CircleRadius1 + 4,\ArrowStartY3) {$y_m = \prod_{i=1}^N \bigg(w_ix_{m,i} + (1 - w_i)\bigg)$};
            
            \node (X1) at (\ArrowStartX1-2,\ArrowStartY1) {$x_{m,1}$};
            \node (X2) at (\ArrowStartX2-2,\ArrowStartY2) {$x_{m,2}$};
            \node (X3) at (\ArrowStartX3-2,\ArrowStartY3) {$x_{m,N}$};
            \node at (\ArrowStartX3-2,\CircleCenterY1) {$\vdots$};
            \node at (\ArrowStartX3-2,\ArrowStartY3-1) {$\mathbfit{x}_m \in \mathbb{R}^N$};

            \node [rectangle, draw=black, fill=green!6, rounded corners] (G1) at (\ArrowStartX1+0.65,\ArrowStartY1) {$g(w_1,x_{m,1})$};
            \node [rectangle, draw=black, fill=green!6, rounded corners] (G2) at (\ArrowStartX2+0.65,\ArrowStartY2) {$g(w_2,x_{m,2})$};
            \node [rectangle, draw=black, fill=green!6, rounded corners] (G3) at (\ArrowStartX3+0.65,\ArrowStartY3) {$g(w_N,x_{m,N})$};
            
            \node at (\ArrowStartX3+0.5,\CircleCenterY1) {$\vdots$};

            \draw[thick,-latex] (X1) -- (G1);
            \draw[thick,-latex] (X2) -- (G2);
            \draw[thick,-latex] (X3) -- (G3);
            
    
            \node at (\WeightPositionX3-0.5,\ArrowStartY3-1) {$\mathbfit{w} \in \mathbb{R}^N$};
            
        \end{tikzpicture}
    }
    \caption{The product unit with neutral element}
    \label{fig:prod_unit_with_neutral_elt}
\end{figure}

        Given training samples $\{ \mathbfit{x}_m \in \mathbb{R}^N \}_{m=1}^M$ and learnable parameters $\mathbfit{w} \in \mathbb{R}^N$, we introduce a parametric product node that incorporates a form of neutral element to prevent vanishing products when parameters approach zero. This product with neutral element unit is defined as (Figure~\ref{fig:prod_unit_with_neutral_elt}):
        \begin{equation}
            y_m = f(\mathbfit{w};\mathbfit{x}_m) =\prod_{i=1}^{N} \Big(w_i x_{m,i} + (1 - w_i)\Big) =\prod_{i=1}^{N} \Big(w_i z_{m,i} + 1\Big) =\prod_{i=1}^{N} a_{m,i} \quad \forall m \in\{1,\ldots,M\}
        \end{equation}

        where every entry of the training set $x_{m,i}$ can be equivalently replaced by $z_{m,i}:= x_{m,i} - 1 \quad\forall(m,i) $ and $a_{m,i}:=(w_i z_{m,i} + 1) $. 
       
        This formulation implements a linear interpolation between the input value $x_{m,i}$ and the multiplicative identity 1, controlled by the learnable parameter $w_i$.  When parameters are binary ($w_i \in \{0,1\}$), this provides an intuitive feature selection mechanism where excluded features ($w_i = 0$) contribute a neutral multiplicative factor of 1, while selected features ($w_i = 1$) contribute their full value $x_{m,i}$ to the product. This approach draws inspiration from neural arithmetic units \cite{NeuralArithmeticUnits,NeuralArithmeticLogicUnits} and provides a method to avoid vanishing products when parameters equal zero.       
        
        The empirical risk is:
        \begin{equation}
            \hat{\ell}_{\overset{\bullet}{\Pi}} = \frac{1}{M}\sum_{m=1}^{M}(y_m - y_m^{\mathrm{true}})^2
        \end{equation}
        
        The optimization problem is formulated as:
        \begin{align}
            (P_{\overset{\bullet}{\Pi}}): \quad \hat{\mathbfit{w}}_{\overset{\bullet}{\Pi}} := \arg\min_{\mathbfit{w} \in \mathbb{R}^N} \hat{\ell}_{\overset{\bullet}{\Pi}}(\mathbfit{w})
        \end{align}

        \begin{proposition}[Gradient of Product with Neutral Element]\label{prop:product_with_ne_gradient}
            For the product node with neutral element, with empirical risk $\hat{\ell}_{\overset{\bullet}\Pi}(\mathbfit{w})$, the j-th component of the gradient $\nabla_{\mathbf{w}} \hat{\ell}_{\overset{\bullet}\Pi}(\mathbf{w})$  is:
            \begin{equation}
                \begin{aligned}
                    \frac{\partial \hat{\ell}_{\overset{\bullet}\Pi}(\mathbfit{w}) }{\partial w_j} = \frac{2}{M}\sum_{m=1}^{M}\left[z_{m,j}\times\left(\prod_{\substack{i=1 \\ i \neq j}}^{N} a_{m,i}\right)\times \left(\prod_{i=1}^{N} a_{m,i} - \prod_{i=1}^{N} a_{m,i}^\mathrm{true}\right)\right]\\
                \end{aligned}
            \end{equation} 
            where $z_{m,i}= x_{m,i} - 1$ and $a_{m,i}=(w_i z_{m,i} + 1) \quad\forall(m,i) $.
        \end{proposition}
    
        \begin{proof}
            From the expression of the empirical risk:

            \begin{equation}\label{eq:empirical_risk_product_with_ne}
                \begin{aligned}
                    \hat{\ell}_{\overset{\bullet}{\Pi}} = \frac{1}{M}\sum_{m=1}^{M}\left[y_m - y_m^{\mathrm{true}}\right]^2= \frac{1}{M}\sum_{m=1}^{M}\left[\prod_{i=1}^{N} a_{m,i} - \prod_{i=1}^{N} a_{m,i}^\mathrm{true}\right]^2\\
                \end{aligned}                
            \end{equation}

            The derivative of $\prod_{i=1}^{N} a_{m,i}$ w.r.t to parameter $w_j$ being:
            \begin{equation}
                \frac{\partial}{\partial w_j}\left(\prod_{i=1}^{N} a_{m,i}\right) = \frac{\partial}{\partial w_j}\left(\prod_{i=1}^{N} \Big(w_iz_{m,i}+1\Big)\right) = z_{m,j}\times\prod_{\substack{i=1 \\ i \neq j}}^{N} \Big(w_i z_{m,i} + 1\Big) = z_{m,j} \times \prod_{\substack{i=1 \\ i \neq j}}^{N} a_{m,i}
            \end{equation}

            The j-th component of the gradient $\nabla_{\mathbf{w}} \hat{\ell}_{\overset{\bullet}\Pi}(\mathbf{w})$ is:

            \begin{equation}
                \begin{aligned}
                    \frac{\partial \hat{\ell}_{\overset{\bullet}\Pi}(\mathbfit{w}) }{\partial w_j} = \frac{2}{M}\sum_{m=1}^{M}\left[z_{m,j}\times\left(\prod_{\substack{i=1 \\ i \neq j}}^{N} a_{m,i}\right)\times \left(\prod_{i=1}^{N} a_{m,i} - \prod_{i=1}^{N} a_{m,i}^\mathrm{true}\right)\right]\\
                \end{aligned}
            \end{equation}

            As for the naive product node it can be seen that the problem is non linearly separable as the $j$-th gradient term depends on all the parameters.
        \end{proof}

        \begin{proposition}[Hessian of Product with Neutral Element]\label{prop:product_with_ne_hessian}
            The Hessian matrix of the empirical risk $\hat{\ell}_{\overset{\bullet}\Pi}(\mathbfit{w})$ has entries:
            \begin{equation}
                [\mathbfit{H}]_{j,l} = \frac{\partial^2 \hat{\ell}_{\overset{\bullet}\Pi}(\mathbfit{w})}{\partial w_j \partial w_l} = \frac{2}{M}\sum_{m=1}^{M}
                \left\{
                \begin{aligned}
                    &z_{m,j}^2\times\prod_{\substack{i=1 \\ i \neq j}}^{N} a_{m,i}^2& \text{if } j = l \\
                    &\left(z_{m,j}z_{m,l}\times \prod_{\substack{i=1 \\ i \neq j \neq l}}^{N} a_{m,i}\right) \times \left(2a_{m,j}a_{m,l}\times\prod_{\substack{i=1 \\ i \neq j \neq l}}^{N} a_{m,i}
                    - \prod_{i=1}^{N} a_{m,i}^\mathrm{true}\right)& \text{if } j \neq l
                \end{aligned}
                \right.
            \end{equation}
            where $z_{m,i}= x_{m,i} - 1$ and $a_{m,i}=(w_i z_{m,i} + 1)\quad\forall(m,i) $.
            
        \end{proposition}
    
        \begin{proof}
            The Hessian of the empirical risk is defined as:
            \begin{equation}
                \mathbfit{H} = \nabla_{\mathbfit{w}}^2 \hat{\ell}_{\overset{\bullet}\Pi}(\mathbfit{w}) = \frac{\partial}{\partial \mathbfit{w}}\nabla_{\mathbfit{w}} \hat{\ell}_{\overset{\bullet}\Pi}(\mathbfit{w})
            \end{equation}
            
            Given the j-th component of the gradient (Proposition~\ref{prop:product_with_ne_gradient}):
            
            \begin{equation}
                \begin{aligned}
                    \frac{\partial \hat{\ell}_{\overset{\bullet}\Pi}(\mathbfit{w}) }{\partial w_j} = \frac{2}{M}\sum_{m=1}^{M}\left[z_{m,j}\times\left(\prod_{\substack{i=1 \\ i \neq j}}^{N} a_{m,i}\right)\times \left(\prod_{i=1}^{N} a_{m,i} - \prod_{i=1}^{N} a_{m,i}^\mathrm{true}\right)\right]\\
                \end{aligned}
            \end{equation}

            The Hessian entries are given by $[\mathbfit{H}]_{j,l} = \frac{\partial^2 \hat{\ell}_{\overset{\bullet}\Pi}(\mathbfit{w})}{\partial w_j \partial w_l}$. Two cases arise:
            
            \textbf{Diagonal entries} ($j = l$):
            
            \begin{equation}
                \begin{aligned}
                    \frac{\partial^2 \hat{\ell}_{\overset{\bullet}\Pi}(\mathbfit{w})}{\partial w_j^2} &=\frac{2}{M}\sum_{m=1}^{M}\left[z_{m,j}\times \left(\prod_{\substack{i=1 \\ i \neq j}}^{N} a_{m,i}\right) \times z_{m,j}\times \left(\prod_{\substack{i=1 \\ i \neq j}}^{N} a_{m,i} \right)\right]\\
                    &=\frac{2}{M}\sum_{m=1}^{M}\left[z_{m,j}^2\times\prod_{\substack{i=1 \\ i \neq j}}^{N} a_{m,i}^2\right]\\
                \end{aligned}                
            \end{equation}

            \textbf{Off-diagonal entries }($j \neq l$):
            
            The derivative of $\prod_{i=1}^{N} a_{m,i}^2$ w.r.t to parameter $w_j$ being:
            \begin{equation}
                \frac{\partial}{\partial w_j}\left(\prod_{i=1}^{N} a_{m,i}^2\right) = \frac{\partial}{\partial w_j}\left(\prod_{i=1}^{N} \Big(w_iz_{m,i}+1\Big)^2\right) = 2z_{m,j}\Big(w_jz_{m,j}+1\Big) \times \prod_{\substack{i=1 \\ i \neq j}}^{N} \Big(w_i z_{m,i} + 1\Big)^2 = 2z_{m,j}a_{m,j} \times \prod_{\substack{i=1 \\ i \neq j}}^{N} a_{m,i}^2
            \end{equation}
            We have:
            \begin{equation}
                \begin{aligned}
                    \frac{\partial^2 \hat{\ell}_{\overset{\bullet}\Pi}(\mathbfit{w})}{\partial w_j \partial w_l} &= \frac{2}{M}\sum_{m=1}^{M}\left[z_{m,j}a_{m,j}\times\frac{\partial}{\partial w_l}\left(\prod_{\substack{i=1 \\ i \neq j}}^{N} a_{m,i}^2\right)- z_{m,j} \times \frac{\partial}{\partial w_l}\left(\prod_{\substack{i=1 \\ i \neq j}}^{N} a_{m,i}\right) \times \left(\prod_{i=1}^{N} a_{m,i}^\mathrm{true}\right)\right]\\
                    &= \frac{2}{M}\sum_{m=1}^{M}\left[z_{m,j}a_{m,j}\times 2z_{m,l}a_{m,l}\times\prod_{\substack{i=1 \\ i \neq j \neq l}}^{N} a_{m,i}^2 - z_{m,j}z_{m,l}\times \left(\prod_{\substack{i=1 \\ i \neq j \neq l}}^{N} a_{m,i}\right)\times\left(\prod_{i=1}^{N} a_{m,i}^\mathrm{true}\right)\right]\\
                    &= \frac{2}{M}\sum_{m=1}^{M}\left[\left(z_{m,j}z_{m,l}\times \prod_{\substack{i=1 \\ i \neq j \neq l}}^{N} a_{m,i}\right) \times \left(2a_{m,j}a_{m,l}\times\prod_{\substack{i=1 \\ i \neq j \neq l}}^{N} a_{m,i}
                    - \prod_{i=1}^{N} a_{m,i}^\mathrm{true}\right)\right]\\
                \end{aligned}
            \end{equation}
        \end{proof}

        \begin{proposition}[Non-Convexity of Product with Neutral Element]\label{prop:product_node_with_ne_non_convexity}
            The optimization problem for the product node with neutral element is non-convex in $\mathbfit{w}$. Specifically, under the assumption that $x_{m,i}$ are i.i.d. samples from a distribution with $\mathbb{E}[\mathrm{x}] = 0$ and $\mathbb{E}[\mathrm{x}^2] > 0$, there exist test vectors $\mathbfit{q}$ such that:
            \begin{equation}
                \mathbfit{q}^T\mathbb{E}\left[\nabla_\mathbfit{w}^2\hat{\ell}_{\overset{\bullet}{\Pi}} (\mathbfit{w})\right]\mathbfit{q} < 0
            \end{equation}
            for $N \geq 2$, such that the Hessian matrix is not always positive semi-definite.
    
        \end{proposition}
    
        \begin{proof}
            letting $w_i = 0 \: \forall i \in\{1,...,N\}, w_1^\mathrm{true} = 1, w_i^\mathrm{true} = 0 \: \forall i \in \{2,...,N\}$, we can prove that the Hessian matrix is non-convex. Indeed, there exist instance of $\mathbfit{q}$ of the form $\mathbfit{q}=(c,1,...,1)^T$ such that $\mathbfit{q}^T\nabla_\mathbfit{w}^2\hat{\ell}_{\overset{\bullet}{\Pi}} (\mathbfit{w}) \mathbfit{q}<0$. Thus, $\nabla_\mathbfit{w}^2\hat{\ell}_{\overset{\bullet}{\Pi}}(\mathbfit{w})$ is not PSD for certain value of $\mathbfit{w},\mathbfit{w}^\mathrm{true}$ and the problem is not generally convex in $\mathbfit{w}$. 
            
            Under the above hypothesis, and using the facts that $a_{m,i}(w_i=0)=1$ $a_{m,i}(w_i=1)=z_{m,i}+1$, the Hessian matrix from Proposition~\ref{prop:product_with_ne_hessian} can be simplified to:

            \begin{equation}
                [\mathbfit{H}]_{j,l} =  \frac{2}{M}\sum_{m=1}^{M}
                \left\{
                \begin{aligned}
                    &z_{m,j}^2& \text{if } j = l \\
                    &\left(z_{m,j}z_{m,l}\right) \times \left(2\
                    - (z_{m,1}+1)\right)& \text{if } j \neq l
                \end{aligned}
                \right.
            \end{equation}
            
            Recall that $z_{m,i} = x_{m,i} - 1$, and let $x_{m,i}$ be i.i.d. samples of a random variable $\mathrm{x} \sim X$, with $\mathbb{E}_{\mathrm{x} \sim X}\left\{\mathrm{x}\right\} = 0$. Then the expectation of the Hessian matrix can be defined as\footnote{Since $\mathrm{z}$ is entirely determined by $\mathrm{x}$, we use the notation $\mathbb{E}_{\mathrm{x} \sim X}$ interchangeably when taking expectations over variables $\mathrm{z}$ or $\mathrm{x}$.}:
:

            \textbf{Diagonal entries} ($j = l$):
            \begin{equation}
               \begin{aligned}
                   \mathbb{E}_{\mathrm{x} \sim X}\left\{[\mathbfit{H}]_{j,l}\right\} = \mathbb{E}_{\mathrm{x} \sim X}\left\{\frac{2}{M}\sum_{m=1}^{M}\mathrm{z}_{m,j}^2\right\} = 2\mathbb{E}_{\mathrm{x} \sim X}\left\{\mathrm{z}_{j}^2\right\}= 2\mathbb{E}_{\mathrm{x} \sim X}\left\{(\mathrm{x}_{j}-1)^2\right\}= 2\left(\mathbb{E}_{\mathrm{x} \sim X}\left\{\mathrm{x}^2\right\}+1\right)
               \end{aligned}
            \end{equation}
            
            \textbf{Off-diagonal entries} ($j \neq l$):
            \begin{equation}
               \begin{aligned}
                   \mathbb{E}_{\mathrm{x} \sim X}\left\{[\mathbfit{H}]_{j,l}\right\} &= \mathbb{E}_{\mathrm{x} \sim X}\left\{\frac{2}{M}\sum_{m=1}^{M}\left(\mathrm{z}_{m,j}\mathrm{z}_{m,l}\times (1\
                    - \mathrm{z}_{m,1})\right)\right\}\\
                    &= 2\mathbb{E}_{\mathrm{x} \sim X}\left\{\mathrm{z}_{j}\mathrm{z}_{l}\times (1\
                    - \mathrm{z}_{1})\right\}\\
                    &= 2\mathbb{E}_{\mathrm{x} \sim X}\left\{\mathrm{z}_{j}\mathrm{z}_{l}\right\}  - 2\mathbb{E}_{\mathrm{x} \sim X}\left\{\mathrm{z}_{j}\mathrm{z}_{l}\mathrm{z}_{1}\right\}\\
                    &= 2\mathbb{E}_{\mathrm{x} \sim X}^2\left\{\mathrm{z}\right\}  - 2\mathbb{E}_{\mathrm{x} \sim X}\left\{\mathrm{z}_{j}\mathrm{z}_{l}\mathrm{z}_{1}\right\}\\                   
               \end{aligned}
            \end{equation}

            While the indices $j$ and $l$ are necessarily different for off-diagonal entries, one of them may equal 1, meaning that either $z_j$ or $z_l$ corresponds to $z_1$. Hence, in the expectation computation, we must consider the following two cases:

            \begin{equation}
                \mathbb{E}_{\mathrm{x} \sim X}\left\{[\mathbfit{H}]_{j,l}\right\} = \left\{
               \begin{aligned}
                    &2\mathbb{E}_{\mathrm{x} \sim X}^2\left\{\mathrm{z}\right\}  - 2\mathbb{E}_{\mathrm{x} \sim X}^3\left\{\mathrm{z}\right\}&& \text{if } j \neq l \neq 1\\ 
                    &2\mathbb{E}_{\mathrm{x} \sim X}^2\left\{\mathrm{z}\right\}  - 2\mathbb{E}_{\mathrm{x} \sim X}\left\{\mathrm{z}\right\}\times\mathbb{E}_{\mathrm{x} \sim X}\left\{\mathrm{z}^2\right\} && \text{if } j \neq l \text{ and at least one of } j,l \text{ equals } 1\\
               \end{aligned}
               \right.
            \end{equation}

            with 
            $$\mathbb{E}_{\mathrm{x} \sim X}\left\{\mathrm{z}\right\}=\mathbb{E}_{\mathrm{x} \sim X}\left\{\mathrm{x}-1\right\}=\mathbb{E}_{\mathrm{x} \sim X}\left\{\mathrm{x}\right\}-1=-1$$
            
            and
            $$\mathbb{E}_{\mathrm{x} \sim X}\left\{\mathrm{z}^2\right\}=\mathbb{E}_{\mathrm{x} \sim X}\left\{(\mathrm{x}-1)^2\right\}=\mathbb{E}_{\mathrm{x} \sim X}\left\{\mathrm{x}^2\right\}+1$$
            
            such that:
            \begin{equation}
                \mathbb{E}_{\mathrm{x} \sim X}\left\{[\mathbfit{H}]_{j,l}\right\} = 2\times\left\{
                \begin{aligned}
                    &2&& \text{if } j \neq l \neq 1\\ 
                    &2+\mathbb{E}_{\mathrm{x} \sim X}\left\{\mathrm{x}^2\right\}&& \text{if } j \neq l \text{ and at least one of } j,l \text{ equals } 1\\
               \end{aligned}
               \right.
            \end{equation}

            Combining the diagonal and off-diagonal cases, the expected Hessian matrix is fully characterized by:
            \begin{equation}
                \mathbb{E}_{\mathrm{x} \sim X}\left\{[\mathbfit{H}]_{j,l}\right\} = 2\times\left\{
                \begin{aligned}
                    &\mathbb{E}_{\mathrm{x} \sim X}\left\{\mathrm{x}^2\right\}+1&&\text{if } j = l \\ 
                    &2&& \text{if } j \neq l \neq 1\\ 
                    &\mathbb{E}_{\mathrm{x} \sim X}\left\{\mathrm{x}^2\right\}+2&& \text{if } j \neq l \text{ and at least one of } j,l \text{ equals } 1\\
               \end{aligned}
               \right.
            \end{equation}

            Which can equivalently be writen as
            \begin{equation}
               \begin{aligned}
                   \mathbfit{H} &=2\times\begin{pmatrix}
                        \mathbb{E}_{\mathrm{x} \sim X}\left\{\mathrm{x}^2\right\}+1 & \mathbb{E}_{\mathrm{x} \sim X}\left\{\mathrm{x}^2\right\}+2 & \mathbb{E}_{\mathrm{x} \sim X}\left\{\mathrm{x}^2\right\}+2 & \cdots & \mathbb{E}_{\mathrm{x} \sim X}\left\{\mathrm{x}^2\right\}+2\\
                        \mathbb{E}_{\mathrm{x} \sim X}\left\{\mathrm{x}^2\right\}+2 & \mathbb{E}_{\mathrm{x} \sim X}\left\{\mathrm{x}^2\right\}+1 & 2 & \cdots & 2\\
                        \mathbb{E}_{\mathrm{x} \sim X}\left\{\mathrm{x}^2\right\}+2 & 2 & \mathbb{E}_{\mathrm{x} \sim X}\left\{\mathrm{x}^2\right\}+1 & \cdots & 2\\
                        \vdots & \vdots & \vdots & \ddots&\vdots\\
                        \mathbb{E}_{\mathrm{x} \sim X}\left\{\mathrm{x}^2\right\}+2 & 2 & 2 & \cdots&\mathbb{E}_{\mathrm{x} \sim X}\left\{\mathrm{x}^2\right\}+1\\
                    \end{pmatrix}\\
                    &= 2\left(2 \cdot \mathbf{1} \cdot \mathbf{1}^T + (\mathbb{E}_{\mathrm{x} \sim X}\left\{\mathrm{x}^2\right\} - 1)\mathbfit{I} + \mathbb{E}_{\mathrm{x} \sim X}\left\{\mathrm{x}^2\right\}\left[\mathbfit{e}^{(1)}\cdot\mathbf{1}^T + \mathbf {1}\cdot(\mathbfit{e}^{(1)})^T - 2\mathbfit{e}^{(1)}(\mathbfit{e}^{(1)})^T\right]\right)\\
                    &= 2\left(2 \cdot \mathbf{1} \cdot \mathbf{1}^T + (\mathbb{E}_{\mathrm{x} \sim X}\left\{\mathrm{x}^2\right\} - 1)\mathbfit{I} + \mathbb{E}_{\mathrm{x} \sim X}\left\{\mathrm{x}^2\right\}\times\mathbfit{E}\right)
               \end{aligned}
            \end{equation}

            with 
            $$\mathbfit{E} = \begin{bmatrix}
            0 & 1 & 1 & \cdots & 1 \\
            1 & 0 & 0 & \cdots & 0 \\
            1 & 0 & 0 & \cdots & 0 \\
            \vdots & \vdots & \vdots & \ddots & \vdots \\
            1 & 0 & 0 & \cdots & 0
            \end{bmatrix}
            $$
            
            Multiplying by the test vector $\mathbfit{q}$, we have the following inequalities:
            \begin{itemize}
                \item $\mathbfit{q}^T \mathbf{1}.\mathbf{1}^T \mathbfit{q} = (\mathbf{1}^T \mathbfit{q})^2=\left(\sum_{i=1}^N q_i\right)^2$
                \item $\mathbfit{q}^T \mathbfit{I} \mathbfit{q} = \mathbfit{q}^T \mathbfit{q}= \|\mathbfit{q}\|_2^2 = \sum_{i=1}^N q_i^2$
                \item $\mathbfit{q}^T \mathbfit{E} \mathbfit{q} =  2q_1(\mathbf{1}^T\mathbfit{q} - q_1)=2q_1\left(\sum_{i=2}^N q_i\right)$
            \end{itemize}

            Hence,
            \begin{equation}
               \begin{aligned}
                   \mathbfit{q}^T\mathbb{E}_{\mathrm{x} \sim X}\left\{\nabla_\mathbfit{w}^2\hat{\ell}_{\overset{\bullet}{\Pi}} (\mathbfit{w})\right\}\mathbfit{q} =& 4(\mathbf{1}^T\mathbfit{q})^2 + 2(\mathbb{E}_{\mathrm{x} \sim X}\left\{\mathrm{x}^2\right\} - 1)||\mathbfit{q}||_2^2 + 2\mathbb{E}_{\mathrm{x} \sim X}\left\{\mathrm{x}^2\right\}\left[2q_1(\mathbf{1}^T\mathbfit{q} - q_1)\right]\\
               \end{aligned}
            \end{equation}

            Let $q_1 = c, q_i = 1 \: \forall i \in\{2,...,N\}$. We thus have the following updated inequalities:
            \begin{itemize}
                \item $(\mathbf{1}^T \mathbfit{q})^2=\left(\sum_{i=1}^N q_i\right)^2=(c+N-1)^2$
                \item $\|\mathbfit{q}\|_2^2=\sum_{i=1}^N q_i^2 = c^2+N-1$
                \item $2q_1(\mathbf{1}^T\mathbfit{q} - q_1)=2q_1\left(\sum_{i=2}^N q_i\right)=2c(N-1)$
            \end{itemize}

            Which leads to:
            \begin{equation}
               \begin{aligned}
                   \mathbfit{q}^T\mathbb{E}_{\mathrm{x} \sim X}\left\{\nabla_\mathbfit{w}^2\hat{\ell}_{\overset{\bullet}{\Pi}} (\mathbfit{w})\right\}\mathbfit{q} =& 4(c+N-1)^2 + 2(\mathbb{E}_{\mathrm{x} \sim X}\left\{\mathrm{x}^2\right\} - 1)\times(c^2+N-1) + 2\mathbb{E}_{\mathrm{x} \sim X}\left\{\mathrm{x}^2\right\}\left[2c(N-1)\right]\\
                   =&4\left[c^2+2c(N-1)+(N-1)^2\right] + 2(\mathbb{E}_{\mathrm{x} \sim X}\left\{\mathrm{x}^2\right\} - 1)\times(c^2+N-1)\\
                   &+ 2\mathbb{E}_{\mathrm{x} \sim X}\left\{\mathrm{x}^2\right\}\left[2c(N-1)\right]\\
                   =&2\left(\mathbb{E}_{\mathrm{x} \sim X}\left\{\mathrm{x}^2\right\} +1\right)c^2+4(N-1)\times\left(2+\mathbb{E}_{\mathrm{x} \sim X}\left\{\mathrm{x}^2\right\}\right)c\\
                   &+ 4(N-1)^2+2(N-1)\times(\mathbb{E}_{\mathrm{x} \sim X}\left\{\mathrm{x}^2\right\}-1)\\
                   =&2\left(\mathbb{E}_{\mathrm{x} \sim X}\left\{\mathrm{x}^2\right\} +1\right)c^2+4(N-1)\times\left(\mathbb{E}_{\mathrm{x} \sim X}\left\{\mathrm{x}^2\right\}+2\right)c\\
                   &+ 2(N-1)\times(\mathbb{E}_{\mathrm{x} \sim X}\left\{\mathrm{x}^2\right\}+2N-3)\\
               \end{aligned}
            \end{equation}

            If the second order polynomial in $c$ admit 2 real roots, then there exists value of $c$ such that $\mathbfit{q}^T\mathbb{E}_{\mathrm{x} \sim X}\left\{\nabla_\mathbfit{w}^2\hat{\ell}_{\overset{\bullet}{\Pi}} (\mathbfit{w})\right\}\mathbfit{q}<0$. To this end we compute the determinant of the polynomial:
            
            \begin{equation}
               \begin{aligned}
                   \Delta &= 16(N-1)^2\times(\mathbb{E}_{\mathrm{x} \sim X}\left\{\mathrm{x}^2\right\}+2)^2 - 16(N-1)\times(\mathbb{E}_{\mathrm{x} \sim X}\left\{\mathrm{x}^2\right\}+1)\times\left(\mathbb{E}_{\mathrm{x} \sim X}\left\{\mathrm{x}^2\right\}+2N-3\right)\\
                   &=16(N-1)\times\left[(N-1)\times(\mathbb{E}_{\mathrm{x} \sim X}\left\{\mathrm{x}^2\right\}+2)^2 - (\mathbb{E}_{\mathrm{x} \sim X}\left\{\mathrm{x}^2\right\}+1)\times\left(\mathbb{E}_{\mathrm{x} \sim X}\left\{\mathrm{x}^2\right\}+2N-3\right)\right]\\
                   &= 16(N-1)\times\left[(N-1)\times(\mathbb{E}_{\mathrm{x} \sim X}^2\left\{\mathrm{x}^2\right\} + 4\mathbb{E}_{\mathrm{x} \sim X}\left\{\mathrm{x}^2\right\}+4) - \mathbb{E}_{\mathrm{x} \sim X}^2\left\{\mathrm{x}^2\right\} - \mathbb{E}_{\mathrm{x} \sim X}\left\{\mathrm{x}^2\right\}\times(2N-3)\right.\\
                   &\left.- \mathbb{E}_{\mathrm{x} \sim X}\left\{\mathrm{x}^2\right\} - (2N - 3)\right]\\
                   &= 16(N-1)\times\left[\mathbb{E}_{\mathrm{x} \sim X}^2\left\{\mathrm{x}^2\right\}\times(N-2) + 2\mathbb{E}_{\mathrm{x} \sim X}\left\{\mathrm{x}^2\right\}\times(N-1) +2N -1 \right]
               \end{aligned}
            \end{equation}

            As $\mathbb{E}_{\mathrm{x} \sim X}\left\{\mathrm{x}^2\right\} > 0$, $\Delta > 0 \forall N \geq 2$ and the polynomial admit two roots and is thus negative on a certain range of its definition domain. 
        \end{proof}

        \subsection{Convexity Under Sparse Input Constraints}

            While the neutral element modification improves numerical stability by preventing vanishing products, the fundamental optimization challenges persist. The expected Hessian exhibits negative eigenvalues for $N \geq 2$, parameter updates remain non-separable across dimensions, and non-convexity emerges systematically under standard distributional assumptions. These findings confirm that the multiplicative structure inherently creates a complex optimization landscape that cannot be resolved through simple architectural modifications alone.
            
            Having established the general non-convex nature of the product node with neutral element, we now investigate whether specific structural constraints on the input data can restore convexity. This analysis is motivated by the observation that sparsity patterns can simplify multiplicative interactions by eliminating the cross-parameter dependencies that create optimization pathologies. We consider a special case where the input data exhibits a specific sparsity structure that eliminates the multiplicative complexity responsible for non-convexity.
            
            \begin{proposition}[Convexity Under Sparse Input Structure]\label{prop:sparse_convexity}
                Consider input data with the following sparse structure: each sample $\mathbfit{z}_m = \mathbfit{x}_m - \mathbf{1} \in \mathbb{R}^N$ has the form $\mathbfit{z}_m = v_m \mathbfit{e}^{(k_m)}$, where $\mathrm{v}_m \sim V \in \mathbb{R}$ i.i.d. with $\mathbb{P}(\mathrm{v}_m = 0) < 1$ (non-degenerate), $\mathbfit{e}^{(k_m)} \in \{0,1\}^N$ is the $k$-th standard basis vector and $\mathrm{k}_m \sim U_N\{1,\ldots,N\}$ is uniformly distributed. Under this sparse input structure, the Hessian $\nabla_\mathbfit{w}^2\hat{\ell}_{\overset{\bullet}{\Pi}} (\mathbfit{w})$ is positive definite.
            \end{proposition}

            \begin{proof}
                With the considered dataset, only a single term per dataset entry remains in the product reduction and the  empirical risk from \eqref{eq:empirical_risk_product_with_ne} can be simplified to:
                \begin{equation}
                    \begin{aligned} 
                         \hat{\ell}_{\overset{\bullet}{\Pi}} (\mathbfit{w}) &= \frac{1}{M}\sum_{m=1}^{M}\left[\prod_{i=1}^{N} a_{m,i} - \prod_{i=1}^{N} a_{m,i}^\mathrm{true}\right]^2\\
                         &= \frac{1}{M} \sum_{m=1}^M \left(a_{m,k_m} - a_{m,k_m}^\mathrm{true}\right)^2\\
                         &= \frac{1}{M} \sum_{m=1}^M \left(w_{k_m}v_{m} + 1 - w_{k_m}^\mathrm{true}v_{m} - 1\right)^2\\
                         &= \frac{1}{M} \sum_{m=1}^M v_{m}^2\left(w_{k_m} - w_{k_m}^\mathrm{true}\right)^2 \\
                    \end{aligned}
                \end{equation}

                Which has an expectation of:
                \begin{equation}
                    \begin{aligned}
                        \mathbb{E}_{\mathrm{v}_m \sim V, \mathrm{k}_m \sim U_N}\left\{\hat{\ell}_{\overset{\bullet}{\Pi}}(\mathbfit{w})\right\} &= \mathbb{E}_{\mathrm{v}_m \sim V, \mathrm{k}_m \sim U_N}\left\{\frac{1}{M} \sum_{m=1}^M \mathrm{v}_{m}^2\left(w_{\mathrm{k}_m} - w_{\mathrm{k}_m}^\mathrm{true}\right)^2 \right\}\\
                        &= \mathbb{E}_{\mathrm{v}_m \sim V} \left\{\mathrm{v}_m^2\right\}\mathbb{E}_{\mathrm{k}_m \sim U_N}\left\{\left(w_{\mathrm{k}_m} - w_{\mathrm{k}_m}^\mathrm{true}\right)^2 \right\}\\
                        &= \frac{1}{N}\mathbb{E}_{\mathrm{v}_m \sim V} \left\{\mathrm{v}_m^2\right\}\sum_{i=1}^{N}\left(w_{i} - w_{i}^\mathrm{true}\right)^2 \quad \text{(by uniform sampling)}\\
                    \end{aligned}
                \end{equation}
                
                The gradient of the loss can be expressed as:
                \begin{equation}
                    \begin{aligned}
                        \nabla_\mathbfit{w}\hat{\ell}_{\overset{\bullet}{\Pi}} (\mathbfit{w}) &= \frac{\partial}{\partial\mathbfit{w}}\left(\frac{1}{M} \sum_{m=1}^M v_{m}^2\left(w_{k_m} - w_{k_m}^\mathrm{true}\right)^2\right)\\
                        &=\frac{1}{M} \sum_{m=1}^M \frac{\partial}{\partial\mathbfit{w}}\left(v_{m}^2\left(w_{k_m} - w_{k_m}^\mathrm{true}\right)^2\right)\\
                        &=\frac{2}{M} \sum_{m=1}^M v_{m}^2\left(w_{k_m} - w_{k_m}^\mathrm{true}\right)\mathbfit{e}^{(k_m)}\quad \text{(by uniform sampling)}\\
                    \end{aligned}
                \end{equation}
                
                The expected gradient is proportional to the difference $(\mathbfit{w} - \mathbfit{w}^\mathrm{true})$, similar to the well conditioned case of sum node as described in Appendix~\ref{app:sum_node_analysis}, Eq.~\eqref{eq:sum_node_gradient}:
                \begin{equation}
                    \begin{aligned}
                        \mathbb{E}_{\mathrm{v}_m \sim V, \mathrm{k}_m \sim U_N}\left\{\nabla_\mathbfit{w}\hat{\ell}_{\overset{\bullet}{\Pi}}(\mathbfit{w})\right\} &= \mathbb{E}_{\mathrm{v}_m \sim V, \mathrm{k}_m \sim U_N}\left\{\frac{2}{M} \sum_{m=1}^M \mathrm{v}_m^2\left(w_{\mathrm{k}_m} - w_{\mathrm{k}_m}^\mathrm{true}\right)\mathbfit{e}^{(\mathrm{k}_m)}\right\}\\
                        &= 2\mathbb{E}_{\mathrm{v}_m \sim V}\left\{ \mathrm{v}_m^2\right\}\mathbb{E}_{\mathrm{k}_m \sim U_N}\left\{\left(w_{\mathrm{k}_m} - w_{\mathrm{k}_m}^\mathrm{true}\right)\mathbfit{e}^{(\mathrm{k}_m)}\right\}\\
                        &=\frac{2}{N}\mathbb{E}_{\mathrm{v}_m \sim V}\left\{ \mathrm{v}_m^2\right\}(\mathbfit{w} - \mathbfit{w}^\mathrm{true})\propto\left(\mathbfit{w} - \mathbfit{w}^{\mathrm{true}}\right)
                    \end{aligned}
                \end{equation}

                Furthermore, the Hessian can be defined:
                \begin{equation}
                    \begin{aligned}
                        \nabla_\mathbfit{w}^2\hat{\ell}_{\overset{\bullet}{\Pi}} (\mathbfit{w}) &=\frac{\partial}{\partial\mathbfit{w}}\left(\frac{2}{M} \sum_{m=1}^M v_{m}^2\left(w_{k_m} - w_{k_m}^\mathrm{true}\right)\mathbfit{e}^{(k_m)}\right)\\
                        &=\frac{2}{M} \sum_{m=1}^M v_{m}^2\frac{\partial}{\partial\mathbfit{w}}\left(\left(w_{k_m} - w_{k_m}^\mathrm{true}\right)\mathbfit{e}^{(k_m)}\right)\\
                        &= \frac{2}{M} \sum_{m=1}^M v_{m}^2\mathbfit{e}^{(k_m)}\left({\mathbfit{e}^{(k_m)}}\right)^\mathrm{T}
                    \end{aligned}
                \end{equation}
                
                Which has an expectation of:
                \begin{equation}
                    \begin{aligned}
                        \mathbb{E}_{\mathrm{v}_m \sim V, \mathrm{k}_m \sim U_N}\left\{\nabla_\mathbfit{w}^2\hat{\ell}_{\overset{\bullet}{\Pi}} (\mathbfit{w})\right\} &= \mathbb{E}_{\mathrm{v}_m \sim V, \mathrm{k}_m \sim U_N}\left\{\frac{2}{M} \sum_{m=1}^M \mathrm{v}_m^2\mathbfit{e}^{(\mathrm{k}_m)}{\mathbfit{e}^{(\mathrm{k}_m)}}^\mathrm{T}\right\}\\
                        &= 2\mathbb{E}_{\mathrm{v}_m \sim V}\left\{\mathrm{v}_m^2\right\} \mathbb{E}_{\mathrm{k}_m \sim U_N}\left\{\mathbfit{e}^{(\mathrm{k}_m)}{\mathbfit{e}^{(\mathrm{k}_m)}}^\mathrm{T}\right\}\\
                        &= \frac{2}{N} \mathbb{E}_{\mathrm{v}_m \sim V}\left\{\mathrm{v}_m^2\right\}\mathbfit{I} \succ \mathbf{0}\\
                    \end{aligned}
                \end{equation}
                Under this sparse input structure and the non-degeneracy condition $\mathbb{P}(\mathrm{v}_m = 0) < 1$ the expected Hessian is positive definite.

                Note that for finite datasets, the empirical Hessian $\nabla_\mathbfit{w}^2\hat{\ell}_{\overset{\bullet}{\Pi}} (\mathbfit{w}) = \frac{2}{M} \sum_{m=1}^M v_{m}^2\mathbfit{e}^{(k_m)}{\mathbfit{e}^{(k_m)}}^\mathrm{T}$ is positive definite if and only if the dataset has "full coverage," meaning every weight index $i \in \{1, 2, \ldots, N\}$ appears at least once in $\{k_1, k_2, \ldots, k_M\}$. This parallels the full column rank condition required for sum nodes. Without full coverage, some weights remain undetermined, leading to a rank-deficient Hessian.    
            \end{proof}

            \subsection{Intuition and Practical Significance} This result reveals a fundamental insight about product-based computations: sparsity eliminates the multiplicative complexity that makes optimization challenging. By constraining input vectors $\mathbfit{x}_m$ to contain at most one element different from one (equivalently, $\mathbfit{z}_m = \mathbfit{x}_m - \mathbf{1}$ with at most one non-zero element), we artificially transform the problem into a linearly separable one where the product $\prod_{i=1}^{N} (w_i z_{m,i} + 1)$ collapses to a simple linear function $(w_{k_m} v_m + 1)$. This eliminates all cross-parameter interactions that create the non-convex landscape, transforming the optimization from a complex multiplicative system into a weighted least squares formulation, inherently convex and separable across dimensions. Notably, the gradient expression becomes almost identical to that of sum nodes (with a factor $1/N$ - See Appendix~\ref{app:sum_node_analysis}), and the expected Hessian reduces to a positive definite diagonal matrix, making the optimization problem strictly convex. Crucially, this sparsity constraint is only necessary during training to enable proper parameter identification and avoid optimization pathologies. Once weights are learned under sparse conditions, the trained product node can generalize to dense, non-sparse inputs while maintaining its learned multiplicative structure.

    \section{Extension to XOR Node}\label{app:xor_node_analysis}
    
        \begin{figure}[htbp]
    \centerline{
        \begin{tikzpicture}
            \tikzmath{
                function PlotArrow (\ArrowStartX, \ArrowStartY, \CircleCenterX, \CircleCenterY, \CircleRadius) {
                    \ArrowEndX = \CircleCenterX - \CircleRadius*cos(atan((\ArrowStartY -  \CircleCenterY)/(\ArrowStartX - \CircleCenterX)));
                    \ArrowEndY = \CircleCenterY - \CircleRadius*sin(atan((\ArrowStartY - \CircleCenterY)/(\ArrowStartX - \CircleCenterX)));
                    {\draw[thick, -latex] (\ArrowStartX,\ArrowStartY) -- (\ArrowEndX,\ArrowEndY);
                    };
                };
                \CircleRadius1 = 0.5;
                \CircleCenterX1 = 0;
                \CircleCenterY1 = 0;
                \ArrowStartX1 = -2;
                \ArrowStartY1 = +2;
                \WeightPositionX1 = (\ArrowStartX1+\CircleCenterX1)/2;
                \WeightPositionY1 = (\ArrowStartY1+\CircleCenterY1)/2;
                PlotArrow(\ArrowStartX1, \ArrowStartY1, \CircleCenterX1, \CircleCenterY1, \CircleRadius1);
                \ArrowStartX2 = -2;
                \ArrowStartY2 = +1;
                \WeightPositionX2 = (\ArrowStartX2+\CircleCenterX1)/2;
                \WeightPositionY2 = (\ArrowStartY2+\CircleCenterY1)/2;
                PlotArrow(\ArrowStartX2, \ArrowStartY2, \CircleCenterX1, \CircleCenterY1, \CircleRadius1);
                \ArrowStartX3 = -2;
                \ArrowStartY3 = -2;
                \WeightPositionX3 = (\ArrowStartX3+\CircleCenterX1)/2;
                \WeightPositionY3 = (\ArrowStartY3+\CircleCenterY1)/2;
                PlotArrow(\ArrowStartX3, \ArrowStartY3, \CircleCenterX1, \CircleCenterY1, \CircleRadius1);
            };
            \draw[thick, fill=blue!5] (\CircleCenterX1,\CircleCenterY1) circle (\CircleRadius1);
            
            \draw[thick] (\CircleCenterX1 - \CircleRadius1,\CircleCenterY1) -- (\CircleCenterX1 + \CircleRadius1,\CircleCenterY1);
            
            \draw[thick] (\CircleCenterX1,\CircleCenterY1-\CircleRadius1) -- (\CircleCenterX1,\CircleCenterY1+\CircleRadius1);
            
            \draw[thick,-latex] (\CircleCenterX1 + \CircleRadius1,\CircleCenterY1) -- (\CircleCenterX1 + \CircleRadius1 + 1,\CircleCenterY1);
            
            \node at (\CircleCenterX1,\CircleCenterY1 + 2*\CircleRadius1) {XOR};
            
            \node at (\CircleCenterX1 + \CircleRadius1 + 3,\CircleCenterY1) {$y_m = f(\mathbfit{w};\mathbfit{x}_m)$};
            
            \node at (\CircleCenterX1 + \CircleRadius1 + 3,\ArrowStartY3-1) {$y_m \in \mathbb{F}_2$};
            
            \node at (\CircleCenterX1 + \CircleRadius1 + 3,\ArrowStartY3) {$y_m = \bigoplus_{i=1}^{N}{w_ix_{m,i}}$};
            
            \node [thick] at (\ArrowStartX1-0.5,\ArrowStartY1+0.2) {$x_{m,1}$};
            \node [thick] at (\ArrowStartX2-0.5,\ArrowStartY2+0.2) {$x_{m,2}$};
            \node [thick] at (\ArrowStartX3-0.5,\ArrowStartY3-0.2) {$x_{m,N}$};
            \node [thick] at (\ArrowStartX3-0.5,\CircleCenterY1) {$\vdots$};
            \node [thick] at (\ArrowStartX3-0.5,\ArrowStartY3-1) {$\mathbfit{x}_m \in \mathbb{F}_2^N$};
            
            \node [rotate=-30] at (\WeightPositionX1,\WeightPositionY1+0.3) {$w_{1}$};
            \node [rotate=-12] at (\WeightPositionX2,\WeightPositionY2+0.2) {$w_{2}$};
            \node [rotate=36] at (\WeightPositionX3,\WeightPositionY3-0.3) {$w_{N}$};
            \node [rotate=0] at (\WeightPositionX3,\ArrowStartY3-1) {$\mathbfit{w} \in \mathbb{F}_2^N$};
            
        \end{tikzpicture}
    }
    \caption{The XOR unit}
    \label{fig:XOR_unit_appendix}
\end{figure}
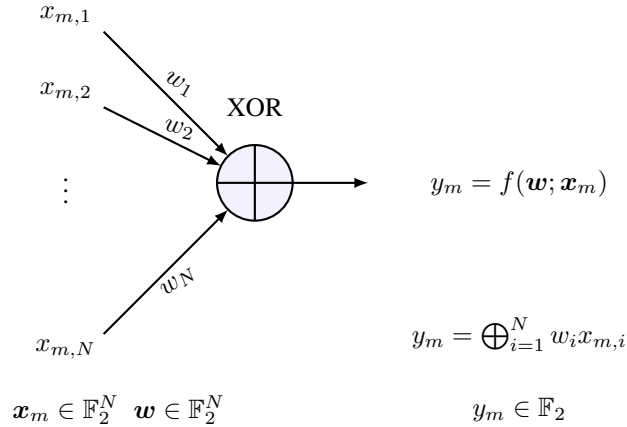
        
        The binary XOR operator is equivalent to a product of bipolar inputs as shown by Eq.~(\ref{eq:xor_product}). Let $\mathbf{b}_m = \{b_{m,i}\}_{i=1}^N \in \mathbb{F}_2^N \: \forall m \in \{1,...,M\}$ be the m-th vector of binary input of the dataset of size $M$ and $\mathbfit{w} = \{w_{i}\}_{i=1}^N \in \mathbb{F}_2^N$ be a set of parameters:
        
        \begin{equation}
        \label{eq:xor_product}
            \begin{aligned}
                y_m &= f(\mathbfit{w};\mathbf{b}_m)\\
                &= \bigoplus_{i=1}^{N}{w_ib_{m,i}} \\
                &= \frac{1}{2}\left(1 - \prod_{i=1}^N w_i(1-2b_{m,i}) + (1-w_i)\right) \\
                &= \frac{1}{2}\left(1 - \prod_{i=1}^N w_ix_{m,i} + (1-w_i)\right) \\
                &= \frac{1}{2}\left(1 - \prod_{i=1}^N\left( w_iz_{m,i} + 1 \right)\right)\\
                &= \frac{1}{2}\left(1 - \prod_{i=1}^N a_{m,i}\right) \quad \forall m \in\{1,\hdots,M\}
            \end{aligned}
        \end{equation}

        where $x_{m,i} := 1 - 2b_{m,i}$ is the bipolar form of binary data ($\{0,1\}$ mapped to $\{+1,-1\}$) and $z_{m,i} = x_{m,i} - 1 = -2b_{m,i}$ and $a_{m,i}=w_iz_{m,i} + 1$. 

        Let $\mathbfit{w} \in \mathbb{R}^N$ be the set of trainable weights of the model to be optimized and $\mathbfit{w}^\mathrm{true} \in \mathbb{F}_2^N$ the set of binary oracle weights. One can defined the unregularized empirical risk with this new node as:
         
        \begin{equation}\label{eq:xor_node_empical_risk}
            \begin{aligned}
                \hat{\ell}_\oplus(\mathbfit{w}) &= \frac{1}{M}\sum_{m=1}^{M}\left(\frac{1}{2}\left(1 - \prod_{i=1}^N a_{m,i} - 1 + \prod_{i=1}^N a_{m,i}^\mathrm{true}\right)\right)^2 \\
                &= \frac{1}{4M}\sum_{m=1}^{M}\left(\prod_{i=1}^N a_{m,i}^\mathrm{true} - \prod_{i=1}^N a_{m,i}\right)^2 \\
                &= \frac{1}{4M}\sum_{m=1}^{M}\left(\prod_{i=1}^N a_{m,i} - \prod_{i=1}^N a_{m,i}^\mathrm{true}\right)^2 = \frac{\hat{\ell}_{\overset{\bullet}{\Pi}}(\mathbfit{w})}{4}\\
            \end{aligned}
        \end{equation}

        The equation of the loss function is proportional by a factor $1/4$ to the equation of the loss in the case of the product node considered in Appendix~\ref{app:product_with_neutral_element_node_analysis}. We provide the gradient and hessian expressions for this node in Proposition~\ref{prop:xor_node_gradient} and \ref{prop:xor_node_hessian} respectively. 
        
        As shown in Appendix~\ref{app:product_with_neutral_element_node_analysis}, Proposition~\ref{prop:sparse_convexity}, if the inputs to the product units contains (during training) at most one element different from one, then the empirical risk is always strictly convex in $\mathbfit{w}$. This can be extended to the XOR node. When the binary inputs vector $\mathbfit{b}_m$ contains at most one '1' (0 otherwise) the inputs $z_{m,i}$ to the product node are all equal to 0 (as $z_{m,i} = x_{m,i} - 1 = - 2b_{m,i}$) except one that is equal to -2; This is a particular case of the more general dataset construction proposed in Appendix~\ref{app:product_with_neutral_element_node_analysis} (see Proposition~\ref{prop:xor_sparse_data}). 

        Note that the weights $w_i$ are defined as continuous parameters in $\mathbb{R}$ to enable gradient-based training, while recovering exact XOR logic when $w_i \in \{0,1\}$. This follows the continuous relaxation approach from \cite{NeuralBeliefPropagationAutoEncoderforLinearBlockCodeDesign,BlindNeuralBeliefPropagation} (see Remark~\ref{remark:continuous_relaxation}). Interestingly, when parameters are restricted to binary values, the expected loss becomes proportional to the Hamming distance between learned and true parameters (see Remark~\ref{remark:hamming_distance}).

        \begin{proposition}[Gradient of XOR Node]\label{prop:xor_node_gradient}
        For the XOR node with empirical risk $\hat{\ell}_\oplus(\mathbfit{w})$, the j-th component of the gradient $\nabla_{\mathbf{w}} \hat{\ell}_\oplus(\mathbf{w})$ is:
        \begin{equation}
            \frac{\partial \hat{\ell}_\oplus(\mathbfit{w}) }{\partial w_j} = \frac{1}{2M}\sum_{m=1}^{M}\left[z_{m,j}\times\left(\prod_{\substack{i=1 \\ i \neq j}}^{N} a_{m,i}\right)\times \left(\prod_{i=1}^{N} a_{m,i} - \prod_{i=1}^{N} a_{m,i}^\mathrm{true}\right)\right]
        \end{equation} 
        where $x_{m,i} := 1 - 2b_{m,i}$ is the bipolar form of binary inputs and $z_{m,i} = x_{m,i} - 1 = -2b_{m,i}$ and $a_{m,i}=w_iz_{m,i} + 1$. 

        \end{proposition}
        
        \begin{proof}
            Since $\hat{\ell}_\oplus(\mathbfit{w}) = \frac{1}{4}\hat{\ell}_{\overset{\bullet}{\Pi}}(\mathbfit{w})$, the gradient $\nabla_{\mathbf{w}}\hat{\ell}_{\overset{\bullet}{\Pi}}(\mathbfit{w})$ (see Proposition~\ref{prop:product_with_ne_gradient}) is scaled by the same factor:
            \begin{equation}
                \frac{\partial \hat{\ell}_\oplus(\mathbfit{w}) }{\partial w_j} = \frac{1}{4} \times \frac{\partial \hat{\ell}_{\overset{\bullet}{\Pi}}(\mathbfit{w}) }{\partial w_j}= \frac{2}{4M}\times\sum_{m=1}^{M}\left[z_{m,j}\times\left(\prod_{\substack{i=1 \\ i \neq j}}^{N} a_{m,i}\right)\times \left(\prod_{i=1}^{N} a_{m,i} - \prod_{i=1}^{N} a_{m,i}^\mathrm{true}\right)\right]\\
            \end{equation}
        \end{proof}

        \begin{proposition}[Hessian of XOR Node]\label{prop:xor_node_hessian}
            The Hessian matrix of the empirical risk $\hat{\ell}_\oplus(\mathbfit{w})$ has entries:
            \begin{equation}
                [\mathbfit{H}_\oplus]_{j,l} = \frac{\partial^2 \hat{\ell}_\oplus(\mathbfit{w})}{\partial w_j \partial w_l} = \frac{1}{2M}\sum_{m=1}^{M}
                \left\{
                \begin{aligned}
                    &z_{m,j}^2\times\prod_{\substack{i=1 \\ i \neq j}}^{N} a_{m,i}^2& \text{if } j = l \\
                    &\left(z_{m,j}z_{m,l}\times \prod_{\substack{i=1 \\ i \neq j \neq l}}^{N} a_{m,i}\right) \times \left(2a_{m,j}a_{m,l}\times\prod_{\substack{i=1 \\ i \neq j \neq l}}^{N} a_{m,i}
                    - \prod_{i=1}^{N} a_{m,i}^\mathrm{true}\right)& \text{if } j \neq l
                \end{aligned}
                \right.
            \end{equation}
            where $x_{m,i} := 1 - 2b_{m,i}$ is the bipolar form of binary inputs and $z_{m,i} = x_{m,i} - 1 = -2b_{m,i}$ and $a_{m,i}=w_iz_{m,i} + 1$. 
        \end{proposition}
        
        \begin{proof}
            Since $\hat{\ell}_\oplus(\mathbfit{w}) = \frac{1}{4}\hat{\ell}_{\overset{\bullet}{\Pi}}(\mathbfit{w})$, the Hessian $\nabla_{\mathbfit{w}}^2 \hat{\ell}_{\overset{\bullet}\Pi}(\mathbfit{w})$ (see Proposition~\ref{prop:product_with_ne_hessian}) is scaled by the same factor:
            \begin{equation}
                [\mathbfit{H}_\oplus]_{j,l} = \frac{1}{4}\times[\mathbfit{H}_{\overset{\bullet}{\Pi}}]_{j,l}=\frac{2}{4M}\times\sum_{m=1}^{M}
                \left\{
                \begin{aligned}
                    &z_{m,j}^2\times\prod_{\substack{i=1 \\ i \neq j}}^{N} a_{m,i}^2& \text{if } j = l \\
                    &\left(z_{m,j}z_{m,l}\times \prod_{\substack{i=1 \\ i \neq j \neq l}}^{N} a_{m,i}\right) \times \left(2a_{m,j}a_{m,l}\times\prod_{\substack{i=1 \\ i \neq j \neq l}}^{N} a_{m,i}
                    - \prod_{i=1}^{N} a_{m,i}^\mathrm{true}\right)& \text{if } j \neq l
                \end{aligned}
                \right.
            \end{equation}
        \end{proof}

        \begin{proposition}[Convexity of XOR Node Under Sparse Binary Inputs]\label{prop:xor_sparse_data}
            Consider the XOR node defined by $y_m = \frac{1}{2}\left(1 - \prod_{i=1}^N\left( w_iz_{m,i} + 1 \right)\right)$ where $z_{m,i} = -2b_{m,i}$ for binary inputs $b_{m,i} \in \{0,1\}$. Under the sparse constraint that each binary input vector $\mathbfit{b}_m$ contains at most one '1' entry, and assuming the dataset has full coverage (i.e., for each coordinate $i \in \{1,\ldots,N\}$, there exists at least one sample $m$ with $b_{m,i} = 1$), the regularized empirical risk is strictly convex in $\mathbfit{w} \in \mathbb{R}^N$.
        \end{proposition}

        \begin{proof}
            Under the sparse binary constraint, each input vector $\mathbfit{b}_m$ has at most one '1' entry. Non-zero samples correspond to the sparse structure from Proposition~\ref{prop:sparse_convexity} with $\mathbfit{z}_m = -2\mathbfit{e}^{(k_m)}$ and $v_m = -2$. All-zero samples ($\mathbf{b}_m = \mathbf{0}$) contribute zero to both gradient and Hessian.
            
            Since $\hat{\ell}_\oplus(\mathbfit{w}) = \frac{1}{4}\hat{\ell}_{\overset{\bullet}{\Pi}}(\mathbfit{w})$, strict convexity follows from Proposition~\ref{prop:sparse_convexity} by scaling invariance. The full coverage assumption ensures that every weight $w_i$ appears in at least one sample, guaranteeing that the empirical Hessian is positive definite. 
        \end{proof}

        \begin{remark}[Continuous Relaxation for Trainability]\label{remark:continuous_relaxation}
            While the XOR operation is inherently discrete with binary parameters $\mathbfit{w} \in \mathbb{F}_2^N$, practical gradient descent optimization requires continuous, differentiable parameters. We therefore consider a continuous relaxation where $\mathbfit{w} \in \mathbb{R}^N$ are real-valued trainable weights, while the oracle parameters $\mathbfit{w}^{\mathrm{true}} \in \mathbb{F}_2^N$ remain binary.
            
            This formulation creates a continuous extension of the XOR function: when the learned weights $\mathbfit{w}$ converge to binary values $\{0,1\}$, the node output exactly recovers the discrete XOR operation $\bigoplus_{i=1}^{N} w_i b_{m,i}$. However, during training, the continuous parameterization $\mathbfit{w} \in \mathbb{R}^N$ enables gradient-based optimization while preserving the multiplicative structure that characterizes XOR logic.
            
            This approach is analogous to other continuous relaxations in machine learning (e.g., Gumbel-Softmax for discrete sampling), where continuous approximations enable tractable optimization of inherently discrete problems.
        \end{remark}

        \begin{remark}[Connection to Hamming Distance]\label{remark:hamming_distance}
            Complementarily to Remark~\ref{remark:continuous_relaxation}, as the real-valued weights converge toward binary values under gradient descent, minimizing the empirical risk in the continuous space becomes equivalent to minimizing the Hamming distance to the true binary solution. When restricting to binary parameter vectors $\mathbfit{w} \in \mathbb{F}_2^N$, the empirical risk is no longer differentiable w.r.t. $\mathbfit{w}$, but its expectation becomes proportional to the Hamming distance under the sparse dataset structure:
            \begin{equation}
                \begin{aligned}
                    \mathbb{E}_{\mathrm{v}_m \sim V, \mathrm{k}_m \sim U_N}\left\{\hat{\ell}_\oplus(\mathbfit{w})\right\}  &= \mathbb{E}_{\mathrm{v}_m \sim V, \mathrm{k}_m \sim U_N}\left\{\frac{\hat{\ell}_{\overset{\bullet}{\Pi}}(\mathbfit{w})}{4}\right\}\\
                    &=\frac{1}{4N} \times \mathbb{E}_{\mathrm{v}_m \sim V} \left\{\mathrm{v}_m^2\right\} \times \sum_{i=1}^{N}\left(w_{i} - w_{i}^\mathrm{true}\right)^2\\
                    &=\frac{1}{4N} \times \mathbb{E}_{\mathrm{v}_m \sim V} \left\{\mathrm{v}_m^2\right\} \times \sum_{i=1}^{N}|w_{i} - w_{i}^\mathrm{true}| \quad \text{(since } w_i, w_i^{\mathrm{true}} \in \{0,1\}\text{)}\\
                    &=\frac{1}{4N} \times \mathbb{E}_{\mathrm{v}_m \sim V} \left\{\mathrm{v}_m^2\right\} \times \delta_H(\mathbfit{w},\mathbfit{w}^\mathrm{true}) \propto \delta_H(\mathbfit{w},\mathbfit{w}^\mathrm{true})
                \end{aligned}
            \end{equation}
            where $\delta_H(\mathbfit{w},\mathbfit{w}^\mathrm{true})$ is the Hamming distance between oracle parameters $\mathbfit{w}^\mathrm{true}$ and model parameters $\mathbfit{w}$.
            
            
        \end{remark}

    \section{Analysis of Parity Learning under Stochastic Sparsity}\label{app:probabilistic_sparsity_analysis}

        From the previous analysis (Appendix~\ref{app:xor_node_analysis}), we established that XOR nodes achieve strictly convex optimization when input vectors contain at most one '1' entry (one-hot structure). This section extends that deterministic constraint to a probabilistic framework. 

        Consider XOR nodes operating under a Bernoulli input model where each element of the input vectors $b_{m,i} \in \{0,1\}$ is i.i.d. and equals one with probability $p_e$: $b_{m,i} \sim B(p_e)$. The number of active bits within a vector of size $N$ follows: $n_e \sim \text{Binomial}(N, p_e)$. Hence, the probability of having exactly $q$ active bits is: 
        
        $$\mathbb{P}_N\{n_e = q\} = \binom{N}{q}p_e^q(1-p_e)^{N-q}$$ 
        
        Under this probabilistic model, the favorable one-hot structure occurs with probability $\mathbb{P}_N\{n_e = 1\} = Np_e(1-p_e)^{N-1}$. This section investigates whether XOR nodes can effectively learn when the strict convexity condition (exactly one '1' per input) is only satisfied probabilistically. We analyze the convergence behavior when training data contains a mixture of favorable one-hot vectors and potentially problematic multi-bit vectors. The expected gradient for the studied XOR node under this input distribution model is provided in Proposition~\ref{prop:xor_node_gradient_stochastic_sparsity}. Under the assumption of weights distributed such that $\mathbb{E}\{w_i\} = 1/2$ (condition that we can easily meet at model initialization), the gradient properties and optimal sparsity are derived in Proposition~\ref{prop:optimal_error_probability}.

        \begin{proposition}[Expected Gradient of XOR Node Under Stochastic Sparsity]\label{prop:xor_node_gradient_stochastic_sparsity}
            For the XOR node with empirical risk $\hat{\ell}_\oplus(\mathbfit{w})$, when the inputs of the node $z_{m,i}:= -2b_{m,i}$ follow a Bernoulli distribution with probability $p_e$ such that $b_{m,i} \sim B(p_e)$,the j-th component of the expected gradient $\mathbb{E}_{\mathrm{z} \sim Z}\left\{\nabla_{\mathbf{w}} \hat{\ell}_\oplus(\mathbf{w})\right\}$ is:
            \begin{equation}
                \begin{aligned}
                \mathbb{E}_{\mathrm{z} \sim Z}\left\{\frac{\partial \hat{\ell}_\oplus(\mathbfit{w}) }{\partial w_j}\right\} =   p_e&\left[(2w_j-1)\times\prod_{\substack{i=1 \\ i \neq j}}^{N} \left(4p_ew_i^2-4p_ew_i+1\right)\right.\\
                &\left.- (2w_j^\mathrm{true}-1)\times\prod_{\substack{i=1 \\ i \neq j}}^{N} \left(4p_ew_iw_i^\mathrm{true}-2p_e(w_i+w_i^\mathrm{true})+1\right)\right]\\
                \end{aligned}
            \end{equation}  
        \end{proposition}
        \begin{proof}
            We recall the gradient of the XOR node (Proposition~\ref{prop:xor_node_gradient}):
            \begin{equation}
                \frac{\partial \hat{\ell}_\oplus(\mathbfit{w}) }{\partial w_j} = \frac{1}{2M}\sum_{m=1}^{M}\left[z_{m,j}\times\left(\prod_{\substack{i=1 \\ i \neq j}}^{N} a_{m,i}\right)\times \left(\prod_{i=1}^{N} a_{m,i} - \prod_{i=1}^{N} a_{m,i}^\mathrm{true}\right)\right]
            \end{equation}  
            where $z_{m,i} = -2b_{m,i}$ and $a_{m,i}=w_iz_{m,i} + 1$. 
    
            Under the proposed stochastic dataset the expected gradient is\footnote{Since $\mathrm{a}$ is entirely determined by $\mathrm{z}$ (or equivalently $\mathrm{b}$), we use the notation $\mathbb{E}_{\mathrm{z} \sim Z}$ (or equivalently $\mathbb{E}_{\mathrm{b} \sim B}$) interchangeably when taking expectations over variables $\mathrm{z}$, $\mathrm{b}$ or $\mathrm{a}$.}:
            \begin{equation}
                \begin{aligned}
                \mathbb{E}_{\mathrm{z} \sim Z}\left\{\frac{\partial \hat{\ell}_\oplus(\mathbfit{w}) }{\partial w_j}\right\} &= \mathbb{E}_{\mathrm{z} \sim Z}\left\{\frac{1}{2M}\sum_{m=1}^{M}\left[\mathrm{z}_{m,j}\times\left(\prod_{\substack{i=1 \\ i \neq j}}^{N} \mathrm{a}_{m,i}\right)\times \left(\prod_{i=1}^{N} \mathrm{a}_{m,i} - \prod_{i=1}^{N} \mathrm{a}_{m,i}^\mathrm{true}\right)\right]\right\}\\
                &= \frac{1}{2}\mathbb{E}_{\mathrm{z} \sim Z}\left\{\mathrm{z}_{j}\times\left(\prod_{\substack{i=1 \\ i \neq j}}^{N} \mathrm{a}_{i}\right)\times \left(\prod_{i=1}^{N} \mathrm{a}_{i} - \prod_{i=1}^{N} \mathrm{a}_{i}^\mathrm{true}\right)\right\}\\
                &= \frac{1}{2}\mathbb{E}_{\mathrm{z} \sim Z}\left\{\mathrm{z}_{j}\mathrm{a}_{j}\times\left(\prod_{\substack{i=1 \\ i \neq j}}^{N} \mathrm{a}_{i}^2\right) - \mathrm{z}_{j}\mathrm{a}_{j}^\mathrm{true}\times\left(\prod_{\substack{i=1 \\ i \neq j}}^{N} a_{i}\right) \times \left(\prod_{\substack{i=1 \\ i \neq j}}^{N} \mathrm{a}_{i}^\mathrm{true}\right)\right\}\\
                &= \frac{1}{2}\left[\mathbb{E}_{\mathrm{z} \sim Z}\left\{\mathrm{z}_{j}\mathrm{a}_{j}\times\left(\prod_{\substack{i=1 \\ i \neq j}}^{N} \mathrm{a}_{i}^2\right)\right\} - \mathbb{E}_{\mathrm{z} \sim Z}\left\{\mathrm{z}_{j}\mathrm{a}_{j}^\mathrm{true}\times\left(\prod_{\substack{i=1 \\ i \neq j}}^{N} \mathrm{a}_{i}\right) \times \left(\prod_{\substack{i=1 \\ i \neq j}}^{N} \mathrm{a}_{i}^\mathrm{true}\right)\right\}\right]\\
                &= \frac{1}{2}\left[\mathbb{E}_{\mathrm{z} \sim Z}\left\{\mathrm{z}_j\mathrm{a}_j\right\}\times\mathbb{E}_{\mathrm{z} \sim Z}\left\{\prod_{\substack{i=1 \\ i \neq j}}^{N} \mathrm{a}_{i}^2\right\} - \mathbb{E}_{\mathrm{z} \sim Z}\left\{\mathrm{z}_j\mathrm{a}_j^\mathrm{true}\right\}\times\mathbb{E}_{\mathrm{z} \sim Z}\left\{\prod_{\substack{i=1 \\ i \neq j}}^{N} \mathrm{a}_{i}\mathrm{a}_{i}^\mathrm{true}\right\}\right]\\
                &= \frac{1}{2}\left[\mathbb{E}_{\mathrm{z} \sim Z}\left\{\mathrm{z}_j\mathrm{a}_j\right\}\times\prod_{\substack{i=1 \\ i \neq j}}^{N} \mathbb{E}_{\mathrm{z} \sim Z}\left\{\mathrm{a}_{i}^2\right\} - \mathbb{E}_{\mathrm{z} \sim Z}\left\{\mathrm{z}_j\mathrm{a}_j^\mathrm{true}\right\}\times\prod_{\substack{i=1 \\ i \neq j}}^{N} \mathbb{E}_{\mathrm{z} \sim Z}\left\{\mathrm{a}_{i}\mathrm{a}_{i}^\mathrm{true}\right\}\right] \quad \text{(i.i.d. variables)}\\
                \end{aligned}
            \end{equation} 
    
            Given that $z_i=-2b_i$ and $a_i=w_iz_i+1=1-2w_ib_i$ we have:
            \begin{itemize}
                \item $\mathbb{E}_{\mathrm{z} \sim Z}\left\{\mathrm{z}_i\mathrm{a}_i\right\} = \mathbb{E}_{\mathrm{b} \sim B}\left\{-2\mathrm{b}(1-2w_i\mathrm{b})\right\} = \mathbb{E}_{\mathrm{b} \sim B}\left\{4w_i\mathrm{b}^2-2\mathrm{b}\right\} = 4w_i\mathbb{E}_{\mathrm{b} \sim B}\left\{\mathrm{b}^2\right\}-2\mathbb{E}_{\mathrm{b} \sim B}\left\{\mathrm{b}\right\}=2p_e(2w_i-1)$
                \item $\mathbb{E}_{\mathrm{z} \sim Z}\left\{\mathrm{a}_{i}^2\right\} = \mathbb{E}_{\mathrm{b} \sim B}\left\{(1-2w_i\mathrm{b})^2\right\} = \mathbb{E}_{\mathrm{b} \sim B}\left\{4(w_i\mathrm{b})^2-4w_i\mathrm{b} + 1\right\} = 4p_ew_i^2-4p_ew_i+1$
                \item $\mathbb{E}_{\mathrm{z} \sim Z}\left\{\mathrm{z}_i\mathrm{a}_i^\mathrm{true}\right\} = \mathbb{E}_{\mathrm{b} \sim B}\left\{4w_i^\mathrm{true}\mathrm{b}^2-2\mathrm{b}\right\} = 4w_i^\mathrm{true}\mathbb{E}_{\mathrm{b} \sim B}\left\{\mathrm{b}^2\right\}-2\mathbb{E}_{\mathrm{b} \sim B}\left\{\mathrm{b}\right\}=2p_e(2w_i^\mathrm{true}-1)$
                \item $\mathbb{E}_{\mathrm{z} \sim Z}\left\{\mathrm{a}_{i}\mathrm{a}_{i}^\mathrm{true}\right\} = \mathbb{E}_{\mathrm{b} \sim B}\left\{(1-2w_i\mathrm{b})\times(1-2w_i^\mathrm{true}\mathrm{b})\right\} = \mathbb{E}_{\mathrm{b} \sim B}\left\{4\mathrm{b}^2w_iw_i^\mathrm{true}-2\mathrm{b}(w_i+w_i^\mathrm{true}) + 1\right\} = 4p_ew_iw_i^\mathrm{true}-2p_e(w_i+w_i^\mathrm{true})+1$
            \end{itemize}
    
            Hence, the gradient is:
            \begin{equation}
                \begin{aligned}
                \mathbb{E}_{\mathrm{z} \sim Z}\left\{\frac{\partial \hat{\ell}_\oplus(\mathbfit{w}) }{\partial w_j}\right\}        = p_e&\left[(2w_j-1)\times\prod_{\substack{i=1 \\ i \neq j}}^{N} \left(4p_ew_i^2-4p_ew_i+1\right)\right.\\
                &\left.- (2w_j^\mathrm{true}-1)\times\prod_{\substack{i=1 \\ i \neq j}}^{N} \left(4p_ew_iw_i^\mathrm{true}-2p_e(w_i+w_i^\mathrm{true})+1\right)\right]\\
                \end{aligned}
            \end{equation} 

        \end{proof}

        \begin{proposition}[Optimal Sparsity at Initialization]\label{prop:optimal_error_probability}
            When the model's weights are distributed such that $\mathbb{E}_{\mathrm{w} \sim D}\{\mathrm{w}_i\} = 1/2$ and the oracle weights are binary, the expected gradient has the correct sign to move weights contained within $(0,1)$ in the direction of their oracle values under a gradient descent update. The data sparsity $p_e = 1/N$ maximizes the expected gradient magnitude at initialization. We note that this value corresponds to the probability that maximizes the chance of observing input vectors of size $N$ with one-hot structure.
        \end{proposition}

        \begin{proof}\label{proof:optimal_error_probability}
            We recall the expression of the gradient from Proposition~\ref{prop:xor_node_gradient_stochastic_sparsity}:
            \begin{equation}
                \begin{aligned}
                \mathbb{E}_{\mathrm{z} \sim Z}\left\{\frac{\partial \hat{\ell}_\oplus(\mathbfit{w}) }{\partial w_j}\right\}        = p_e&\left[(2w_j-1)\times\prod_{\substack{i=1 \\ i \neq j}}^{N} \left(4p_ew_i^2-4p_ew_i+1\right)\right.\\
                &\left.- (2w_j^\mathrm{true}-1)\times\prod_{\substack{i=1 \\ i \neq j}}^{N} \left(4p_ew_iw_i^\mathrm{true}-2p_e(w_i+w_i^\mathrm{true})+1\right)\right]\\
                \end{aligned}
            \end{equation} 

            We take expectation over the value of the model's weights $\mathbfit{w}$:
            \begin{equation}
                \begin{aligned}
                \mathbb{E}_{\mathrm{z} \sim Z,\mathrm{w}\sim D}\left\{\frac{\partial \hat{\ell}_\oplus(\mathbfit{w}) }{\partial w_j}\right\} &= p_e\times\mathbb{E}_{\mathrm{w}\sim D}\left\{(2\mathrm{w}_j-1)\times\prod_{\substack{i=1 \\ i \neq j}}^{N} \left(4p_e\mathrm{w}_i^2-4p_e\mathrm{w}_i+1\right)\right.\\
                &\left.- (2w_j^\mathrm{true}-1)\times\prod_{\substack{i=1 \\ i \neq j}}^{N} \left(4p_e\mathrm{w}_iw_i^\mathrm{true}-2p_e(\mathrm{w}_i+w_i^\mathrm{true})+1\right)\right\}\\
                &=- p_e(2w_j^\mathrm{true}-1)\times\mathbb{E}_{\mathrm{w}\sim D}\left\{\prod_{\substack{i=1 \\ i \neq j}}^{N} \left(4p_e\mathrm{w}_iw_i^\mathrm{true}-2p_e(\mathrm{w}_i+w_i^\mathrm{true})+1\right)\right\}\\
                &=- p_e(2w_j^\mathrm{true}-1)\times\prod_{\substack{i=1 \\ i \neq j}}^{N} \mathbb{E}_{\mathrm{w}\sim D}\left\{4p_e\mathrm{w}_iw_i^\mathrm{true}-2p_e(\mathrm{w}_i+w_i^\mathrm{true})+1\right\}\\
                &=- p_e(2w_j^\mathrm{true}-1)\times\prod_{\substack{i=1 \\ i \neq j}}^{N} \left(2p_ew_i^\mathrm{true}-p_e-2p_ew_i^\mathrm{true}+1\right)\\
                &=- p_e(2w_j^\mathrm{true}-1)\times\prod_{\substack{i=1 \\ i \neq j}}^{N} \left(1-p_e\right)\\
                &=- (2w_j^\mathrm{true}-1)\times p_e\left(1-p_e\right)^{N-1}\\
                \end{aligned}
            \end{equation} 

            We can outline several interesting properties:

            \textbf{Gradient Correctness} Under a standard gradient descent update $\mathbfit{w}[k+1] = \mathbfit{w}[k] - \alpha\nabla_{\mathbf{w}} \hat{\ell}_\oplus(\mathbfit{w})$, the expected gradient has the correct sign to move weights contained within $(0,1)$ in the direction of their oracle values: negative gradient when $w_j^\mathrm{true}=1$ (pushing weights toward 1) and positive gradient when $w_j^\mathrm{true}=0$ (pushing weights toward 0).

            \textbf{Optimal Sparsity} The expected gradient magnitude is maximized at sparsity $p_e = 1/N$, which corresponds to the probability that maximizes the likelihood of observing exactly one active input (one-hot structure):
            \begin{equation}
                \begin{aligned}
                    \frac{\partial}{\partial p_e} p_e(1-p_e)^{N-1}= \left(1-p_e\right)^{N-1}-p_e(N-1)\times\left(1-p_e\right)^{N-2} =(1-p_e)^{N-2}\times(1-Np_e)\\
                \end{aligned}
            \end{equation}
            with $p_e \in [0,1]$ (probability):
            \begin{equation}
                \frac{\partial}{\partial p_e}p_e(1-p_e)^{N-1} = 0 \Leftrightarrow (1-p_e)^{N-2}(1-Np_e) = 0 \Leftrightarrow p_e = 1/N
            \end{equation}

            Hence the maximum gradient magnitude is obtained for $p_e=1/N$, which coincides with:
            \begin{equation}
                \arg\max_{p_e}\mathbb{P}_N\{n_e = 1\} = \arg\max_{p_e}Np_e(1-p_e)^{N-1}=1/N
            \end{equation}
            where $n_e$ denotes the number of active input features. 

            \textbf{Sparsity Regimes} The expected gradient magnitude exhibits distinct behavior across sparsity regimes, as illustrated in Figure~\ref{fig:gradient_expectation}. The gradient expression $\propto|p_e(1-p_e)^{N-1}|$ reveals two distinct regimes:
            \begin{itemize}
                \item \textbf{Ultra-Sparse Regime ($p_e \ll 1/N$):} The gradient magnitude increases approximately linearly with $p_e$, reaching its maximum at $p_e = 1/N$.
                \item \textbf{Dense Regime ($p_e \gg 1/N$):} The gradient magnitude decreases exponentially with $p_e$ following the $(1-p_e)^{N-1}$ term.
            \end{itemize}

            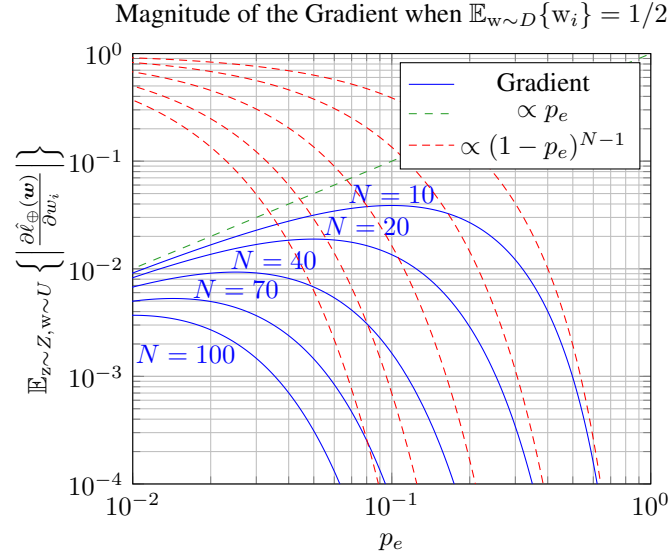
\begin{figure}
    \centering
    \begin{tikzpicture}[scale=1]
        \begin{axis}[
            title={Magnitude of the Gradient when $\mathbb{E}_{\mathrm{w} \sim D}\{\mathrm{w}_i\} = 1/2$},
            xmin=0.01, 
            xmax=1, 
            xminorgrids=true,
            xmajorgrids=true,
            xmode=log,
            xlabel = \(p_e\), 
            ymin=0.0001, 
            ymax=1, 
            yminorgrids=true,
            ymajorgrids=true,
            ymode=log,
            ylabel = \(\mathbb{E}_{\mathrm{z} \sim Z, \mathrm{w} \sim U}\left\{\left|\frac{\partial\hat{\ell}_{\oplus} (\mathbfit{w})}{\partial w_i}\right|\right\}\),
        ]   
            \definecolor{gradient}{RGB}{0,0,255};        
            \definecolor{linear}{RGB}{44,160,44};        
            \definecolor{exponential}{RGB}{255,0,0};     
            
            \def\n{10};
            \addplot [domain=0.001:1,samples=400,color=gradient,mark=none] {x*(1-x)^(\n-1)};
            \addlegendentry{Gradient};
            
            \addplot [domain=0.001:1,samples=400,color=linear,dashed,mark=none] {x};
            \addlegendentry{\(\propto p_e\)};

            \addplot [domain=0.001:1,samples=400,color=exponential,densely dashed,mark=none] {(1-x)^(\n-1)};
            \addlegendentry{\(\propto (1-p_e)^{N-1}\)};

            \node at (axis cs: 0.1, 0.05) [color=gradient] {$N=10$};

            \def\n{20};
            \addplot [domain=0.001:1,samples=400,color=gradient,mark=none] {x*(1-x)^(\n-1)};
            \addplot [domain=0.001:1,samples=400,color=exponential,densely dashed,mark=none] {(1-x)^(\n-1)};
            \node at (axis cs: 0.08, 0.025) [color=gradient] {$N=20$};
            

            \def\n{40};
            \addplot [domain=0.001:1,samples=400,color=gradient,mark=none] {x*(1-x)^(\n-1)};
            \addplot [domain=0.001:1,samples=400,color=exponential,densely dashed,mark=none] {(1-x)^(\n-1)};
            \node at (axis cs: 0.035, 0.012) [color=gradient] {$N=40$};
            
            \def\n{70};
            \addplot [domain=0.001:1,samples=400,color=gradient,mark=none] {x*(1-x)^(\n-1)};
            \addplot [domain=0.001:1,samples=400,color=exponential,densely dashed,mark=none] {(1-x)^(\n-1)};
            \node at (axis cs: 0.025, 0.0065) [color=gradient] {$N=70$};

            \def\n{100};
            \addplot [domain=0.001:1,samples=400,color=gradient,mark=none] {x*(1-x)^(\n-1)};
            \addplot [domain=0.001:1,samples=400,color=exponential,densely dashed,mark=none] {(1-x)^(\n-1)};
            \node at (axis cs: 0.016, 0.0016) [color=gradient] {$N=100$};
            
        \end{axis}
    \end{tikzpicture}
    \caption{Expected gradient magnitude as a function of sparsity $p_e$ for various input dimensions $N$. The log-log scale reveals two distinct regimes: ultra-sparse ($p_e \ll 1/N$) with linear growth, and dense ($p_e \gg 1/N$) with exponential decay. The optimal sparsity $p_e = 1/N$ maximizes gradient magnitude.}
    \label{fig:gradient_expectation}
\end{figure}

            
            

            This aligns with the intuition that favoring sparse inputs with one-hot structure provides optimal gradient signal, while dense inputs lead to exponentially vanishing gradients that worsen with increasing input dimension N, potentially causing complete learning collapse in high-dimensional settings.
        \end{proof}

    \subsection{General Definitions and Properties under Normally Distributed Weights}\label{app:def_and_prop_normally_distributed_weights}
        This section establishes usefull properties of XOR noded learning dynamics when model weights follow Gaussian distributions. We begin by deriving the expected gradient under the assumption that weights are i.i.d. according to two Gaussian distributions, one for weights corresponding to oracle targets of 0, and another for oracle targets of 1 (Proposition~\ref{prop:gaussian_gradient_expression}). We prove that Gaussian distributions are preserved under gradient descent updates (Proposition~\ref{prop:gaussian_conservation}), allowing us to work with the same analytical expressions at each iteration. We further show that if the two weight distributions are initially symmetric w.r.t. each other around 1/2, this symmetry is maintained throughout training (Proposition~\ref{prop:symmetry_conservation}), significantly simplifying the gradient expressions and making convergence properties invariant to the proportion of oracle weights. These conservation properties provide the mathematical foundation for our subsequent convergence analysis.
          
        \begin{remark}[Notation Convention]\label{remark:notation_convention}
            Throughout the subsequent sections, we study model evolution as a function of training step $k$. To streamline notation, we use the bracket $[k]$ to indicate dependence on iteration-varying quantities. Specifically, any function $f[k]$ denotes $f(\mu[k], \sigma^2[k], p_e, \ldots)$ where all time-dependent variables are evaluated at step $k$, while fixed parameters (such as $N$, $p_e$, $\alpha$) are suppressed from the notation but remain implicit.
            
            For example: in Eq.~\eqref{eq:compact_notations}, $A_0[k] := A_0(\mu_0[k], \sigma_0^2[k], p_e)$ and $B_1[k] := B_1(\mu_1[k],p_e)$.
        \end{remark}

        \begin{proposition}[Expected Gradient under Gaussian Weight Distributions]\label{prop:gaussian_gradient_expression}
            Consider the XOR learning problem under Bernoulli input model with error probability $p_e$. Let weights $w_j[k]$ be independently distributed according to Gaussian distributions $D_0[k] = \mathcal{N}(\mu_0[k], \sigma_0^2[k])$ and $D_1[k] = \mathcal{N}(\mu_1[k], \sigma_1^2[k])$ for weights corresponding to oracle targets $w_j^{\mathrm{true}} = 0$ and $w_j^{\mathrm{true}} = 1$ respectively, with proportion $p_w$ of oracle weights equal to one.
            
            Then the expected gradient with respect to weight $w_i[k]$ is given by:
            \begin{equation}
                \begin{aligned}
                    \mathbb{E}_{z \sim Z, w_j \sim D}\Bigg\{\frac{\partial \hat{\ell}_{\oplus}(\mathbf{w})[k]}{\partial w_i[k]}\Bigg\} = p_e \Bigg[&(2w_i[k]-1)\times\bigg((1-p_w^{(\setminus i)}) \times A_0[k]+p_w^{(\setminus i)} \times A_1[k]\bigg)^{N-1}\\
                    -&(2w_i^\mathrm{true}-1)\times \bigg((1-p_w^{(\setminus i)}) \times B_0[k]+p_w^{(\setminus i)} \times B_1[k]\bigg)^{N-1}\Bigg]\\
                \end{aligned}
            \end{equation}
    
            where
            \begin{equation}\label{eq:compact_notations}
                \left\{
                \begin{aligned}
                    A_0[k]&=4p_e(\mu_0^2[k]+\sigma_0^2[k])-4p_e\mu_0[k]+1\\
                    A_1[k]&=4p_e(\mu_1^2[k]+\sigma_1^2[k])-4p_e\mu_1[k]+1\\
                    B_0[k]&=1-2p_e\mu_0[k]\\
                    B_1[k]&=1-2p_e(1-\mu_1[k])\\
                \end{aligned}
                \right.
            \end{equation} 
                
            where $p_w^{(\setminus i)}$ denotes the oracle weight proportion when considering all weights except $w_i$. For large $N$, we can assume  $p_w^{(\setminus i)} \approx p_w$, the overall oracle weight proportion.
        \end{proposition}
        
        \begin{proof}
            We recall the general expression of the expectation of the $i$-th term of the expected gradient at step $k$ under Bernoulli input model (Proposition~\ref{prop:xor_node_gradient_stochastic_sparsity}):
            
            \begin{equation}
                \begin{aligned}
                    \mathbb{E}_{\mathrm{z} \sim Z}\Bigg\{\frac{\partial \hat{\ell}_{\oplus}(\mathbfit{w}[k])}{\partial w_i[k]}\Bigg\} =&p_e\left[(2w_i[k] -1)\times \prod_{j\neq i}\Big(4p_ew_j^2[k] - 4p_ew_j[k] + 1 \Big)\right.\\
                    &\left.- (2w_i^\mathrm{true} -1)\times \prod_{j\neq i}\Big( 4p_ew_j[k]w_j^\mathrm{true} - 2p_e(w_j[k] + w_j^\mathrm{true}) + 1 \Big)\right]
                \end{aligned}
            \end{equation}
    
            Let's assume that the parameters are i.i.d and follow an arbitrary distribution $\mathrm{w}_j[k]\sim D[k]$, and the proportion of oracle weights equal to one is defined as $p_w$, Let's further assume that all parameter $w_j[k]$ associated to a parameter $w_j^\mathrm{true}=0$ ($w_j^\mathrm{true}=1$ respectively) follow a distribution $D_0[k]$ ($D_1[k]$ respectively). Then, the expectation of the gradient can be written as\footnote{Note that we take the expectation over all weights $w_j$ with $j \neq i$, denoted as $\mathbb{E}_{w_{j} \sim D}$. This is valid due to the product structure of XOR operations, where the gradient for weight $w_i$ factorizes into terms depending on $w_i$ (and similarly $w_i^\mathrm{true}$) and products over all other weights that can be treated as i.i.d. random variables with finite moments. While it is always possible to define the expected gradient this way, we should not expect the empirical gradient to converge to its expectation \textbf{in the general case} due to the complex product interactions between variables. Yet, we can expect convergence under large dataset size $M$, weights number $N$ and sparse input structure where the problematic product interactions are significantly reduced. The archetypal case is deterministic unit sparsity (i.e., one-hot distributed inputs) where the product interaction reduces to a single term, leaving only summations terms and enabling standard Law of Large Numbers convergence (see Proposition~\ref{prop:sparse_convexity}).}:
    
            \begin{equation}
                \begin{aligned}
                    \mathbb{E}_{\mathrm{z} \sim Z, \mathrm{w}_j \sim D}\Bigg\{\frac{\partial \hat{\ell}_{\oplus}(\mathbf{w}[k])}{\partial w_i[k]}\Bigg\} &=p_e\Bigg[(2w_i[k] -1)\times \mathbb{E}_{ \mathrm{w}_j \sim D}\Big\{4p_ew_j^2[k] - 4p_ew_j[k] + 1 \Big\}^{N-1}\\
                    &\quad \quad - (2w_i^\mathrm{true} -1)\times \mathbb{E}_{ \mathrm{w}_j \sim D}\Big\{4p_ew_j[k]w_j^\mathrm{true} - 2p_e(w_j[k] + w_j^\mathrm{true}) + 1 \Big\}^{N-1}\Bigg]\\
                    &=p_e\Bigg[(2w_i[k] -1)\times \bigg((1-p_w^{(\setminus i)})\times\mathbb{E}_{ \mathrm{w}_j \sim D_0}\Big\{4p_ew_j^2[k] - 4p_ew_j[k] + 1\Big\}\\
                    &\quad \quad \quad \quad \quad \quad \quad \quad \quad \quad +p_w^{(\setminus i)}\times\mathbb{E}_{ \mathrm{w}_j \sim D_1}\Big\{4p_ew_j^2[k] - 4p_ew_j[k] + 1\Big\}\bigg)^{N-1}\\
                    &\quad \quad - (2w_i^\mathrm{true} -1)\times \bigg((1-p_w^{(\setminus i)})\times\mathbb{E}_{ \mathrm{w}_j \sim D_0}\Big\{- 2p_ew_j[k] + 1\Big\}\\
                    &\quad \quad \quad \quad \quad \quad \quad \quad \quad \quad p_w^{(\setminus i)}\times\mathbb{E}_{ \mathrm{w}_j \sim D_1}\Big\{4p_ew_j[k] - 2p_e(w_j[k] + 1) + 1\Big\}\bigg)^{N-1}\Bigg]\\
                \end{aligned}
            \end{equation}
            where $p_w^{(\setminus i)}$ denotes the oracle weight proportion when considering all weights except $w_i$. For large $N$, we can safely assume that $p_w^{(\setminus i)} \approx p_w$.

            If $D_0[k]$ and $D_1[k]$ are Gaussian distributions we can rewrite the expression of the gradient as:
            \begin{equation}\label{eq:double_gaussian_gradient_xor_node}
                \begin{aligned}
                    &\mathbb{E}_{\mathrm{z} \sim Z, \mathrm{w}_j \sim D}\Bigg\{\frac{\partial \hat{\ell}_{\oplus}(\mathbf{w})[k]}{\partial w_i[k]}\Bigg\} = p_e \Bigg[&(2w_i[k]-1)\times\bigg(&(1-p_w^{(\setminus i)})\times\Big(4p_e(\mu_0^2[k]+\sigma_0^2[k])-4p_e\mu_0[k]+1\Big)\\
                    &&&+p_w^{(\setminus i)}\times\Big(4p_e(\mu_1^2[k]+\sigma_1^2[k])-4p_e\mu_1[k]+1\Big)\bigg)^{N-1}\\
                    &&-(2w_i^\mathrm{true}-1)\times \bigg(&(1-p_w^{(\setminus i)})\times \Big(1-2p_e\mu_0[k]\Big)\\
                    &&&+p_w^{(\setminus i)}\times\Big(2p_e\mu_1[k]-2p_e+1\Big)\bigg)^{N-1}\Bigg]\\
                \end{aligned}
            \end{equation}
            where $\mu_0[k]$ ($\mu_1[k]$ respectively) and $\sigma_0^2[k]$ ($\sigma_1^2[k]$ respectively) are the mean and variance of the distribution $D_0[k]$ ($D_1[k]$ respectively) at step $k$.
    
            We compactly note this expression as:
            \begin{equation}
                \begin{aligned}
                    \mathbb{E}_{z \sim Z, w_j \sim D}\Bigg\{\frac{\partial \hat{\ell}_{\oplus}(\mathbf{w})[k]}{\partial w_i[k]}\Bigg\} = p_e \Bigg[&(2w_i[k]-1)\times\bigg((1-p_w^{(\setminus i)}) \times A_0[k]+p_w^{(\setminus i)} \times A_1[k]\bigg)^{N-1}\\
                    -&(2w_i^\mathrm{true}-1)\times \bigg((1-p_w^{(\setminus i)}) \times B_0[k]+p_w^{(\setminus i)} \times B_1[k]\bigg)^{N-1}\Bigg]\\
                \end{aligned}
            \end{equation}
    
            where
            \begin{equation}
                \left\{
                \begin{aligned}
                    A_0[k]&=4p_e(\mu_0^2[k]+\sigma_0^2[k])-4p_e\mu_0[k]+1\\
                    A_1[k]&=4p_e(\mu_1^2[k]+\sigma_1^2[k])-4p_e\mu_1[k]+1\\
                    B_0[k]&=1-2p_e\mu_0[k]\\
                    B_1[k]&=1-2p_e(1-\mu_1[k])\\
                \end{aligned}
                \right.
            \end{equation}
        \end{proof}

        \begin{proposition}[Conservation of Gaussian Distributions]\label{prop:gaussian_conservation}
            Let $D_0[k]$ and $D_1[k]$ be two Gaussian distributions  $D_0[k] = \mathcal{N}(\mu_0[k],\sigma_0^2[k])$, $D_1[k] = \mathcal{N}(\mu_1[k],\sigma_1^2[k])$ characterizing the value of the parameters $w_j[k]$ at step $k$ matched to true value of $w_j^\mathrm{true}$ of 0 and 1 respectively. Let $p_w$ be the proportion of oracle weights equal to one. Under a plain gradient descent, the two distributions remains Gaussian $D_0[k+1] = \mathcal{N}(\mu_0[k+1],\sigma_0^2[k+1])$, $D_1[k+1] = \mathcal{N}(\mu_1[k+1],\sigma_1^2[k+1])$.
            
        \end{proposition}

        \begin{proof}\label{proof:gaussian_conservation}
            When considering a standard gradient descent, the weight update is defined as:
            \begin{equation}
                \begin{aligned}
                    w_i[k+1] &= w_i[k] - \alpha\mathbb{E}_{z \sim Z, w_j \sim D}\Bigg\{\frac{\partial \hat{\ell}_{\oplus}(\mathbf{w})[k]}{\partial w_i[k]}\Bigg\}\\
                \end{aligned}
            \end{equation}

            Substituting the gradient expression from Eq.~\ref{eq:double_gaussian_gradient_xor_node}, we can see that the update rule is an affine transform of the form 
            \begin{equation}
                w_i[k+1]=m[k]w_i[k]+c(w_i^\mathrm{true})[k]
            \end{equation}

            with
            \begin{equation}
                \left\{
                \begin{aligned}
                    m_i[k]=&1-2\alpha p_e\bigg((1-p_w^{(\setminus i)})A_0[k] +p_w^{(\setminus i)}A_1[k]\bigg)^{N-1}\\
                    c_i[k]=&\alpha p_e\Bigg[\bigg((1-p_w^{(\setminus i)}) A_0[k] +p_w^{(\setminus i)}A_1[k]\bigg)^{N-1}
                    +(2w_i^\mathrm{true}-1)\times \bigg((1-p_w^{(\setminus i)}) B_0[k]+p_w^{(\setminus i)} B_1[k]\bigg)^{N-1}\Bigg]
                \end{aligned}
                \right.
            \end{equation}
            
            The only variables that vary with index $i$ are the biased proportion $p_w^{(\setminus i)}$ (entirely determined by $w_i^\mathrm{true}$) and the oracle target value $w_i^\mathrm{true}$ itself. Hence, there are clearly only two distinct affine transformations:
            \begin{itemize}
                \item The one applied to all weights $w_i$ whose corresponding oracle target value is $w_i^\mathrm{true}=0$. We denote its coefficients as $m_0[k]$ and $c_0[k]$.
                \item The one applied to all weights $w_i$ whose corresponding oracle target value is $w_i^\mathrm{true}=1$. We denote its coefficients as $m_1[k]$ and $c_1[k]$.
            \end{itemize}
            
            Since a Gaussian distribution modified through an affine transform remains Gaussian (i.e. if $\mathrm{x} \sim\mathcal{N}(\mu,\sigma^2)$ then $a\mathrm{x}+b\sim\mathcal{N}(a\mu+b,a^2\sigma^2)$), if
            \begin{equation}
                \left\{
                    \begin{aligned}
                        D_0[k] = \mathcal{N}\Big(\mu_0[k],\sigma_0^2[k]\Big)\\
                        D_1[k] = \mathcal{N}\Big(\mu_1[k],\sigma_1^2[k]\Big)\\
                    \end{aligned}
                \right.
            \end{equation}
            the updated distribution of the weight after one step of gradient descent are therefore:
            \begin{equation}
                \left\{
                    \begin{aligned}
                        D_0[k+1] = \mathcal{N}\Big(\mu_0[k+1],\sigma_0^2[k+1]\Big) = \mathcal{N}\Big(m_0[k]\mu_0[k]+c_0[k],m_0^2[k]\sigma_0^2[k]\Big) \\
                        D_1[k+1] = \mathcal{N}\Big(\mu_1[k+1],\sigma_1^2[k+1]\Big) = \mathcal{N}\Big(m_1[k]\mu_1[k]+c_1[k],m_1^2[k]\sigma_1^2[k]\Big) \\
                    \end{aligned}
                \right.
            \end{equation}

        \end{proof}

        \begin{proposition}[Conservation of Symmetry]\label{prop:symmetry_conservation}
            Let $D_0[k]$ and $D_1[k]$ be two Gaussian distributions, symmetric to each other with respect to $1/2$:
            \begin{equation}
                \left\{
                    \begin{aligned}
                        &\mu_0[k] = 1 - \mu_1[k]\\
                        &\sigma_0^2[k] = \sigma_1^2[k]\\
                    \end{aligned}
                \right.
            \end{equation}
            Under a plain gradient descent, the two distributions remain symmetric to each other with respect to $1/2$:
            \begin{equation}
                \left\{
                    \begin{aligned}
                        &\mu_0[k+1] = 1 - \mu_1[k+1]\\
                        &\sigma_0^2[k+1] = \sigma_1^2[k+1]\\
                    \end{aligned}
                \right.
            \end{equation}
        \end{proposition}

        \begin{proof}\label{proof:symmetry_conservation}
            We recall the coefficient of the affine transform defining the weights update (Proposition~\ref{prop:gaussian_conservation}):
            \begin{equation}
                \left\{
                \begin{aligned}
                    m_i[k]=&1-2\alpha p_e\bigg((1-p_w^{(\setminus i)})A_0[k] +p_w^{(\setminus i)}A_1[k]\bigg)^{N-1}\\
                    c_i[k]=&\alpha p_e\Bigg[\bigg((1-p_w^{(\setminus i)}) A_0[k] +p_w^{(\setminus i)}A_1[k]\bigg)^{N-1}
                    +(2w_i^\mathrm{true}-1)\times \bigg((1-p_w^{(\setminus i)}) B_0[k]+p_w^{(\setminus i)} B_1[k]\bigg)^{N-1}\Bigg]
                \end{aligned}
                \right.
            \end{equation}
            
            If the two distributions are symmetric we have:
            \begin{equation}
                \begin{aligned}
                    A_1[k] &= 4p_e(\mu_1^2[k]+\sigma_1^2[k])-4p_e\mu_1[k]+1\\
                    &= 4p_e((1-\mu_0[k])^2+\sigma_0^2[k])-4p_e(1-\mu_0[k])+1\\
                    &= 4p_e(\mu_0^2[k]+\sigma_0^2[k])-4p_e\mu_0[k]+1\\
                    &= A_0[k]\\
                \end{aligned}
                \quad\quad\textrm{and}\quad\quad
                \begin{aligned}
                    B_1[k] &= 1-2p_e(1-\mu_1[k])\\
                    &= 1-2p_e\mu_0[k]\\
                    &=B_0[k]\\
                    &\\
                \end{aligned}
            \end{equation}

            This allows us to simplify the $w_i[k+1]=m_i[k]w_i[k]+c_i[k]$ update rule as:
            \begin{equation}
                \left\{
                \begin{aligned}
                    &m_i[k]=1-2\alpha p_eA_0^{N-1}[k] \\
                    &c_i[k]=\alpha p_e\Bigg[A_0^{N-1}[k]+(2w_i^\mathrm{true}-1)\times B_0^{N-1}[k]\Bigg]\\
                \end{aligned}
                \right. \quad(\text{Since} \quad p_w^{(\setminus i)}+(1-p_w^{(\setminus i)})=1)
            \end{equation}

            We can see that $m_i[k]$ becomes independent of $i$ such that $m_0[k]=m_1[k]$. We keep the simplified notation $m[k]$. In contrast, $c_i[k]$ remains dependent on $i$ through the oracle weight value $w_i^\mathrm{true}$. This simplification further shows that any dependency on the oracle weights proportion $p_w$ disappears under symmetric distributions.

            Given that the symmetry rule is verified in $k$, i.e.: 
            \begin{equation}
                \left\{
                    \begin{aligned}
                        &\mu_0[k] = 1 - \mu_1[k]\\
                        &\sigma_0^2[k] = \sigma_1^2[k]\\
                    \end{aligned}
                \right.
            \end{equation}
            
            We want to verify that it still holds in $k+1$, i.e.: 
            \begin{equation}
                \left\{
                    \begin{aligned}
                        &\mu_0[k+1] = 1 - \mu_1[k+1]\\
                        &\sigma_0^2[k+1] = \sigma_1^2[k+1]\\
                    \end{aligned}
                \right.
            \end{equation}
            
            Under the previous affine transform we need to verify:
            \begin{equation}
                \left\{
                    \begin{aligned}
                        &\mu_0[k+1]=m[k]\mu_0[k]+c_0[k] = 1 - \Big(m[k]\mu_1[k]+c_1[k]\Big)=1-\mu_1[k+1]\\
                        &\sigma_0^2[k+1]=m^2[k]\sigma_0^2[k] = m^2[k]\sigma_1^2[k]=\sigma_1^2[k+1]\\
                    \end{aligned}
                \right.
            \end{equation}
        
            From the above, it is obvious that the symmetry rule is conserved for the variance, i.e. $\sigma_0^2[k+1] = \sigma_1^2[k+1]$.
        
            For the mean, we have:
            \begin{equation}
                \begin{aligned}
                    &m[k]\mu_0[k]+c_0[k] = 1 - m[k](1-\mu_0)[k]-c_1[k]\\
                    &\Longleftrightarrow\quad c_0[k]+c_1[k]=1-m[k]\\
                    &\Longleftrightarrow\quad \alpha p_e\Bigg[A_0^{N-1}[k]-B_0^{N-1}[k]\Bigg]+\alpha p_e\Bigg[A_0^{N-1}[k]+B_0^{N-1}[k]\Bigg]=2\alpha p_eA_0^{N-1}[k]\\
                    &\Longleftrightarrow\quad 2\alpha p_eA_0^{N-1}[k]=2\alpha p_eA_0^{N-1}[k]\\
                \end{aligned}
            \end{equation}

            Hence, the symmetry rule is also conserved for the mean. Therefore, if the two distributions are Gaussian and symmetric to each other with respect to $1/2$ in $k$, they remain symmetric Gaussians in $k+1$.
        \end{proof}

        \begin{remark}[Practical Implications on Intialization and System Evolution]\label{remark:gaussian_init}
            If at $k=0$, the initial weight distributions $D_0[0]$ and $D_1[0]$ are Gaussian and symmetric to each other with respect to $1/2$, then under the Gaussian conservation property from Proposition~\ref{prop:gaussian_conservation} and the symmetry property from Proposition~\ref{prop:symmetry_conservation}, they remain symmetric Gaussians throughout training. The gradient expression remains, at any step $k$, equal to:
            \begin{equation}\label{eq:symmetric_gradient}
                \begin{aligned}
                    \mathbb{E}_{\mathrm{z} \sim Z, \mathrm{w}_j \sim D}\Bigg\{\frac{\partial \hat{\ell}_{\oplus}(\mathbf{w})[k]}{\partial w_i[k]}\Bigg\} =  p_e \Bigg[&(2w_i[k]-1)\times A_0^{N-1}[k]-(2w_i^\mathrm{true}-1)\times B_0^{N-1}[k]\Bigg]\\
                \end{aligned}
            \end{equation}     
    
            where $A_0[k]=4p_e(\mu_0^2[k]+\sigma_0^2[k])-4p_e\mu_0[k]+1$ and $B_0[k]=1-2p_e\mu_0[k]$ depend on the distribution's mean and variance at step $k$. 

            Notably, this expression no longer depends on the proportion of oracle weights equal to one, $p_w$, implying that the convergence properties of this system become invariant to $p_w$.

            Although we cannot explicitly partition weights according to their oracle targets (which are presumed unknown to the model) and assign distinct distributions $D_0[0]$ and $D_1[0]$, we can achieve the same theoretical framework by initializing all weights from a single Gaussian distribution $D$ that satisfies the symmetry condition with respect to itself, namely a Gaussian distribution, centered in $1/2$. This creates an implicit partition where $D_0[0] = D_1[0] = D$, preserving both Gaussianity and symmetry properties. While the two distributions start identical, they do not need to remain equal during training as they evolve according to their respective oracle targets, maintaining only their symmetric relationship and Gaussian structure.
        \end{remark}

        \begin{remark}\label{remark:optimal_sparsity_gaussian_init} 
            Combining the results from Remark~\ref{remark:gaussian_init} and Proposition~\ref{prop:optimal_error_probability}, using sparsity $p_e = 1/N$ appears particularly favorable when initializing weights from a symmetric Gaussian distribution centered at $1/2$.
    
            Under such initialization, the expected gradient magnitude is maximized, providing the strongest learning signal. However, since Gaussian distributions have unbounded support, some weights may lie outside the interval $(0,1)$, where the directional argument of Proposition~\ref{prop:optimal_error_probability} does not strictly apply. Nevertheless, for sufficiently tight distributions (i.e., small $\sigma^2$), the vast majority of weights are typically sampled within $(0,1)$, ensuring that the favorable directional gradient properties hold for most of the weights\footnote{Throughout this paper, we use $\sigma^2[0] = 1/4$ (see Assumption~\ref{assumption:system_assumption}). For this initialization, the interval $[0,1]$ corresponds to the $1\sigma$ range around the mean, containing approximately 68\% of weights. This ensures that the directional gradient properties hold for more than two-thirds of the parameters, which will empirically prove sufficient for establishing convergence to the oracle weights.}
    
            This theoretical insight, that sparse, one-hot-like inputs provide optimal gradient signal when combined with symmetric Gaussian initialization, motivates our choice to study the overall convergence behavior of the system under sparsity $p_e = 1/N$ throughout this paper. This setting represents the theoretically optimal regime where gradient magnitude is maximized at initialization, while preserving the correct directional updates for the majority of weights.
        \end{remark}

        \begin{remark}[Behavior Close to Convergence]\label{remark:close_to_convergence_behavior}
            As training progresses and most parameters approach their true values, we have $\mu_0[k] = 1-\mu_1[k] \approx 0$ and $\sigma_0^2[k] = \sigma_1^2[k] \approx 0$. Under these conditions, the gradient for any remaining parameter $w_i$ simplifies to:

            \begin{equation}
            \mathbb{E}_{\mathrm{z} \sim Z, \mathrm{w}_j \sim D}\left\{\frac{\partial \hat{\ell}_{\oplus}(\mathbf{w})[k]}{\partial w_i[k]}\right\} \approx p_e \Big[(2w_i[k]-1) - (2w_i^{\mathrm{true}}-1)\Big]
            \end{equation}
            
            This gives us the linearized gradient:
            \begin{equation}
            \mathbb{E}_{\mathrm{z} \sim Z, \mathrm{w} \sim D}\left\{\nabla_{\mathbf{w}}\hat{\ell}_{\oplus}(\mathbf{w}) \big| \mathbf{w}^{\mathrm{true}}\right\} = 2p_e(\mathbfit{w} - \mathbfit{w}^{\mathrm{true}})
            \end{equation}
            
            The gradient magnitude is maximized when $p_e = 1$, suggesting that higher sparsity becomes beneficial as training progresses.
            
            This analysis provides an \textbf{intuitive explanation} rather than a rigorous proof for why it may be advantageous to increase sparsity $p_e$ during training:
            \begin{itemize}
                \item \textbf{At initialization}: $p_e = 1/N$ is optimal (Proposition~\ref{prop:optimal_error_probability})
                \item \textbf{Near convergence}: $p_e = 1$ maximizes expected gradient magnitude
                \item \textbf{Empirically}: best constant $p_e$ appears to be around $2/N$, as shown in our experimental results in Section~ \ref{subsubsection:impact_of_data_sparsity}.
            \end{itemize}
            
            However, choosing $p_e$ too high initially can prevent convergence entirely, suggesting the existence of an upper bound $p_e^{\lim}$ as observed in our experimental results (Section~ \ref{subsubsection:impact_of_data_sparsity}).
        \end{remark}


    \subsection{Convergence Analysis}\label{app:convergence_analysis}

        We have established that parameters $w_i$ follow two symmetric Gaussian distributions $D_0$ and $D_1$ that preserve their Gaussian nature (Property~\ref{prop:gaussian_conservation}) and symmetric behavior (Proposition~\ref{prop:symmetry_conservation}). To ensure these propositions remain valid throughout training, we initialize with $\mathcal{N}(1/2,\sigma^2[0])$ as specified in Remark~\ref{remark:gaussian_init}. Since distribution parameters evolve with iteration step $k$, we denote them as $D_0[k]$ and $D_1[k]$. In this section we would like to demonstrates that $\lim_{k\to\infty} D_0[k]$ converges to a degenerate distribution concentrated at zero:
        \begin{equation}
            \left\{
                \begin{aligned}
                    \lim_{k\to\infty} \mu_0[k] = 0\\
                    \lim_{k\to\infty} \sigma_0^2[k] = 0\\
                \end{aligned}
            \right.
        \end{equation}
        
        By symmetry, $D_1[k]$ converges to a degenerate distribution at one:
        \begin{equation}
            \left\{
                \begin{aligned}
                    \lim_{k\to\infty} \mu_1[k] = 1\\
                    \lim_{k\to\infty} \sigma_1^2[k] = 0\\
                \end{aligned}
            \right.
        \end{equation}
        
        From Proposition~\ref{prop:gaussian_conservation}, the system dynamics reduce to affine transformations with state-dependent coefficients $m[k]$ and $c[k]$, yielding the equations $\mu[k+1]=m[k]\mu[k]+c[k]$ and $\sigma^2[k+1]=m^2[k]\sigma^2[k]$, where:
        
        \begin{equation}
            \left\{
            \begin{aligned}
                &m[k]=1-2\alpha p_e\left[(4p_e(\mu^2[k]+\sigma^2[k]-\mu[k])+1)^{N-1}\right]\\
                &c[k]=\alpha p_e\left[(4p_e(\mu^2[k]+\sigma^2[k]-\mu[k])+1)^{N-1}-(1-2p_e\mu[k])^{N-1}\right]\\
            \end{aligned}
            \right.
        \end{equation}

        To ensure tractable analysis, our proof strategy relies on logarithmic series expansions of terms $(1+x)^{N-1}$, decomposing the dynamics into leading-order terms and providing rigorous upper bounds on truncation errors (Appendix~\ref{app:general_approx_and_bounds}, Proposition~\ref{prop:exp_approx_when_pe_1_N}). Considering sparsity $p_e=1/N$ (as per Remark~\ref{remark:optimal_sparsity_gaussian_init}), we bound our system's dynamics under the domain of assumptions $\mu_{\text{min}} \leq 0 \leq \mu[k] \leq 1/2$ and $0 < \sigma^2[k] \leq \sigma^2[0]$.  By fixing reasonable values for $\mu_{\text{min}}$ and $\sigma^2[0]$, we demonstrate in the next steps that the assumed domains are valid. Our convergence proof requires verifying:
        
        \textbf{Variance convergence:} follows directly from $|m[k]|<1$ for all $k$ (ensuring $0 < \sigma^2[k] \leq \sigma^2[0]$ (Proposition~\ref{prop:sigma_convergence}).
        
        \textbf{Mean convergence:} requires establishing four key properties:
        
        \begin{itemize}
            \item \textbf{Property 1 - Monotonic Convergence:} We must establish that the evolution $\mu[k+1]=m[k]\mu[k]+c[k]$ exhibits monotonic behavior toward instantaneous fixed points. This requires verifying the condition (Proposition~\ref{prop:monotonic_constraint}):
            
            $$0 < m[k] < 1$$
            
            Under this condition, $\mu[k]$ will approach its instantaneous fixed point $\overline{\mu}[k]$ at each step without oscillation (as per Theorem~\ref{thm:complete_variable_dynamics}).
        
            \item \textbf{Property 2 - Domain Invariance:} We need to prove that all instantaneous fixed points remain within our assumption domain (Proposition~\ref{prop:envelope_invariance}):
            
            $$\overline{\mu}[k] \subset \left[\mu_{\text{min}},1/2\right] \quad \forall k$$

            Or, under bounded update steps, that this conditions holds on boundary regions of our domain (Proposition~\ref{prop:relaxed_stability}).
            
            These properties ensures forward invariance of our domain assumptions throughout the evolution.
        
            \item \textbf{Property 3 - Bounded Update Steps:} We must establish uniform bounds on gradient descent updates to bound the dynamics of our system. We define the distance $\delta[k] = \mu[k+1] - \mu[k]$, and propose an upper bound to the absolute value of this distance, noted $\delta_{\text{max}}$ (Proposition~\ref{prop:bounded_update_steps}).
        
            \item \textbf{Property 4 - Bounded Convergence Interval:} We need to identify intersection point(s) where $\overline{\mu}[k](\mu[k])$ meets the identity line, satisfying an equilibrium point $\overline{\mu}[k]=\mu[k]=\mu[k+1]$. We bound this intersection point and prove that $\mu[k]$ converges to this interval that becomes arbitrarily tight as $N \to \infty$ under bounded update steps(Proposition~\ref{prop:bounded_convergence_interval}).
        \end{itemize}

        \begin{assumption}[System Assumptions]\label{assumption:system_assumption}
            We assume that $\mu[k] \in [-1/4, 1/2]$, $\sigma^2[k] \in (0, 1/4]$ for all $k \geq 0$, with initial conditions $\mu[0]=1/2$ and $\sigma^2[0]=1/4$.
        \end{assumption}

        These constraints serve as working hypotheses for our analysis, but we will subsequently prove that they remain valid throughout the system's evolution via forward invariance properties (see Propositions~\ref{prop:envelope_invariance} and~\ref{prop:relaxed_stability})

    \subsubsection{Series Expansion and Error Bounds}\label{app:approx_and_error_bounds}

        This section establishes usefull formulations to the complex $(N-1)$ power terms that arise in the gradient descent update equations. Under the sparsity regime $p_e = 1/N$, the update coefficients $m[k]$ and $c[k]$ involve expressions of the form $(1 + x/N)^{N-1}$ where $x/N \to 0$ as $N$ grows large.  We leverage this asymptotic structure by employing series expansion to decompose these terms into a dominant exponential component and bounded error terms, expressed as $e^{x} \cdot e^{\eta^{(x)}}$, see section \ref{app:general_approx_and_bounds}. Here, $e^{x}$ represents the principal term that captures the leading-order behavior, while $e^{\eta^{(x)}}$ denotes the truncation error term, which is rigorously bounded and vanishes as $N \to \infty$ (Proposition~\ref{prop:exp_approx_when_pe_1_N}). This approach enables us to rewrite the update equations in exponential form (Remark~\ref{remark:exp_approx_update_rules}), making the system dynamics analytically tractable while maintaining rigorous error control. We also establish improved bounds under the zero-variance limit (Proposition~\ref{prop:exp_approx_when_pe_1_N_zero_variance}), which becomes relevant as the system approaches convergence and variance naturally decreases.

        \begin{proposition}[Exponential Approximation when $p_e=1/N$]\label{prop:exp_approx_when_pe_1_N}
            Under Assumptions~\ref{assumption:system_assumption}, consider the update equations with $p_e = 1/N$. For all $N > 2$, the exponential terms can be written as:
            
            \begin{equation}
                \left\{
                \begin{aligned}
                    &\Big(4p_e(\mu^2[k]+\sigma^2[k])-4p_e\mu[k]+1\Big)^{N-1} = e^{\xi[k]} \cdot e^{\eta^{(\xi)}[k]} \\
                    &\Big(1-2p_e\mu[k]\Big)^{N-1} = e^{\zeta[k]} \cdot e^{\eta^{(\zeta)}[k]}
                \end{aligned}
                \right.
            \end{equation}
            where $\xi[k] = 4(\mu^2[k]+\sigma^2[k]-\mu[k])$, $\zeta[k] = -2\mu[k]$, and the truncation error terms are bounded by:
            \begin{equation}
                \left\{
                \begin{aligned}
                        &|\eta^{(\xi)}[k]| < \frac{153}{32N-72} :=\eta^{(\xi)}_{\max}(N)\\
                    &|\eta^{(\zeta)}[k]| \leq \frac{3}{2N} := \eta^{(\zeta)}_{\max}(N)
                \end{aligned}
                \right.\quad\forall N>2
            \end{equation}
            
        \end{proposition}

        \begin{proof}\label{proof:exp_approx_when_pe_1_N}

            In Proposition~\ref{prop:exp_approx} we show that for any integer $N \geq 2$ and $|x| < 1$, the following exact series expansion holds: 
            \begin{equation}
                (1+x)^{N-1} = e^{Nx} \cdot e^{\eta(x)}
            \end{equation}
            where $e^{Nx}$ is the leading order term, and $e^{\eta(x)}$ is the corresponding truncation error term, bounded as \begin{equation}
                |\eta(x)| \leq \frac{(N-1)x^2}{2(1-|x|)}+|x|
            \end{equation}

            Applying Proposition~\ref{prop:exp_approx} with $x = \frac{\xi[k]}{N}$ and $x = \frac{\zeta[k]}{N}$, we need to verify the conditions ensuring $|x|<1$ across our assumption domain \ref{assumption:system_assumption}. From the definition domains, we have:
            \begin{equation}
                \left\{
                \begin{aligned}
                    &\frac{\xi[k]}{N} = \frac{4}{N}(\mu^2[k]+\sigma^2[k]-\mu[k]) \in \left[-\frac{1}{N},\frac{9}{4N}\right] \subset (-1,1) \\
                    &\frac{\zeta[k]}{N} = -\frac{2}{N}\mu[k] \in \left[-\frac{1}{N},\frac{1}{2N}\right] \subset (-1,1)
                \end{aligned}
                \right.\quad\forall N>2
            \end{equation}
            
            Hence, having $N>2$ ensures the validity of the expansion from Proposition~\ref{prop:exp_approx}, thus allowing us to use it in our upcoming derivations.

            \textbf{Worst-case analysis:} With $x = \frac{\xi[k]}{N}$ and $x = \frac{\zeta[k]}{N}$, we bound the corresponding truncation error terms $\eta^{(\xi)}[k]$ and $\eta^{(\zeta)}[k]$ using the fact that $|\eta(x)| \leq \frac{(N-1)x^2}{2(1-|x|)} + |x|$ (Proposition~\ref{prop:exp_approx}). Since the error bound is monotonically increasing with $|x|$, the worst case occurs at the maximum values of $\left|\frac{\xi[k]}{N}\right|$ and $\left|\frac{\zeta[k]}{N}\right|$ within our assumption domain.
        
            We now use these extrema over our assumption domain to quantify the truncation error bounds:
            
            \textbf{First error term ($\eta^{(\xi)}[k]$):} With $\mu[k] = -1/4, \sigma^2[k] = 1/4$:
            \begin{equation}
                \begin{aligned}
                    |\eta^{(\xi)}[k]| &\leq \frac{(N-1) \times (\frac{9}{4N})^2}{2(1-\frac{9}{4N})} + \frac{9}{4N} \\
                    &= \frac{(N-1) \times 81}{32N^2(1-\frac{9}{4N})} + \frac{9}{4N} \\
                    &= \frac{81(N-1) \cdot 4N}{32N^2(4N-9)} + \frac{9}{4N} \\
                    &= \frac{324(N-1)}{32N(4N-9)} + \frac{9}{4N} \\
                    &< \frac{324N}{32N(4N-9)} + \frac{9}{4N-9} \\
                    &= \frac{612}{32(4N-9)}\\
                    &= \frac{153}{(32N-72)}\\
                    &:= \eta^{(\xi)}_{\max}(N)\\
                \end{aligned}
            \end{equation}
 
            \textbf{Second error term ($\eta^{(\zeta)}[k]$):} With $\mu[k] = 1/2$:
            \begin{equation}
                \begin{aligned}
                    |\eta^{(\zeta)}[k]| &\leq \frac{(N-1) \times (-\frac{1}{N})^2}{2(1-|-\frac{1}{N}|)} + |-\frac{1}{N}| \\
                    &= \frac{N(N-1)}{2N^2(N-1)} + \frac{1}{N}\\
                    &= \frac{1}{2N} + \frac{1}{N}\\
                    &= \frac{3}{2N}\\
                    &:= \eta^{(\zeta)}_{\max}(N)\\
                \end{aligned}
            \end{equation}
        \end{proof}
 
        \begin{remark}[Exponential Approximation of the Update Rules when $p_e=1/N$]\label{remark:exp_approx_update_rules}
            As per Proposition~\ref{prop:exp_approx_when_pe_1_N}, the update equations can be written as:
            \begin{equation}
                \left\{
                \begin{aligned}
                    &m[k] = 1-\frac{2\alpha}{N}e^{\xi[k]} \cdot e^{\eta^{(\xi)}[k]} \\
                    &c[k] = \frac{\alpha}{N}\left[e^{\xi[k]} \cdot e^{\eta^{(\xi)}[k]} - e^{\zeta[k]} \cdot e^{\eta^{(\zeta)}[k]}\right]
                \end{aligned}
                \right.
            \end{equation}
            where $\xi[k] = 4(\mu^2[k]+\sigma^2[k]-\mu[k])$, $\zeta[k] = -2\mu[k]$ and $\eta^{(\xi)}[k],\eta^{(\zeta)}[k]$ are the corresponding truncation error terms.
    
            The gradient expression becomes:
            \begin{equation}
                \mathbb{E}_{\mathrm{z} \sim Z, \mathrm{w}_j \sim D}\Bigg\{\frac{\partial \hat{\ell}_{\oplus}(\mathbf{w})[k]}{\partial w_i[k]}\Bigg\} = \frac{1}{N}\left[2e^{\xi[k]}e^{\eta^{(\xi)}[k]} \cdot w[k] - e^{\xi[k]}e^{\eta^{(\xi)}[k]} + e^{\zeta[k]}e^{\eta^{(\zeta)}[k]}\right] \quad \text{($\forall i$ such that $w_i^\mathrm{true}=0$)}
            \end{equation}
            
            Note that both error terms vanish as $N \to \infty$.
        \end{remark}
        
        
        \begin{proposition}[Zero Variance Limit]\label{prop:exp_approx_when_pe_1_N_zero_variance}
            When $\sigma^2[k] \to 0$, the first truncation error term bound improves to:
            $$\underset{\sigma\rightarrow0^+}{\eta^{(\xi)}_{\max}(N)} = \frac{65}{32N-40}$$
            By continuity, this bound approaches the zero-variance value as $\sigma^2[k] \to 0^+$.
        \end{proposition}

        \begin{proof}\label{proof:exp_approx_when_pe_1_N_zero_variance}
            When $\sigma^2[k] \to 0$, we have $\xi[k] \to 4(\mu^2[k] - \mu[k])$. The largest value is obtained when $\mu[k] = -1/4$, giving $\xi[k] = 4 \times \frac{5}{16} = \frac{5}{4}$. Thus:
            \begin{equation}
                \begin{aligned}
                    |\eta^{(\xi)}[k]| &\leq \frac{(N-1) \times (\frac{5}{4N})^2}{2(1-\frac{5}{4N})} + \frac{5}{4N} \\
                    &= \frac{25(N-1)}{32N^2(1-\frac{5}{4N})} + \frac{5}{4N} \\
                    &= \frac{100N(N-1)}{32N^2(4N-5)} + \frac{5}{4N} \\
                    &= \frac{100(N-1)}{32N(4N-5)} + \frac{5}{4N} \\
                    &< \frac{100N}{32N(4N-5)} + \frac{5}{4N-5} \\
                    &=\frac{100}{32(4N-5)} + \frac{160}{32(4N-5)} \\
                    &=\frac{260}{32(4N-5)}\\
                    &=\frac{65}{32N-40}\\
                    &:=\underset{\sigma\rightarrow0^+}{\eta^{(\xi)}_{\max}(N)}\\
                \end{aligned}
            \end{equation}
        \end{proof}

        \subsubsection{Convergence Analysis - Variance}

            In this section, we establish sufficient conditions on the learning rate $\alpha$ to ensure variance convergence to zero. We derive explicit bounds by analyzing the multiplicative factor $m[k]$ in the variance update rule $\sigma^2[k+1] = m^2[k]\sigma^2[k]$, requiring $|m[k]| < 1$ for convergence (Proposition~\ref{prop:sigma_convergence}). Additionally, we provide a more restrictive condition that guarantees $0 < m[k] < 1$ which will prove useful in demonstrating the convergence of the mean (Proposition~\ref{prop:monotonic_constraint}, linked to the application of theorem \ref{thm:complete_variable_dynamics}). In Remarks~\ref{remark:constraint_relaxation} and \ref{remark:maintained_sigma_convergence_under_monotonic_constraint} we show how these conditions can be relaxed under limiting regimes.
            
            \begin{proposition}[Variance Convergence]\label{prop:sigma_convergence}
                Under Assumptions~\ref{assumption:system_assumption}, there exist $\alpha_0 > 0$ such that for all $0<\alpha < \alpha_0$, and $p_e = 1/N$, the variance $\sigma^2[k]$ converges to 0 with:
                $$\alpha_0=Ne^{-\frac{72N-9}{32N-72}}$$
            \end{proposition}
            
            \begin{proof}\label{proof:sigma_convergence}
                Since $\sigma^2[k+1] = m^2[k]\sigma^2[k]$, convergence to zero requires $|m[k]| < 1$ for all $k$.
                
                From Remark~\ref{remark:exp_approx_update_rules}, we have:
                $$m[k] = 1 - \frac{2\alpha}{N}e^{\xi[k]} \cdot e^{\eta^{(\xi)}[k]}$$
                
                For convergence, we need $|m[k]| < 1$, which is equivalent to:
                $$0 < \frac{2\alpha}{N}e^{\xi[k]} \cdot e^{\eta^{(\xi)}[k]} < 2$$
                
                The lower bound $\frac{2\alpha}{N}e^{\xi[k]}\cdot e^{\eta^{(\xi)}[k]}>0$ is automatically satisfied since $\alpha, N > 0$ and the exponential terms are positive. For the upper bound, we need to verify:
                \begin{equation}
                    \begin{aligned}
                        &\frac{2\alpha}{N}e^{\xi[k]}\cdot e^{\eta^{(\xi)}[k]}<2\\
                        &\Leftrightarrow \alpha < Ne^{-\xi[k]} \cdot e^{-\eta^{(\xi)}[k]}
                    \end{aligned}
                \end{equation}

                \textbf{Worst-case analysis:} 
                Since the right-hand side $Ne^{-\xi[k]} \cdot e^{-\eta^{(\xi)}[k]}$ is decreasing in both $\xi[k]$ and $|\eta^{(\xi)}[k]|$, the most restrictive condition (smallest upper bound on $\alpha$) occurs when these terms are maximized. The maximum value of $\xi[k]$ within Assumptions~\ref{assumption:system_assumption} occurs at $\mu[k] = -1/4, \sigma^2[k] = 1/4$ such that:
                \begin{equation}
                    \max_{\mu[k] \in [-1/4,1/2], \sigma^2[k]\in(0;1/4]}\xi[k] = \frac{9}{4}
                \end{equation}
                
                Since the approximation error term satisfies $|\eta^{(\xi)}[k]| < \frac{153}{32N-72}$ (as per Proposition~\ref{prop:exp_approx_when_pe_1_N}), a sufficient condition for convergence is:
                \begin{equation}
                    \alpha < Ne^{-\frac{9}{4}} \cdot e^{-\frac{153}{32N-72}} = Ne^{-\frac{72N-9}{32N-72}}
                \end{equation}
    
            \end{proof}

            \begin{remark}[Constraint Relaxation]\label{remark:constraint_relaxation}
                This sufficient condition can be relaxed under limiting regimes:

                \begin{itemize}
                    \item \textbf{Large $N$ regime:} As $N \to \infty$, the approximation error vanishes and the condition approaches $\alpha < Ne^{-9/4}$.
                    \item \textbf{Near convergence:} As $\sigma^2[k] \to 0$ and $\mu[k] \to 0$, we have $\xi[k] \to 0$, so the constraint becomes $\alpha < Ne^{-\frac{153}{32N-72}}$, which approaches $\alpha < N$ as $N \to \infty$.
                \end{itemize}
            
                This shows that the learning rate constraint for the variance to converge becomes increasingly permissive as either the system size grows or the global system approaches convergence.
                
            \end{remark}

            \begin{proposition}[Monotonic Convergence Constraint]\label{prop:monotonic_constraint}
                Under Assumptions~\ref{assumption:system_assumption}, there exists $\alpha_1 > 0$ such that for all $0 < \alpha < \alpha_1$ and $p_e = 1/N$, we have $0 < m[k] < 1$ for all $k$, where:
                $$\alpha_1 = \frac{N}{2}e^{-\frac{72N-9}{32N-72}}=\frac{\alpha_0}{2}<\alpha_0$$
            \end{proposition} 
            
            \begin{proof} \label{proof:monotonic_constraint}
                From Remark~\ref{remark:exp_approx_update_rules}, we have:
                $$m[k] = 1 - \frac{2\alpha}{N}e^{\xi[k]} \cdot e^{\eta^{(\xi)}[k]}$$
                
                For $0 < m[k] < 1$, we need:
                $$0 < \frac{2\alpha}{N}e^{\xi[k]} \cdot e^{\eta^{(\xi)}[k]} < 1$$
                
                The lower bound is automatically satisfied since all terms are positive. For the upper bound:
                $$\frac{2\alpha}{N}e^{\xi[k]} \cdot e^{\eta^{(\xi)}[k]} < 1 \Leftrightarrow \alpha < \frac{N}{2}e^{-\xi[k]} \cdot e^{-\eta^{(\xi)}[k]}$$
                
                Using the same worst-case analysis as in Proposition~\ref{prop:sigma_convergence}, the sufficient condition is:
                $$\alpha < \frac{N}{2}e^{-\frac{72N-9}{32N-72}} = \frac{\alpha_0}{2}$$
    
            \end{proof}

            \begin{remark}\label{remark:maintained_sigma_convergence_under_monotonic_constraint}
                Since $0 < m[k] < 1$ implies $|m[k]| < 1$, the variance convergence from Proposition~\ref{prop:sigma_convergence} still holds under this more restrictive constraint. 
                
                The constraint relaxation properties from Remark~\ref{remark:constraint_relaxation} remain valid with factor two division $\alpha_1 = \frac{\alpha_0}{2}$.
            \end{remark}
                
    \subsubsection{Convergence Analysis - Fixed Point Envelope}

        In this section, we introduce the notion of instantaneous fixed point function $\overline{\mu}[k]$, which represents the value toward which the mean is evolving at step $k$ under the affine transform $\mu[k+1]=m[k]\mu[k]+c[k]$ if all parameters remained unchanged. Proposition~\ref{prop:envelope_characterization} establishes envelope bounds for $\overline{\mu}[k]$ using our exponential framework and the corresponding error bounds from Proposition~\ref{prop:exp_approx_when_pe_1_N}. To ensure monotonic evolution of the mean without overshoot, it requires $0 < m[k] < 1$, guaranteed by $\alpha<\alpha_1$ in Proposition~\ref{prop:monotonic_constraint} (which also ensures variance convergence in Proposition~\ref{prop:sigma_convergence}). With this assumption,  Proposition~\ref{prop:envelope_invariance} proves that these fixed point bounds are forward invariant and contained within the assumed parameter domain $[-1/4, 1/2]$ for $N \geq 43$, thereby ensuring forward invariance of both mean and variance within the complete assumption domain from Assumption~\ref{assumption:system_assumption}. This assumption conservation guarantees the persistent validity of all our results (assumptions, bounds, convergence properties), as formalized in Corollary~\ref{cor:persistent_validity}.
        
        \begin{proposition}[Fixed Point Envelope Characterization]\label{prop:envelope_characterization}
            Under Assumptions~\ref{assumption:system_assumption} and $p_e = 1/N$, the instantaneous fixed point function (see Definition~\ref{def:one_step_affine_recurrence}) is:
            
            \begin{equation}
                \overline{\mu}[k] = \frac{1}{2}\left[1 - e^{\zeta[k]-\xi[k]} \cdot e^{\eta^{(\zeta)}[k]-\eta^{(\xi)}[k]}\right]            
            \end{equation}
            where $\zeta[k] - \xi[k] = -4\mu^2[k] - 4\sigma^2[k] + 2\mu[k]$, admits the envelope bounds:
            \begin{equation}
                \overline{\mu}_{\min}[k] < \overline{\mu}[k] < \overline{\mu}_{\max}[k]                
            \end{equation}
            where:
            \begin{equation}
            \left\{
            \begin{aligned}
            \overline{\mu}_{\min}[k] &= \frac{1}{2}\left[1 - e^{\zeta[k]-\xi[k]} \cdot e^{+\varepsilon(N)}\right]\\
            \overline{\mu}_{\max}[k] &= \frac{1}{2}\left[1 - e^{\zeta[k]-\xi[k]} \cdot e^{-\varepsilon(N)}\right]
            \end{aligned}
            \right.
            \end{equation}
            with $\varepsilon(N) = \frac{402N-216}{2N(32N-72)}$.  
            
            Note that this error term vanishes as $N\to0$.
            
            The global envelope extrema are achieved at:
            \begin{equation}
                \left\{
                \begin{aligned}
                    \overline{\mu}_{\max} &= \overline{\mu}_{\max}[k]\Big|_{\substack{\mu[k] = -1/4 \\ \sigma^2[k] = 1/4}} \\
                    \overline{\mu}_{\min} &= \overline{\mu}_{\min}[k]\Big|_{\substack{\mu[k] = +1/4 \\ \sigma^2[k] \to 0}}
                \end{aligned}
                \right.
            \end{equation}
        \end{proposition}
        
        \begin{proof}\label{proof:envelope_characterization}
            Since both $m[k]$ and $c[k]$ are functions of the current state $(\mu[k], \sigma^2[k])$, the instantaneous fixed point $\overline{\mu}[k] = \frac{c[k]}{1-m[k]}$ varies across the parameter domain. Based on the definitions from Theorem~\ref{thm:complete_variable_dynamics}, we characterize the \textbf{global fixed point envelope extrema} $[\overline{\mu}_{\min}, \overline{\mu}_{\max}]$ that bounds all possible instantaneous fixed point values across the assumption domain. Using the expression of $m[k]$ and $c[k]$ from Remark~\ref{remark:exp_approx_update_rules}, the instantaneous fixed point is defined as:
            \begin{equation}
                \overline{\mu}[k] = \frac{c[k]}{1-m[k]} = \frac{e^{\xi[k]} \cdot e^{\eta^{(\xi)}[k]} - e^{\zeta[k]} \cdot e^{\eta^{(\zeta)}[k]}}{2e^{\xi[k]} \cdot e^{\eta^{(\xi)}[k]}}
            \end{equation}
            
            which simplifies to:
            \begin{equation}
                \overline{\mu}[k] = \frac{1}{2}\left[1 - e^{\zeta[k]-\xi[k]} \cdot e^{\eta^{(\zeta)}[k]-\eta^{(\xi)}[k]}\right]
            \end{equation}
            
            Since the truncation error term from the fixed point function has the form $(\eta^{(\zeta)}[k]-\eta^{(\xi)}[k])$, we define the envelope functions using the worst case scenario based on our absolute error bounds from Proposition~\ref{prop:exp_approx_when_pe_1_N}:
            \begin{equation}
                \left\{
                \begin{aligned}
                    &|\eta^{(\xi)}[k]| < \frac{153}{32N-72} :=\eta^{(\xi)}_{\max}(N)\\
                    &|\eta^{(\zeta)}[k]| \leq \frac{3}{2N} := \eta^{(\zeta)}_{\max}(N)
                \end{aligned}
                \right.\quad\forall N>2
            \end{equation}
            
            Hence, an upper bound of the largest possible combined error is:
            \begin{equation}
                \eta^{(\zeta)}_{\max}(N)+\eta^{(\xi)}_{\max}(N):=\varepsilon(N)
            \end{equation} 
            while the smallest one is:
            \begin{equation}
                -\eta^{(\zeta)}_{\max}(N)-\eta^{(\xi)}_{\max}(N)=-\varepsilon(N)
            \end{equation}
            Developping the combined bound leads to:
            \begin{equation}
                \varepsilon(N) = \eta^{(\zeta)}_{\max}(N) + \eta^{(\xi)}_{\max}(N) = \frac{3}{2N} + \frac{153}{32N-72} = \frac{402N-216}{2N(32N-72)}
            \end{equation}
            
            Hence, for any valid state satisfying Assumptions~\ref{assumption:system_assumption}, the fixed point satisfies:
            \begin{equation}
                \overline{\mu}_{\min}[k] < \overline{\mu}[k] < \overline{\mu}_{\max}[k]
            \end{equation}
            
            where:
            \begin{equation}
                \left\{
                \begin{aligned}
                    \overline{\mu}_{\min}[k] &= \frac{1}{2}\left[1 - e^{\zeta[k]-\xi[k]} \cdot e^{+\varepsilon(N)}\right]\\
                    \overline{\mu}_{\max}[k] &= \frac{1}{2}\left[1 - e^{\zeta[k]-\xi[k]} \cdot e^{-\varepsilon(N)}\right]
                \end{aligned}
                \right.
            \end{equation}
            where $\zeta[k] - \xi[k] = -2\mu[k] - 4(\mu^2[k] + \sigma^2[k] - \mu[k]) = -4\mu^2[k] - 4\sigma^2[k] + 2\mu[k]$         
            
            These bounds establish the envelope of possible fixed points across the parameter space from Assumption~\ref{assumption:system_assumption} accounting for the truncated series expansion error. It is important to distinguish between:
            \begin{itemize}
                \item $\overline{\mu}_{\max}[k]$: upper envelope function that bounds $\overline{\mu}[k]$ from above (function of $\mu[k], \sigma^2[k], N$)
                \item $\overline{\mu}_{\max}$: global maximum of $\overline{\mu}_{\max}[k]$ over the parameter domain
                \item $\overline{\mu}_{\min}[k]$: lower envelope function that bounds $\overline{\mu}[k]$ from below (function of $\mu[k], \sigma^2[k], N$)
                \item $\overline{\mu}_{\min}$: global minimum of $\overline{\mu}_{\min}[k]$ over the parameter domain
            \end{itemize}

            \textbf{Upper bound:} Finding  $\overline{\mu}_{\max}$ requires to maximize $\overline{\mu}_{\max}[k] = \frac{1}{2}\left[1 - e^{\zeta[k]-\xi[k]} \cdot e^{-\varepsilon(N)}\right]$ over the domain of interest, which is equivalent to minimize the always positive subtracted term:
            \begin{itemize}
                \item $\sigma^2[k] = 1/4$ (minimize $e^{-4\sigma^2[k]}$)
                \item $\mu[k] = -1/4$ (minimize $e^{-4\mu^2[k] + 2\mu[k]}$)
                \item Small $N$ (minimize $e^{-\varepsilon(N)}$)
            \end{itemize}
            
            Therefore: $\overline{\mu}_{\max} = \overline{\mu}_{\max}[k]_{\Bigg|\substack{\mu[k] = -1/4 \\ \sigma^2[k] = 1/4 \\ N \text{ small}}}$
            
            \textbf{Lower bound:} Finding  $\overline{\mu}_{\min}$ requires to minimize $\overline{\mu}_{\min}[k] = \frac{1}{2}\left[1 - e^{\zeta[k]-\xi[k]} \cdot e^{+\varepsilon(N)}\right]$ over the domain of interest, which is equivalent to maximize the always positive subtracted term:
            \begin{itemize}
                \item $\sigma^2[k] \to 0$ (maximize $e^{-4\sigma^2[k]} \to 1$)
                \item $\mu[k] = +1/4$ (maximize $e^{-4\mu^2[k] + 2\mu[k]}$)
                \item Small $N$ (maximize $e^{+\varepsilon(N)}$)
            \end{itemize}
            
            Therefore: $\overline{\mu}_{\min} = \overline{\mu}_{\min}[k]_{\Bigg|\substack{\mu[k] = +1/4 \\ \sigma^2[k] \to 0 \\ N \text{ small}}}$

        \end{proof}

        \begin{proposition}[Envelope Invariance and Parameter Domain Validity]\label{prop:envelope_invariance}
            For $N \geq 43$ and $\alpha < \frac{N}{2}e^{-\frac{72N-9}{32N-72}}$, the envelope $[\overline{\mu}_{\min}, \overline{\mu}_{\max}]$ satisfies:
            $$[\overline{\mu}_{\min}, \overline{\mu}_{\max}] \subset \left[-\frac{1}{4}, \frac{1}{2}\right]$$
            and is forward invariant. Consequently, any trajectory starting with $\mu[0] \in [-1/4, 1/2]$ and $\sigma^2[0] \in (0, 1/4]$ remains within the valid parameter domain for all iterations $k \geq 0$.
        \end{proposition}

        \begin{proof}\label{proof:envelope_invariance}

            For the proof, we need to apply the Theorem~\ref{thm:complete_variable_dynamics}, dependent on the   Proposition~\ref{prop:monotonic_constraint}: when $\alpha < \frac{N}{2}e^{-\frac{72N-9}{32N-72}}$, we have $0 < m[k] < 1$ for all $k$, ensuring the monotonic contraction condition required by Theorem~\ref{thm:complete_variable_dynamics}.
                        
            From Theorem~\ref{thm:complete_variable_dynamics}, the envelope $[\overline{\mu}_{\min}, \overline{\mu}_{\max}]$ is forward invariant: if $\mu[k]$ lies within the envelope, then $\mu[k+1]$ will also lie within it.
            
            Since the range of instantaneous fixed points across the parameter domain from Assumptions~\ref{assumption:system_assumption} satisfies:
            \begin{equation}
                \overline{\mu}_{\min} \leq \overline{\mu}_{\min}[k] < \overline{\mu}[k] < \overline{\mu}_{\max}[k] \leq \overline{\mu}_{\max},
            \end{equation}
            we require extrema of this envelope to be contained within the valid parameter domain:
            \begin{equation}
            [\overline{\mu}_{\min}, \overline{\mu}_{\max}] \subseteq \left[-\frac{1}{4}, \frac{1}{2}\right]
            \end{equation}
            
            This is equivalent to the following conditions:
            \begin{align}
            \text{Condition (A):} \quad &\overline{\mu}_{\min} \geq -\frac{1}{4} \\
            \text{Condition (B):} \quad &\overline{\mu}_{\max} \leq \frac{1}{2}
            \end{align}
            
            Using the envelope characterization from Proposition~\ref{prop:envelope_characterization}, we seek threshold values $N_A$ and $N_B$ such that the envelope remains within the valid parameter domain. Since the truncation error terms decrease as $N$ increases, any $N$ larger than these thresholds will also satisfy the conditions:
            \begin{equation}
                \left\{
                \begin{aligned}
                    &N \geq N_A \quad \Rightarrow \quad \overline{\mu}_{\min} = \overline{\mu}_{\min}[k]\Big|_{\substack{\mu[k] = +1/4 \\ \sigma^2[k] \to 0}} \geq -\frac{1}{4} \\
                    &N \geq N_B \quad \Rightarrow \quad \overline{\mu}_{\max} = \overline{\mu}_{\max}[k]\Big|_{\substack{\mu[k] = -1/4 \\ \sigma^2[k] = 1/4}} \leq \frac{1}{2}
                \end{aligned}
                \right.
            \end{equation}
            
            Setting $N_{\max} = \max(N_A, N_B)$ ensures both conditions are satisfied $\forall N \geq N_{\max}$.
            
            \textbf{Condition (B):} It is easy to verify condition (B) ($\overline{\mu}_{\max} \leq 1/2$), since $\overline{\mu}_{\max}[k]$ has the form:
    
            \begin{equation}
                \overline{\mu}_{\max}[k] = \frac{1}{2}\left[1-e^{-4\mu^2[k]+2\mu[k]} \times e^{-4\sigma^2[k]} \times e^{-\varepsilon(N)}\right]
            \end{equation}
            
            and all exponential terms are positive by definition, we have $\overline{\mu}_{\max}[k] < \frac{1}{2}$ for all $(\mu[k], \sigma^2[k], N)$. Therefore, $\overline{\mu}_{\max} = \sup \overline{\mu}_{\max}[k] < \frac{1}{2}$.
    
            \textbf{Condition (A):} To verify condition (A), we must find the condition on $N$ such that $\overline{\mu}_{\min} \geq -1/4$. From our extrema analysis, $\overline{\mu}_{\min}$ is achieved when $\overline{\mu}_{\min}[k]$ is minimized, which occurs at $\mu[k] = 1/4, \sigma^2[k] \to 0$:
            
            \begin{equation}
                \overline{\mu}_{\min} = \frac{1}{2}\left[1-e^{-4\mu^2[k]+2\mu[k]} \times e^{-4\sigma^2[k]} \times e^{+\varepsilon(N)}\right]\Bigg|_{\substack{\mu[k] = 1/4 \\ \sigma^2[k] \to 0}}
            \end{equation}
            
            The condition becomes:
            \begin{equation}
                \frac{1}{2}\left[1-e^{-4 \cdot (1/4)^2 + 2 \cdot (1/4)} \times e^{0} \times e^{+\varepsilon(N)}\right] \geq -\frac{1}{4}
            \end{equation}
            
            Simplifying $-4 \cdot (1/4)^2 + 2 \cdot (1/4) = -1/4 + 1/2 = 1/4$:
            \begin{equation}
                \frac{1}{2}\left[1-e^{1/4} \times e^{+\varepsilon(N)}\right] \geq -\frac{1}{4}
            \end{equation}

            \begin{equation}
                \begin{aligned}
                    &\Leftrightarrow 1-e^{1/4} \times e^{+\varepsilon(N)} \geq -\frac{1}{2} \\
                    &\Leftrightarrow e^{1/4} \times e^{+\varepsilon(N)} \leq \frac{3}{2} \\
                    &\Leftrightarrow e^{+\varepsilon(N)} \leq \frac{3}{2}e^{-1/4} \\
                    &\Leftrightarrow \varepsilon(N) \leq \ln\left(\frac{3}{2}\right) - \frac{1}{4} \\
                    &\Leftrightarrow \frac{402N-216}{2N(32N-72)} \leq C
                \end{aligned}
            \end{equation}
            
            where $C = \ln(3/2) - 1/4$.
            
            This is equivalent to studying the second-order polynomial:
            \begin{equation}
                -64C \times N^2 + (402+144C) \times N - 216 \leq 0
            \end{equation}
            
            Which admits a positive discriminant:
            \begin{equation}
                \Delta = (402+144C)^2 - 4 \times 64 \times 216C \approx 171508 > 0
            \end{equation}
            
            Hence, two real roots:
            \begin{equation}
                \left\{
                \begin{aligned}
                    &N_0 = \frac{(402+144C) + \sqrt{\Delta}}{128C} \approx 42.14\\
                    &N_1 = \frac{(402+144C) - \sqrt{\Delta}}{128C} \approx 0.5152\\
                \end{aligned}
                \right.
            \end{equation}
    
            Since the coefficient of $N^2$ is negative, the parabola opens downward. Therefore:
            \begin{itemize}
                \item $P(N) > 0$ for $N \in (N_1, N_0)$
                \item $P(N) < 0$ for $N < N_1$ or $N > N_0$
            \end{itemize}
            
            For the inequality $P(N) \leq 0$ to hold, we need (integer) $N \geq N_0$. Therefore, $N \geq 43$ ensures $\overline{\mu}_{\min} > -1/4$.
            
            Since, Condition (B) is valid for any $N$, we have:
            \begin{equation}
                -1/4 < \overline{\mu}_{\min} \leq \overline{\mu}[k] \leq \overline{\mu}_{\max} <1/2 \quad \forall N \geq 43, \quad \mu[k] \in[-1/4,1/2], \sigma^2[k] \in (0,1/4]
            \end{equation}

            In conclusion, since Theorem~\ref{thm:complete_variable_dynamics} guarantees that the envelope $[\overline{\mu}_{\min}, \overline{\mu}_{\max}]$ is forward invariant when $0 < m[k] < 1$, any trajectory starting within the valid parameter domain $\mu[0] \in [-1/4, 1/2]$ remains within it throughout the algorithm's evolution when $N \geq 43$. 

            Furthermore, the necessary monotonic contraction condition $0 < m[k] < 1$ from Proposition~\ref{prop:monotonic_constraint}, which holds when $\alpha < \frac{N}{2}e^{-\frac{72N-9}{32N-72}}$, simultaneously ensures convergence of the variance $\sigma^2[k]$. This dual guarantee, forward invariance of the mean and convergence of the variance, ensures that both $\mu[k]$ and $\sigma^2[k]$ remain within the assumption domain from Assumption~\ref{assumption:system_assumption} throughout the algorithm's evolution. 

        \end{proof}
       
        \begin{corollary}[Persistent Validity of Analysis]\label{cor:persistent_validity}
            Under the conditions of Proposition~\ref{prop:envelope_invariance}, all error bounds, convergence guarantees, and system properties established under Assumptions~\ref{assumption:system_assumption} remain valid throughout the algorithm's evolution.
        \end{corollary}

    \subsubsection{Convergence Analysis - Mean}

        Having established convergence of the variance and forward invariance of the parameter domain in previous sections, we now address the remaining component: the convergence behavior of the mean parameter $\mu[k]$. First, Proposition~\ref{prop:bounded_update_steps} establishes that our gradient descent updates are uniformly bounded across our assumption domain. Proposition~\ref{prop:relaxed_stability} leverages this property to significantly relax the forward invariance condition from $N \geq 43$ (as required in Proposition~\ref{prop:envelope_invariance}) to $N \geq 8$ by analyzing envelope invariance within traversable boundary buffers rather than across the entire parameter domain.  
        
        To demonstrate convergence of the mean, Proposition~\ref{prop:bounded_convergence_interval} exploits the key insight that convergence occurs at an equilibrium point where the instantaneous fixed point equals the current mean: $\overline{\mu}[k] = \mu[k]$. Since finding this intersection point is analytically intractable, we develop a bounding approach using the envelope functions established earlier in Proposition~\ref{prop:envelope_characterization}. This proposition characterizes the convergence region by bounding the intersection points of the envelope functions with the identity line, ensuring that trajectories converge to a bounded interval around the true equilibrium. Finally, Proposition~\ref{prop:complete_convergence_closure} provides the complete convergence guarantee, closing our theoretical analysis by showing that both mean and variance converge, with the mean's convergence interval size scaling as $\mathcal{O}(1/N) + \mathcal{O}(\alpha/N)$. Figure~\ref{fig:convergence_dynamics} provides a visual guide to our proof strategy.

        \begin{proposition}[Bounded Update Steps]\label{prop:bounded_update_steps}
            Under the conditions of Proposition~\ref{prop:monotonic_constraint} ensuring $0 < m[k] < 1$, for all $\mu[k],\sigma^2[k]$ satisfying Assumptions~\ref{assumption:system_assumption}, and $N \geq 2$, the update step $\delta[k] = \mu[k+1] - \mu[k]$ satisfies:
            $$|\delta[k]| < \frac{\alpha}{N} \times \left[\frac{3}{2} e^{9/4} e^{\eta^{(\xi)}_{\max}(N)} - e^{-1} e^{-\eta^{(\zeta)}_{\max}(N)}\right]$$
            where $\eta^{(\xi)}_{\max}(N) = \frac{153}{32N-72}$ and $\eta^{(\zeta)}_{\max}(N) = \frac{3}{2N}$ are the error bounds as established in Proposition~\ref{prop:exp_approx_when_pe_1_N}.
        \end{proposition}

        \begin{proof}\label{proof:bounded_update_steps}
            Under Proposition~\ref{prop:monotonic_constraint} (ensuring $0 < m[k] < 1$), we can define an update step in our system as:
            
            \begin{equation}
            \begin{aligned}
                \delta[k] &= \mu[k+1] - \mu[k] = (m[k] - 1)\mu[k] + c[k] \\
                &= -\frac{2\alpha}{N}e^{\xi[k]} e^{\eta^{(\xi)}[k]} \mu[k] + \frac{\alpha}{N}\left[e^{\xi[k]} e^{\eta^{(\xi)}[k]} - e^{\zeta[k]} e^{\eta^{(\zeta)}[k]}\right] \\
                &= \frac{\alpha}{N} \left[e^{\xi[k]} e^{\eta^{(\xi)}[k]}(1 - 2\mu[k]) - e^{\zeta[k]} e^{\eta^{(\zeta)}[k]}\right]\\
                &= \frac{\alpha}{N}\left[f[k]-g[k]\right]
            \end{aligned}
            \end{equation}
            
            where $\xi[k] = 4(\mu^2[k] + \sigma^2[k] - \mu[k])$ and $\zeta[k] = -2\mu[k]$ as defined in Remark~\ref{remark:exp_approx_update_rules}.
            
            This update step is naturally bounded on our interval of interest $\mu[k] \in [-1/4, +1/2]$, $\sigma^2[k] \in (0, 1/4]$ and can be made arbitrarily small by choosing $\alpha$ sufficiently small. As a reminder, $\alpha$ is also subject to the constraint $\alpha < \alpha_1$ from Proposition~\ref{prop:monotonic_constraint}.
            
            We now establish conservative bounds for the update step, as a function of $\alpha$ and $N$ only, which will prove useful in the rest of our analysis of the convergence dynamics of the system.
            
            We analyze the extrema of each component:      
            
            \textbf{For the first term} $f[k]=e^{\xi[k]} e^{\eta^{(\xi)}[k]}(1 - 2\mu[k])$: Since $\xi[k] = 4(\mu^2[k] + \sigma^2[k] - \mu[k])$ is decreasing in $\mu[k]$ on $(-\infty,1/2)$, $e^{\xi[k]}>0$ is also decreasing on $(-\infty,1/2)$. Since $(1-2\mu[k])$ is decreasing on $(-\infty,+\infty)$ with a zero at $\mu[k] = 1/2$, we have for the product of the two terms:
            \begin{itemize}
                \item $\min f[k]=0$ (achieved at $\mu[k] = 1/2$)
                \item $\max f[k] < e^{9/4} e^{+\eta^{(\xi)}_{\max}(N)} \cdot \frac{3}{2}$ (achieved at $\mu[k] = -1/4$, $\sigma^2[k] = 1/4$) 
            \end{itemize}
            
            \textbf{For the second term} $g[k]=e^{\zeta[k]} e^{\eta^{(\zeta)}[k]}$ where $\zeta[k] = -2\mu[k]$:
            \begin{itemize}
                \item $\min g[k]=e^{-1} e^{-\eta^{(\zeta)}_{\max}(N)}$ (achieved at $\mu[k] = 1/2$)
                \item $\max g[k]=e^{1/2} e^{+\eta^{(\zeta)}_{\max}(N)}$ (achieved at $\mu[k] = -1/4$)
            \end{itemize}
            where $\eta^{(\xi)}_{\max}(N) = \frac{153}{(32N-72)}$ and $\eta^{(\zeta)}_{\max}(N) = \frac{3}{2N}$ as per Proposition~\ref{prop:exp_approx_when_pe_1_N} .
            
            Using conservative bounding by taking the most pessimistic combination of extrema, even though they can occur at different steps, we obtain:
            \begin{equation}
            \begin{aligned}
            \frac{\alpha}{N}\left[\min f[k] - \max g[k]\right]&\leq \delta[k] \leq \frac{\alpha}{N}\left[\max f[k] - \min g[k]\right]\\
                -\frac{\alpha}{N} \times e^{1/2} e^{+\eta^{(\zeta)}_{\max}(N)} &\leq \delta[k] < \frac{\alpha}{N} \times \left[\frac{3}{2} e^{9/4} e^{\eta^{(\xi)}_{\max}(N)} - e^{-1} e^{-\eta^{(\zeta)}_{\max}(N)}\right]\\
            \end{aligned}
            \end{equation}
            
            Both bounds decrease in absolute value as $N$ increases, ensuring that update steps naturally become smaller for larger $N$. 
    
            Since we are interested in the absolute value of $\delta$, we have to study the bounds absolute values. To show that the upper bound is always greater than the lower bound in absolute value, we need to prove:
            \begin{equation}
            \frac{\alpha}{N} \times \left[\frac{3}{2} e^{9/4} e^{\eta^{(\xi)}_{\max}(N)} - e^{-1} e^{-\eta^{(\zeta)}_{\max}(N)}\right] > \frac{\alpha}{N} \times e^{1/2} e^{+\eta^{(\zeta)}_{\max}(N)}
            \end{equation}
            
            Canceling $\frac{\alpha}{N}$ (which is positive), we need:
            \begin{equation}
                \frac{3}{2} e^{9/4} e^{\eta^{(\xi)}_{\max}(N)} - e^{-1} e^{-\eta^{(\zeta)}_{\max}(N)} > e^{1/2} e^{+\eta^{(\zeta)}_{\max}(N)}
            \end{equation}
            
            Since $\eta^{(\xi)}_{\max}(N) > 0$ and $\eta^{(\zeta)}_{\max}(N) > 0$ we can bound the min and max as:
            \begin{align}
                e^{9/4} e^{\eta^{(\xi)}_{\max}(N)} &> e^{9/4} \\
                e^{-1} e^{-\eta^{(\zeta)}_{\max}(N)} &< e^{-1}
            \end{align}
            
            Using these bounds, we obtain the bounded left-hand side:
            \begin{equation}\label{eq:update_step_upper_bounding}
                \frac{3}{2} e^{9/4} e^{\eta^{(\xi)}_{\max}(N)} - e^{-1} e^{-\eta^{(\zeta)}_{\max}(N)} > \frac{3}{2}e^{9/4} - e^{-1}
            \end{equation}
            
            Similarly for the right-hand side, since $\eta^{(\zeta)}_{\max}(N) = \frac{3}{2N}$ decreases with $N$, the worst case occurs at $N = 2$, leading to the bounded expression:
            \begin{equation}
                e^{1/2} e^{+\eta^{(\zeta)}_{\max}(N)} \leq e^{1/2} e^{+\frac{3}{4}} = e^{5/4}
            \end{equation}

            Hence we need to verify:
            \begin{equation}
                \frac{3}{2} e^{9/4} e^{\eta^{(\xi)}_{\max}(N)} - e^{-1} e^{-\eta^{(\zeta)}_{\max}(N)} > \frac{3}{2}e^{9/4} - e^{-1} > e^{5/4} \geq  e^{1/2} e^{+\eta^{(\zeta)}_{\max}(N)}
            \end{equation}
            
            Where it sufficient to verify:
            \begin{equation}
                \frac{3}{2}e^{9/4} - e^{-1} > e^{5/4}
            \end{equation}
            Numerically: $14.23 - 0.3679 = 13.860 > 3.490$ 
            
            Therefore, the upper bound dominates for all $N \geq 2$ such that 
            \begin{equation}
                |\delta[k]| < \frac{\alpha}{N} \times \left[\frac{3}{2} e^{9/4} e^{\eta^{(\xi)}_{\max}(N)} - e^{-1} e^{-\eta^{(\zeta)}_{\max}(N)}\right]  := \delta_{\max}(\alpha, N)
            \end{equation}
    
            The exact constraint can be computed for any specific $N$ of interest. The key observation is that this bound exhibits the asymptotic behavior:
            \begin{equation}
                \left\{
                \begin{aligned}
                    &\lim_{N\to +\infty}\delta_{\max}(\alpha, N) = 0\\
                    &\lim_{\alpha\to 0^+}\delta_{\max}(\alpha, N) = 0\\
                \end{aligned}
                \right.
            \end{equation}
            The bound scales as $\mathcal{O}(\alpha/N)$ and can be made arbitrarily small by choosing $\alpha$ sufficiently small.

        \end{proof}

        In Proposition~\ref{prop:bounded_update_steps} we show that gradient descent update steps are uniformly bounded by $\delta_{\max}(\alpha,N)$ across our assumption domain. Hence, to violate the invariance constraint $\mu[k] \geq -1/4$, the algorithm would need to pass through the buffer zone $[-1/4, -1/4+\delta_{\max}(\alpha,N)]$ since it cannot "jump over" this region due to the bounded update constraint. Therefore instead of ensuring that $\overline{\mu}_{\min}[k] \geq -1/4$ across the entire interval of possible $\mu[k] \in [-1/4, +1/2]$ (as in Proposition~\ref{prop:envelope_invariance}) it suffices to verify that $\overline{\mu}_{\min}[k] \geq -1/4$ only with $\mu[k] \in [-1/4, -1/4+\delta_{\max}(\alpha,N)]$. This creates a protective barrier: if no fixed point below $-1/4$ exists in the buffer zone, then values below $-1/4$ become unreachable from any starting point in the admissible domain. The critical evaluation point becomes the worst case within $[-1/4, -1/4+\delta_{\max}(\alpha,N)]$, which occurs at $\mu[k] = -1/4+\delta_{\max}(\alpha,N)$ since the $\overline{\mu}_{\min}[k]$ function is decreasing on $[-1/4, +1/4]$. In the limiting case where $\delta_{\max}(\alpha,N) \to 0$, this reduces to verifying the condition at $\mu[k] = -1/4$, yielding the significantly relaxed condition $N \geq 8$ instead of $N \geq 43$ as shown in Proposition~\ref{prop:relaxed_stability}.

        \begin{proposition}[Relaxed Stability Condition via Buffer Zone Analysis]\label{prop:relaxed_stability}
            There exists $\alpha'>0$ and $N' = 8$ such that for $N \geq N'$ and $\alpha<\alpha'$, the forward invariance of the assumption domain is satisfied with a significantly relaxed bound compared to the conservative analysis of Proposition~\ref{prop:envelope_invariance} (which required $N\geq43$).
        \end{proposition}

        \begin{proof}\label{proof:relaxed_stability}
            Let $N'$ be the smallest integer such that $\overline{\mu}_{\min}(\mu[k]=-1/4, \sigma^2[k]=0, N=N') > -1/4$. Since $\overline{\mu}_{\min}[k]$ is increasing in $N$ (larger $N$ leads to tighter error margins), this implies $\overline{\mu}_{\min}(\mu[k]=-1/4, \sigma^2[k]=0, N) > -1/4$ for all $N \geq N'$.

            Since $\overline{\mu}_{\min}[k]$ is decreasing for $\mu[k]\in(-\infty,+1/4)$ and $\overline{\mu}_{\min}(\mu[k]=-1/4, \sigma^2[k]=0, N) > -1/4$ (strict inequality) by our previous definition, then any crossing (if any) where $\overline{\mu}_{\min}[k] = -1/4$ must occur at some positive distance $\lambda$ from $\mu[k] = -1/4$. This creates a buffer zone $\mathcal{B} = [-1/4, -1/4 + \lambda]$ where $\overline{\mu}_{\min}[k] \geq -1/4$ for all $\mu[k]\in \mathcal{B}$ ($\overline{\mu}_{\min}[k] = -1/4$ on the $-1/4 + \lambda$ limit of the interval).

            From Proposition~\ref{prop:bounded_update_steps}, a single update step satisfy $|\delta[k]| \leq \delta_{\max}(\alpha, N)$ where $\delta_{\max}$ is proportional to $\alpha$. Hence, we can choose $\alpha<\alpha'$ sufficiently small such that $|\delta[k]| < \lambda$. Due to bounded updates, the algorithm cannot bypass the buffer zone and must traverse it when attempting to exit the valid domain.             

            Considering $\alpha < \frac{N}{2}e^{-\frac{72N-9}{32N-72}}$, we have $0 < m[k] < 1$ for all $k$, ensuring the monotonic contraction condition required by Theorem~\ref{thm:complete_variable_dynamics}, that ensures envelope invariance.
            Since $\overline{\mu}_{\min}[k] \geq -1/4$ throughout the buffer zone by construction, any trajectory within the buffer experiences dynamics that prevent it from moving further (below $-1/4$) by local application of Theorem~\ref{thm:complete_variable_dynamics} to the buffer zone $\mathcal{B}$,  thereby ensuring containment within the valid domain $[-1/4, +1/2]$. 
            
            To find the value $N'$ guaranteeing the existence of $\mathcal{B}$, if it exists, we solve $\overline{\mu}_{\min}(\mu[k]=-1/4, \sigma^2[k]=0, N) > -1/4$.

            Substituting into the formula gives:
            \begin{equation}
            \frac{1}{2}\left[1-e^{-3/4} e^{\varepsilon(N)}\right] > -\frac{1}{4}
            \end{equation}
            
            This simplifies to $\varepsilon(N) < \ln(3/2) + 3/4 := C'$, which is equivalent to:
            \begin{equation}
            P(N) = -64C' N^2 + (402+144C') N - 216 < 0
            \end{equation}
            
            With $C' \approx 1.156$, the roots are approximately $N \approx 0.4009$ and $N \approx 7.285$. Since the parabola opens downward, we need $N > 7.28$, hence $N' = 8$.

            Forward invariance of the assumption domain follows from three key properties:
            \begin{enumerate}
                \item \textbf{Upper bound protection:} For any $\mu[k] \in [-1/4, +1/2]$, we have $\overline{\mu}[k] < \overline{\mu}_{\max}[k] < 1/2$ (Proposition~\ref{prop:envelope_invariance}), preventing trajectories from exceeding the upper boundary.
                \item \textbf{Lower bound protection (buffer zone):} For $N \geq 8$, there exists $\alpha' > 0$ such that for all $\alpha < \alpha'$ and $\mu[k] \in \mathcal{B}$, we have $\overline{\mu}[k] > \overline{\mu}_{\min}[k] \geq -1/4$. Since bounded updates force trajectories to traverse $\mathcal{B}$ before exiting the domain, and dynamics within $\mathcal{B}$ prevent further descent below $-1/4$, the lower boundary is protected.
                \item \textbf{Variance convergence:} The monotonic contraction condition $\alpha < \alpha_1$ (Proposition~\ref{prop:monotonic_constraint}) ensures applicability of Theorem~\ref{thm:complete_variable_dynamics} and guarantees variance convergence (Proposition~\ref{prop:sigma_convergence}).
            \end{enumerate}
            
            Therefore, $N \geq 8$ and $\alpha < \min(\alpha', \alpha_1)$ guarantee forward invariance of the domain $[-1/4, +1/2]$ and convergence of the variance.
        \end{proof}

        After validating forward invariance of the domain with relaxed conditions, we now study convergence of the mean in Proposition~\ref{prop:bounded_convergence_interval}. We encourage the reader to follow the visual sketch of the proof in Figure~\ref{fig:convergence_dynamics} for facilitated understanding.
        
        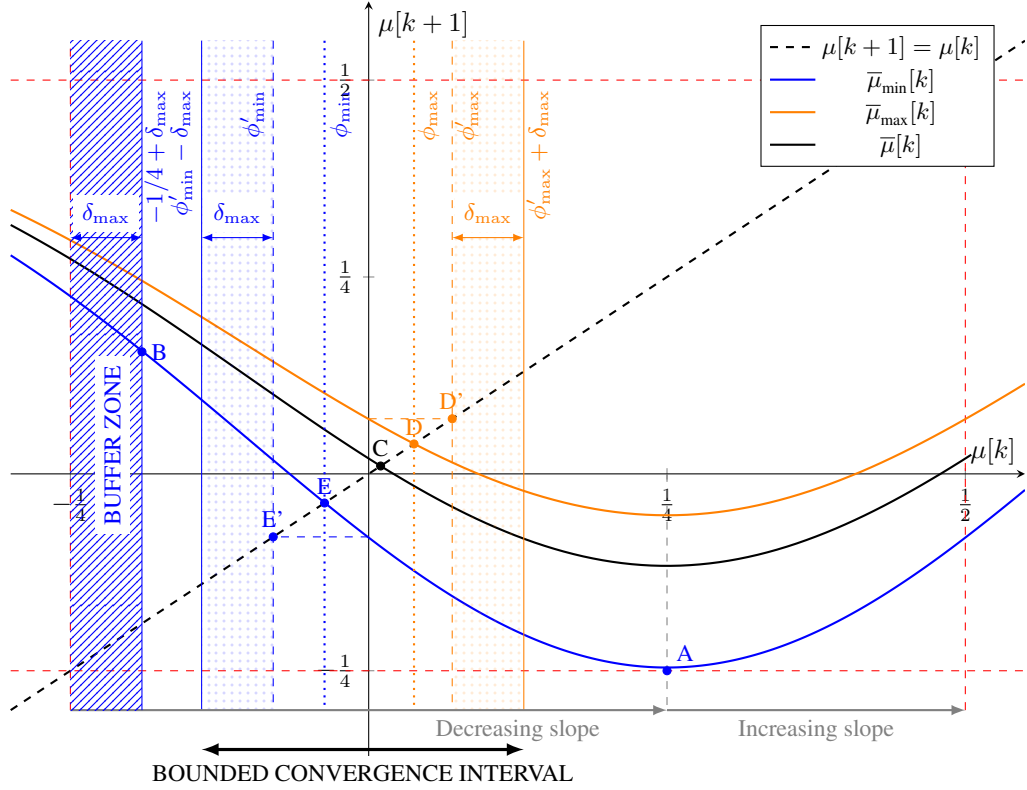
\begin{figure}
    \centering
    \begin{tikzpicture}[scale=1]
        \begin{axis}[
            width=15cm,
            height=12cm,
            grid=none,
            xlabel={$\mu[k]$},
            ylabel={$\mu[k+1]$},
            axis lines=middle,
            xmin=-0.3, xmax=0.55, ymin=-0.4, ymax=0.6,
            xtick={-0.25, 0.0, 0.25, 0.5},
            ytick={-0.25, 0.0, 0.25, 0.5},
            xticklabels={$-\frac{1}{4}$, $0$, $\frac{1}{4}$ , $\frac{1}{2}$},
            yticklabels={$-\frac{1}{4}$, $0$, $\frac{1}{4}$ , $\frac{1}{2}$},
            legend pos=north east,
            legend style={font=\small},
        ]

            \addplot[thick, black, dashed, domain=-0.3:0.55] {x};

            \draw[red, dashed] (axis cs:-0.3,-0.25) -- (axis cs:+0.55,-0.25);
            \draw[red, dashed] (axis cs:-0.3,+0.5) -- (axis cs:+0.55,+0.5);
            \draw[red, dashed] (axis cs:-0.25,-0.3) -- (axis cs:-0.25,+0.55);
            \draw[red, dashed] (axis cs:+0.5,-0.3) -- (axis cs:+0.5,+0.55);

            \pgfmathsetmacro{\sigma}{0.0}
            \pgfmathsetmacro{\epsilonExp}{0.15} 
            \addplot[thick, blue, solid, domain=-0.3:+0.55, samples=1000] {
                0.5*(1 - exp(-4*x^2 - 4*(\sigma)^2 + 2*x + \epsilonExp))
            };
            \addplot[thick, orange, solid, domain=-0.3:+0.55, samples=100] {
                0.5*(1 - exp(-4*x^2 - 4*(\sigma)^2 + 2*x - \epsilonExp))
            };
            \addplot[thick, black, solid, domain=-0.3:+0.505, samples=100] {
                0.5*(1 - exp(-4*x^2 - 4*(\sigma)^2 + 2*x - 0.04))
            };
            
            \addplot[black, only marks, mark=*, mark size=1.5pt] coordinates {(0.01, 0.01)};
            \node[black, above, rotate=0,font=\footnotesize] at (axis cs:0.01,0.01) {C};

            \addplot[orange, only marks, mark=*, mark size=1.5pt] coordinates {(0.038, 0.038)};
            \node[orange, above, rotate=0,font=\footnotesize] at (axis cs:0.038, 0.038) {D};
            \draw[orange, dotted, thick] (axis cs:0.038,0.55) -- (axis cs:0.038,-0.3);
            \node[orange, above left, rotate=90, font=\footnotesize] at (axis cs:0.067,+0.5) {$\phi_\mathrm{max}$};

            \addplot[blue, only marks, mark=*, mark size=1.5pt] coordinates {(-0.037, -0.037)};
            \node[blue, above, rotate=0,font=\footnotesize] at (axis cs:-0.037, -0.037) {E};
            \draw[blue, dotted, thick] (axis cs:-0.037,0.55) -- (axis cs:-0.037,-0.3);
            \node[blue, below left, rotate=90, font=\footnotesize] at (axis cs:-0.04,+0.5) {$\phi_\mathrm{min}$};

            \addplot[orange, only marks, mark=*, mark size=1.5pt] coordinates {(0.07, 0.07)};
            \node[orange, above, rotate=0,font=\footnotesize] at (axis cs:0.07, 0.07) {D'};
            \draw[orange, dashed] (axis cs:0.0,0.07) -- (axis cs:0.07,0.07);
            \draw[orange, dashed] (axis cs:0.07,0.55) -- (axis cs:0.07,-0.3);
            \node[orange, above left, rotate=90, font=\footnotesize] at (axis cs:0.10,+0.5) {$\phi_\mathrm{max}'$};

            \addplot[blue, only marks, mark=*, mark size=1.5pt] coordinates {(-0.08, -0.08)};
            \node[blue, above, rotate=0,font=\footnotesize] at (axis cs:-0.08, -0.08) {E'};
            \draw[blue, dashed] (axis cs:0.0,-0.08) -- (axis cs:-0.08,-0.08);
            \draw[blue, dashed] (axis cs:-0.08,0.55) -- (axis cs:-0.08,-0.3);
            \node[blue, above left, rotate=90, font=\footnotesize] at (axis cs:-0.08,+0.5) {$\phi_\mathrm{min}'$};

            \draw[orange, solid] (axis cs:+0.13,0.55) -- (axis cs:+0.13,-0.3);
            \node[orange, above left, rotate=90,font=\footnotesize] at (axis cs:+0.16,+0.5) {$\phi_\mathrm{max}'+\delta_\mathrm{max}$};
            \addplot [ draw=none, postaction={pattern = dots, pattern color=orange, opacity=0.4}] coordinates {
                (0.07, 0.55)
                (0.07, -0.3)
                (0.13, -0.3)
                (0.13, 0.55)
            };
            \draw[latex-latex, orange, thin] (axis cs:0.07,0.3) -- (axis cs:0.13,0.3) 
                node[pos=0.5, above, sloped, fill=white, font=\footnotesize, yshift=2pt] {$\delta_\mathrm{max}$};

            \draw[blue, solid] (axis cs:-0.14,0.55) -- (axis cs:-0.14,-0.3);
            \node[blue, above left, rotate=90,font=\footnotesize] at (axis cs:-0.14,+0.5) {$\phi_\mathrm{min}'-\delta_\mathrm{max}$};
            \addplot [ draw=none, postaction={pattern = dots, pattern color=blue, opacity=0.4}] coordinates {
                (-0.14, 0.55)
                (-0.14, -0.3)
                (-0.08, -0.3)
                (-0.08, 0.55)
            };

            \draw[latex-latex, blue, thin] (axis cs:-0.08,0.3) -- (axis cs:-0.14,0.3) 
                node[pos=0.5, above, sloped, fill=white, font=\footnotesize, yshift=2pt] {$\delta_\mathrm{max}$};
                
            \draw[gray, dashed] (axis cs:0.25,0.0) -- (axis cs:0.25,-0.3);
            \draw[-latex, gray, thick] (axis cs:-0.25,-0.3) -- (axis cs:+0.25,-0.3) 
                node[pos=0.75, below, sloped, font=\footnotesize] {Decreasing slope};
            \draw[-latex, gray, thick] (axis cs:+0.25,-0.3) -- (axis cs:+0.5,-0.3) 
                node[midway, below, sloped, font=\footnotesize] {Increasing slope};
            \addplot[blue, only marks, mark=*, mark size=1.5pt] coordinates {(0.25, -0.25)};
            \node[blue, above right, rotate=0,,font=\footnotesize] at (axis cs:0.25, -0.25) {A};

            \addplot [draw=none, postaction={pattern = north east lines, pattern color=blue}] coordinates {
                (-0.25, 0.55)
                (-0.25, -0.3)
                (-0.19, -0.3)
                (-0.19, 0.55)
            };
            \draw[blue, solid] (axis cs:-0.19,0.55) -- (axis cs:-0.19,-0.3);
            \draw[latex-latex, blue, thin] (axis cs:-0.25,0.3) -- (axis cs:-0.19,0.3) 
                node[pos=0.5, above, sloped, font=\footnotesize, fill=white, yshift=2pt] {$\delta_\mathrm{max}$};
            \node[blue, above left, rotate=90,font=\footnotesize] at (axis cs:-0.16,+0.5) {$-1/4+\delta_\mathrm{max}$};
            \node[blue, above left, rotate=90,fill=white,font=\footnotesize] at (axis cs:-0.20,+0.15) {BUFFER ZONE};
            \addplot[blue, only marks, mark=*, mark size=1.5pt] coordinates {(-0.19, +0.155)};
            \node[blue, right, rotate=0,,font=\footnotesize] at (axis cs:-0.19, +0.155) {B};

            \draw[latex-latex, black, very thick] (axis cs:-0.14,-0.35) -- (axis cs:+0.13,-0.35) 
                node[pos=0.5, below, sloped, fill=white, font=\footnotesize, yshift=-2pt] {BOUNDED CONVERGENCE INTERVAL};
            
            \addlegendentry{$\mu[k+1]=\mu[k]$}
            \addlegendentry{$\overline{\mu}_{\textrm{min}}[k]$}
            \addlegendentry{$\overline{\mu}_{\textrm{max}}[k]$}
            \addlegendentry{$\overline{\mu}[k]$}

        \end{axis}
    \end{tikzpicture}    
    \caption{Visual sketch of the proof. The convergence analysis seeks to find the equilibrium point where the instantaneous fixed point equals the current mean: $\overline{\mu}[k]=\mu[k]$ (point C, assumed unique for the Figure). The dynamics drive trajectories with $\mu[k]$ below this equilibrium upward, and those above downward, based on the established behavior of the instantaneous fixed point function relative to the identity function. Since the analytic solution of $\overline{\mu}[k]=\mu[k]$ is non-trivial, we bound the true function with $\overline{\mu}_{\min}[k] < \overline{\mu}[k] < \overline{\mu}_{\max}[k]$. We then establish forward invariance of the domain $\mu[k] \in [-1/4, +1/2]$ through two approaches: (Proposition~\ref{prop:envelope_invariance}) ensuring the global extrema over the domain (point A) satisfy the invariance condition, which requires $N \geq 43$; or (Proposition~\ref{prop:relaxed_stability}) using a less restrictive condition which ensures that the minimum (point B) of the traversable buffer zone in a single update step (bounded by $\delta_{\max}$ from Proposition~\ref{prop:bounded_update_steps}) remains within $[-1/4, +1/2]$. The upper bound $+1/2$ is simpler to maintain since $\overline{\mu}_{\max}[k] < 1/2$ can be verified for all $(\mu[k],\sigma^2[k],N)$. Note that this figure represents the case where $\sigma^2[k]\to0$ which minimize the value of $\overline{\mu}_{\min}[k]$, thus making point (A) or (B) worst case scenario across assumption domain. Rather than finding the exact intersection point C, we bound it using the intersections of the envelope functions $\overline{\mu}_{\min}[k]$ and $\overline{\mu}_{\max}[k]$ with the identity line, occurring at $\mu[k] = \phi_{\min}$ and $\mu[k] = \phi_{\max}$ respectively (points E and D). Since these intersection points are also analytically challenging, we exploit the strictly decreasing property of both functions to derive conservative bounds. We approximate the intersections as if the functions were flat (the slowest decreasing case), corresponding to projecting their values at the origin onto the identity line, yielding the more tractable bounds $\phi_{\min}'$ and $\phi_{\max}'$ (points E' and D'). These bounds guarantee that for any $\mu[k] < \phi_{\min}'$, we have $\overline{\mu}[k] > \mu[k]$ (driving the trajectory upward), and for any $\mu[k] > \phi_{\max}'$, we have $\overline{\mu}[k] < \mu[k]$ (driving the trajectory downward). Combined with the bounded gradient descent updates (limited by $\delta_{\max}$ from Proposition~\ref{prop:bounded_update_steps}), this ensures that $\mu[k]$ necessarily converges to the padded interval $[\phi_{\min}' - \delta_{\max}, \phi_{\max}' + \delta_{\max}]\subset[-1/4,+1/4]\quad\forall N\geq 18$ (Proposition~\ref{prop:bounded_convergence_interval}).}
    \label{fig:convergence_dynamics}

\end{figure}

        \begin{proposition}[Bounded Convergence Interval]\label{prop:bounded_convergence_interval}               
            Under the conditions of Proposition~\ref{prop:monotonic_constraint} ensuring $0 < m[k] < 1$, and assuming bounded updates $|\delta[k]| \leq \delta_{\max}$ from Proposition~\ref{prop:bounded_update_steps}, there exists $\alpha_2 > 0$ and $N\geq18$, such that for all $\alpha < \alpha_2$ and $N\geq 18$, the system converges to a bounded interval $[\phi_{\min}' - \delta_{\max}, \phi_{\max}' + \delta_{\max}]$ where:
            $$\alpha_2 = N \times \frac{3 - 2e^{+\varepsilon(N)}}{4\times \left[\frac{3}{2} e^{9/4} e^{\eta^{(\xi)}_{\max}(N)} - e^{-1} e^{-\eta^{(\zeta)}_{\max}(N)}\right]}<\alpha_1$$
            and $\phi_{\min}', \phi_{\max}'$ are bounds on the intersection point of the instantaneous fixed point function $\bar{\mu[k]}$ with the invariance identity $\mu[k+1]=\mu[k]$.
            
            Specifically, if there exist values $\phi_{\min} < 0 < \phi_{\max}$ such that:
            \begin{enumerate}
                \item $\bar{\mu}[k]>\overline{\mu}_{\min}[k] > \mu[k]$ for all $\mu[k] \in (-1/4, \phi_{\min}')$
                \item $\bar{\mu}[k]<\overline{\mu}_{\max}[k] < \mu[k]$ for all $\mu[k] \in (\phi_{\max}',+1/2)$
                \item $\phi_{\min}', \phi_{\max}' \in (-1/4, +1/4)$ (within the valid domain)
            \end{enumerate}
            then any trajectory starting in $[-1/4, 1/2]$ converges to $[\phi_{\min}' - \delta_{\max}, \phi_{\max}' + \delta_{\max}]$. 
        \end{proposition}
        
        \begin{proof}\label{proof:bounded_convergence_interval}
            The proof relies on the directional convergence properties combined with bounded update constraints. We refer the reader to Figure~\ref{fig:convergence_dynamics} for a visual guide throughout the proof.
            
            \textbf{Monotonicity and Symmetry Properties}
            
            From the expression (Proposition~ \ref{prop:envelope_characterization}):
            \begin{equation}
                \overline{\mu}[k] = \frac{1}{2}[1 - e^{\zeta[k]-\xi[k]} \cdot e^{\eta^{(\zeta)}[k]-\eta^{(\xi)}[k]}]
            \end{equation} 
            where:
            \begin{equation}
                \zeta[k] - \xi[k] = -4\mu^2[k] - 4\sigma^2[k] + 2\mu[k]
            \end{equation}

            And the definitions of $\overline{\mu}_{\min}[k]$ and $\overline{\mu}_{\max}[k]$ (Proposition~ \ref{prop:envelope_characterization}): 
            \begin{equation}
                \overline{\mu}_{\min}[k] < \overline{\mu}[k] < \overline{\mu}_{\max}[k]                
            \end{equation}
            where:
            \begin{equation}
                \left\{
                \begin{aligned}
                    \overline{\mu}_{\min}[k] &= \frac{1}{2}\left[1 - e^{\zeta[k]-\xi[k]} \cdot e^{+\varepsilon(N)}\right]\\
                    \overline{\mu}_{\max}[k] &= \frac{1}{2}\left[1 - e^{\zeta[k]-\xi[k]} \cdot e^{-\varepsilon(N)}\right]
                \end{aligned}
                \right.
            \end{equation}
            with $\varepsilon(N) = \frac{402N-216}{2N(32N-72)}$.  
            
            We can verify that both $\overline{\mu}_{\min}[k]$ and $\overline{\mu}_{\max}[k]$ are:
            \begin{itemize}
                \item Decreasing function of $\mu[k]$ on $(-\infty, 1/4)$ and increasing on $(1/4, +\infty)$
                \item Symmetric around $\mu[k] = 1/4$: $f(1/4 + x) = f(1/4 - x)$
            \end{itemize}
            
            To show symmetry around $\mu[k] = 1/4$, we verify that $\zeta[k] - \xi[k]$ has identical values at $\mu[k] = 1/4 \pm x$:
            
            \begin{equation}
                \begin{aligned}
                    \zeta[k] - \xi[k] &= -4\mu^2[k] - 4\sigma^2[k] + 2\mu[k] \\
                    &= -4(1/4 \pm x)^2 - 4\sigma^2[k] + 2(1/4 \pm x) \\
                    &= -4(1/16 \pm x/2 + x^2) - 4\sigma^2[k] + 1/2 \pm 2x \\
                    &= -1/4 \mp 2x - 4x^2 - 4\sigma^2[k] + 1/2 \pm 2x \\
                    &= 1/4 - 4x^2 - 4\sigma^2[k]
                \end{aligned}
            \end{equation}
            
            Since the $\pm 2x$ terms cancel, both $\mu[k] = 1/4 + x$ and $\mu[k] = 1/4 - x$ yield the same expression. Therefore, both $\overline{\mu}_{\min}[k]$ and $\overline{\mu}_{\max}[k]$ are symmetric around $\mu[k] = 1/4$.
            
            \textbf{Intersection Analysis and Conservative Bounds}
    
            Since $\overline{\mu}_{\min}[k]$ is monotonously decreasing on $(-\infty, 1/4)$, \textbf{if} $-1/4<\overline{\mu}_{\min}[k](\mu[k]=0) < 0$, then by Corollary~\ref{cor:existence_conditions}, there exists a unique $\phi_{\min} \in (-1/4, 0)$ such that $\overline{\mu}_{\min}[k](\mu[k]=\phi_{\min}) = \phi_{\min}$.
            
            Similarly, since $\overline{\mu}_{\max}[k]$ is monotonously decreasing on $(-\infty, 1/4)$, \textbf{if} $0<\overline{\mu}_{\max}[k](\mu[k]=0)<+1/4$, there exists a unique $\phi_{\max} \in (0, 1/4)$ such that $\overline{\mu}_{\max}[k](\mu[k]=\phi_{\max}) = \phi_{\max}$. Furthermore, since  $\overline{\mu}_{\max}[k]$ is symmetric w.r.t. $\mu[k]=1/4$, if the $\phi_{\max}$ intersection exists in $(0, 1/4)$ it is necessarily unique on the extended interval $(0, 1/2]$ (and in fact in $(0, +\infty)$ since we have demonstrated that $\overline{\mu}_{\max}[k]<1/2$ anywhere). 
    
            Applying Corollary~\ref{cor:conservative_bounds} to our decreasing functions on the relevant intervals, we can bound the two intersections:

            \begin{equation}
                \left\{
                \begin{aligned}
                &\phi_{\min}' \leq \phi_{\min} \leq 0, \quad \text{where } \phi_{\min}' := \overline{\mu}_{\min}[k](\mu[k]=0)\\
                &0 \leq \phi_{\max} \leq \phi_{\max}', \quad \text{where } \phi_{\max}' := \overline{\mu}_{\max}[k](\mu[k]=0)\\
                \end{aligned}
                \right.
            \end{equation}

            Therefore, the intersection points are conservatively bounded by:
            \begin{equation}
                \left\{
                \begin{aligned}
                    &\phi_{\min}' = \frac{1}{2}[1 - e^{+\varepsilon(N)}]\\
                    &\phi_{\max}' = \frac{1}{2}[1 - e^{-\varepsilon(N)}]\\
                \end{aligned}
                \right.
            \end{equation}
            
            where 
            $$\varepsilon(N) = \eta^{(\zeta)}_{\max}(N) + \eta^{(\xi)}_{\max}(N) = \frac{402N-216}{2N(32N-72)}$$
    
            \textbf{Finding Minimum N for Valid Intersection Bounds}
    
            The previous intersection bounds analysis requires that both bounds remain within the valid domain $[-1/4, +1/4]$, where both function are strictly decreasing. Hence, we need to verify:
            \begin{equation}
                \left\{
                \begin{aligned}
                    &\frac{1}{2}[1 - e^{+\varepsilon(N)}] > -\frac{1}{4}\\
                    &\frac{1}{2}[1 - e^{-\varepsilon(N)}] < +\frac{1}{4}
                \end{aligned}
                \right.
            \end{equation}
            
            where $\varepsilon(N) = \frac{402N-216}{2N(32N-72)}$.
            
            The first condition simplifies to:
            \begin{equation}
            \frac{1}{2}[1 - e^{+\varepsilon(N)}] > -\frac{1}{4} \Leftrightarrow \varepsilon(N) < \ln\left(\frac{3}{2}\right)
            \end{equation}
            
            The second condition simplifies to:
            \begin{equation}
            \frac{1}{2}[1 - e^{-\varepsilon(N)}] < +\frac{1}{4} \Leftrightarrow \varepsilon(N) < \ln(2)
            \end{equation}
            
            Since $\ln(2) > \ln(3/2)$, the binding constraint is $\varepsilon(N) < \ln(3/2) :=C''$.
            
            Substituting the expression for $\varepsilon(N)$:
            \begin{equation}
            \frac{402N-216}{2N(32N-72)} < C''
            \end{equation}
            
            Rearranging to standard form:
            \begin{equation}
            P(N) = -64C''N^2 + (402 + 144C'')N - 216 < 0
            \end{equation}
            
            The discriminant is:
            \begin{equation}
            \Delta = (402 + 144C'')^2 - 4 \cdot 64C'' \cdot 216 \approx 189536 > 0
            \end{equation}
            
            The two real roots are:
            \begin{equation}
            \left\{
            \begin{aligned}
            N_1 &= \frac{402 + 144C'' - \sqrt{\Delta}}{128C''} \approx 0.4823\\
            N_2 &= \frac{402 + 144C'' + \sqrt{\Delta}}{128C''} \approx 17.26
            \end{aligned}
            \right.
            \end{equation}
            
             Therefore $P(N) < 0$ for $N > N_2$. Hence, we need $N\geq N'', N''=18$ to satisfy both conditions with positive margins\footnote{Note that verifying $\bar{\mu}_{\min}[k](\mu[k]=0)>-1/4$ enforce directly the condition from Proposition~\ref{prop:relaxed_stability}, by creating a large buffer zone $\mathcal{B}=[-1/4,0]$.}.
    
            \textbf{Directional Convergence Analysis}
            
            To ensure contracting behavior of the mean toward its fixed point (Theorem~\ref{thm:complete_variable_dynamics}) we need to satisfy the constraint $\alpha<\alpha_1$ ensuring that $0<m[k]<1$ (Proposition~\ref{prop:monotonic_constraint}).
    
            Since the true function $\overline{\mu}[k]$ satisfies $\overline{\mu}_{\min}[k] < \overline{\mu}[k] < \overline{\mu}_{\max}[k]$ by definition, and applying our previously defined intersection bounds $\phi_{\min}', \phi_{\max}'$, we can establish directional convergence properties outside the interval $[\phi_{\min}', \phi_{\max}']$:
            
            \begin{itemize}
                \item For $\mu[k] < \phi_{\min}'$: Since $\overline{\mu}_{\min}[k] > \mu[k]$ in this region, by definition of the intersection, we have $\overline{\mu}[k] > \overline{\mu}_{\min}[k] > \mu[k]$, implying $\mu[k+1] > \mu[k]$ (trajectory increases).
                \item For $\mu[k] > \phi_{\max}'$: Since $\overline{\mu}_{\max}[k] < \mu[k]$ in this region, by definition of the intersection, we have $\overline{\mu}[k] < \overline{\mu}_{\max}[k] < \mu[k]$, implying $\mu[k+1] < \mu[k]$ (trajectory decreases).
                \item For $\mu[k] \in [\phi_{\min}', \phi_{\max}']$: The trajectory direction is indeterminate, but from Proposition~\ref{prop:bounded_update_steps}, updates remain bounded by $|\delta[k]| \leq \delta_{\max}(\alpha,N)$.
            \end{itemize}
            
            \textbf{Convergence Region and Stability Conditions}
            
            The directional analysis ensures that trajectories are attracted toward the interval $[\phi_{\min}', \phi_{\max}']$. Combined with bounded updates, the system converges to the padded interval $[\phi_{\min}' - \delta_{\max}, \phi_{\max}' + \delta_{\max}]$, which accounts for the possibility that the trajectories within $[\phi_{\min}', \phi_{\max}']$ may exit the interval, but can deviate by at most $\delta_{\max}$ in a single step.
            
            For system stability, this convergence region must remain within the valid domain $(-1/4, 1/2)$:
            $$\phi_{\min}' - \delta_{\max} > -\frac{1}{4} \quad \text{and} \quad \phi_{\max}' + \delta_{\max} < \frac{1}{2}$$
            
            From our previous analysis:
            \begin{itemize}
                \item Any $N \geq 18$ ensures $[\phi_{\min}', \phi_{\max}'] \subset (-1/4, +1/4)$ with positive margins.
                \item From Proposition~\ref{prop:bounded_update_steps}, $\delta_{\max}(\alpha,N)$ can be made arbitrarily small by choosing sufficiently small $\alpha$.
            \end{itemize}
            
            Therefore, there exists $\alpha_2 > 0$ such that for all $0 < \alpha < \alpha_2$:
            $$[\phi_{\min}' - \delta_{\max}, \phi_{\max}' + \delta_{\max}] \subset (-1/4, 1/4)$$

            We can derive the value of $\alpha_2$ based on the update step bound (Proposition~\ref{prop:bounded_update_steps}):
            \begin{equation}
                \delta_{\max}(\alpha, N) = \frac{\alpha}{N} \times \left[\frac{3}{2} e^{9/4} e^{\eta^{(\xi)}_{\max}(N)} - e^{-1} e^{-\eta^{(\zeta)}_{\max}(N)}\right]  
            \end{equation}

            Since
            \begin{equation}
                \phi_{\min}' = \frac{1}{2}[1 - e^{+\varepsilon(N)}]
            \end{equation}

            We have:
            \begin{equation}
                \phi_{\min}' - \delta_{\max}(\alpha, N)>-\frac{1}{4} \leftrightarrow \alpha  < N \times \frac{3 - 2e^{+\varepsilon(N)}}{4\times \left[\frac{3}{2} e^{9/4} e^{\eta^{(\xi)}_{\max}(N)} - e^{-1} e^{-\eta^{(\zeta)}_{\max}(N)}\right]}:=\alpha_2 
            \end{equation}

            \textbf{Learning Rate Constraint Hierarchy}
            
            We impose two constraints on the learning rate: 
            $\alpha<\alpha_1$ ensures contracting behavior of the mean toward its fixed point (Theorem~\ref{thm:complete_variable_dynamics}) by guaranteeing $0<m[k]<1$ (Proposition~\ref{prop:monotonic_constraint}), and 
            $\alpha<\alpha_2$ ensures that the previously defined convergence interval lies within $(-1/4,+1/4)$. 
            Our objective is to establish a hierarchy between these constraints, if one exists.
            
            We recall:
            \begin{equation}
                \alpha_2=N \times \frac{3 - 2e^{+\varepsilon(N)}}{4\times \left[\frac{3}{2} e^{9/4} e^{\eta^{(\xi)}_{\max}(N)} - e^{-1} e^{-\eta^{(\zeta)}_{\max}(N)}\right]}
            \end{equation}
            
            where $\varepsilon(N)= \frac{402N-216}{2N(32N-72)}$.
            
            From Proposition~\ref{prop:bounded_update_steps}, Eq.~\eqref{eq:update_step_upper_bounding} we know the following holds true:
            \begin{equation}
                \frac{3}{2} e^{9/4} e^{\eta^{(\xi)}_{\max}(N)} - e^{-1} e^{-\eta^{(\zeta)}_{\max}(N)} > \frac{3}{2}e^{9/4} - e^{-1}:=L 
            \end{equation}
            
            which yields the bound:
            \begin{equation}
                \alpha_2 < \frac{N}{4L} \times \left(3 - 2e^{+\varepsilon(N)}\right)
            \end{equation}
            
            Since $\varepsilon(N)$ decreases to $0$ as $N$ increases, $e^{+\varepsilon(N)}$ decreases to $1$ as $N$ grows. Therefore, the maximum value of the upper bound for $\alpha_2$ is achieved by minimizing the subtractive term (which occurs when $N$ is large):
            \begin{equation}
                \alpha_2 < \frac{N}{4L} \times \left(3 - 2e^{+\varepsilon(N)}\right) \leq \frac{N}{4L} \times \left(3 - 2e^{0}\right) = \frac{N}{4L}:=\alpha_2^\mathrm{max}
            \end{equation}
            
            We recall the expression for $\alpha_1$ from Proposition~\ref{prop:monotonic_constraint}:
            \begin{equation}
                \alpha_1 := \frac{N}{2}e^{-\frac{72N-9}{32N-72}}
            \end{equation}
            
            To establish the constraint hierarchy, we determine the value of $N$ for which:
            \begin{equation}
                \alpha_1>\alpha_2^\mathrm{max}>\alpha_2
            \end{equation}
            
            This requirement leads to:
            \begin{equation}
                \begin{aligned}
                    \alpha_1&>\alpha_2^\mathrm{max}\\
                    \frac{N}{2}e^{-\frac{72N-9}{32N-72}}&>\frac{N}{4L}\\
                    e^{-\frac{72N-9}{32N-72}}&>\frac{1}{2L}\\
                    -\frac{72N-9}{32N-72}&>\ln{(1)}-\ln{(2L)}\\
                    -\frac{72N-9}{32N-72}&>-\ln{(2L)}\\
                    -(72N-9)&>-32N\times\ln{(2L)}+72\times\ln{(2L)}\\
                    N(32\times\ln{(2L)}-72)&>72\times\ln{(2L)}-9\\
                    N&>\frac{72\times\ln{(2L)}-9}{32\times\ln{(2L)}-72}\\
                \end{aligned}
            \end{equation}
            
            Numerical evaluation yields $N>6.71$. Restricting to integer values of $N$ ensures the condition holds for any $N\geq 7$.
            
            Consequently, we demonstrate that $\alpha_2<\alpha_2^\mathrm{max}<\alpha_1$ for $N\geq7$, such that any learning rate satisfying $\alpha<\alpha_2$ automatically satisfies both constraints. Note that the condition $N\geq 7$ is necessarily satisfied when either the domain invariance constraint $N\geq8$ from Proposition~\ref{prop:relaxed_stability} or the bounded convergence interval constraint $N\geq18$ is enforced.

            \textbf{Conclusion}
            
            For $N\geq18$, there exists $\alpha_2$ (satisfying both monotonic convergence constraint $0<m[k]<1$ and convergence interval containment), such that for $\alpha<\alpha_2$, any trajectory starting in the valid domain $[-1/4, 1/2]$ converges to the bounded interval $[\phi_{\min}' - \delta_{\max}, \phi_{\max}' + \delta_{\max}]\subset(-1/4,+1/4)$, where:
            \begin{equation}
                \left\{
                \begin{aligned}
                    \phi_{\min}' &= \frac{1}{2}[1 - e^{+\varepsilon(N)}] \\
                    \phi_{\max}' &= \frac{1}{2}[1 - e^{-\varepsilon(N)}]
                \end{aligned}
                \right.
            \end{equation}
            
            This convergence region narrows as $N \to \infty$ (since $\varepsilon(N) \to 0$) and as $\alpha \to 0$ (since $\delta_{\max} \to 0$), approaching the ideal fixed point at $\mu = 0$.

            It is worth noting that while our conservative bounds are tight enough to guarantee explicit containment only for $N \geq 18$, the actual convergence interval remains within the forward-invariant domain for all $N \geq 8$ supposing sufficiently small $\alpha$, under Proposition~\ref{prop:relaxed_stability}.

        \end{proof}

        \begin{remark}[Convergence Interval Size]\label{remark:convergence_interval_size}
            The size of the convergence interval is:
            $$|\phi_{\max}' - \phi_{\min}'| + 2\delta_{\max} = \frac{1}{2}[e^{+\varepsilon(N)} - e^{-\varepsilon(N)}] + 2\delta_{\max}$$
            
            Using first-order approximation, this approximates to $\varepsilon(N) + 2\delta_{\max} = \mathcal{O}(1/N) + \mathcal{O}(\alpha/N)$ for small $\varepsilon(N)$, confirming that the convergence region narrows with both increasing $N$ and decreasing $\alpha$.
        \end{remark}

        \begin{proposition}[Closure - Parity Learning Convergence]
            \label{prop:complete_convergence_closure}Consider gradient descent on the XOR architecture under unit stochastic sparsity ($p_e = 1/N$) with symmetric initialization $\mu[0] = 1/2$, $\sigma^2[0] = 1/4$. For $N \geq 18$, there exist learning rates satisfying $\alpha < \alpha_2$ such that the algorithm achieves convergence with:
                \begin{equation}
                \left\{
                \begin{aligned}
                    &\lim_{k\to\infty} \sigma^2[k] = 0 \quad \text{(exact geometric convergence)}\\
                    &\lim_{k\to\infty} \mu[k] \in [\phi_{\min}' - \delta_{\max}, \phi_{\max}' + \delta_{\max}] \quad \text{(bounded convergence)}
                \end{aligned}
                \right.
                \end{equation}
                where the convergence interval size satisfies $|\phi_{\max}' - \phi_{\min}'| + 2\delta_{\max} = \mathcal{O}(1/N) + \mathcal{O}(\alpha/N)$.
            \end{proposition}
            
            \begin{proof}
                The proof synthesizes results from key propositions establishing the necessary convergence components.
            
                \textbf{Distributional Framework Foundation:}
                Our analysis operates within the Gaussian distributional framework where weight distributions $D_0[k] \sim \mathcal{N}(\mu_0[k], \sigma_0^2[k])$ and $D_1[k] \sim \mathcal{N}(\mu_1[k], \sigma_1^2[k])$ evolve under gradient descent.
            
                \begin{itemize}
                    \item By Proposition~\ref{prop:gaussian_conservation}, gradient updates constitute affine transformations, ensuring Gaussian distributions remain Gaussian throughout training, allowing consistent system analysis throughout evolution.
                    \item By Proposition~\ref{prop:symmetry_conservation}, under symmetric initialization $\mu_0[0] = \mu_1[0] = 1/2$ and $\sigma_0^2[0] = \sigma_1^2[0] = 1/4$, the symmetry relations $\mu_0[k] = 1 - \mu_1[k]$ and $\sigma_0^2[k] = \sigma_1^2[k]$ are preserved for all $k$, allowing analysis of a single distribution $(\mu[k], \sigma^2[k]) := (\mu_0[k], \sigma_0^2[k])$ without loss of generality.
                \end{itemize}
            
                Under system Assumptions~\ref{assumption:system_assumption}, we establish: 
                
                \textbf{Variance Convergence:}
                By Proposition~\ref{prop:sigma_convergence}, under constraint $\alpha < \alpha_0$, the variance component exhibits geometric decay ensuring $\lim_{k\to\infty} \sigma^2[k] = 0$.
                                
                \textbf{Mean Convergence Analysis:} Under the more restrictive constraint $\alpha < \alpha_1 < \alpha_0$ (Proposition~\ref{prop:monotonic_constraint}), the affine transform driving the evolution of the mean term exhibit a monotonic behavior ($0<m[k]<1)$, reducing distance to the instantaneous fixed point $\overline{\mu}[k]$ without sign reversal.
                
                Under the even more restrictive constraint $\alpha < \alpha_2 < \alpha_1 < \alpha_0$ with $N \geq 18$ (Proposition~\ref{prop:bounded_convergence_interval}), the mean parameter $\mu[k]$ converges to the interval: 
                \begin{equation}
                    [\phi_{\min}' - \delta_{\max}(\alpha,N), \phi_{\max}' + \delta_{\max}(\alpha,N)]
                \end{equation} 
                where $\phi_{\min}', \phi_{\max}'$ bound the intersection points of the instantaneous fixed point $\overline{\mu}[k]$ with the identity line $\mu[k+1]=\mu[k]$ (representing equilibrium), and $\delta_{\max}(\alpha,N) = \mathcal{O}(\alpha/N)$ represents the uniform update step bound across our assumption domain. 
                
                \textbf{Domain Invariance:}
                For sufficiently small $\alpha$, the constraint  $N\geq 8$, ensures forward invariance of the assumption domain (Assumption~\ref{assumption:system_assumption} ). This property is also enforced in a stricter manner by Proposition~\ref{prop:envelope_invariance} demonstrating system invariance  for $N \geq 43$. This ensures continued validity of all properties throughout evolution. While our conservative bounds are tight enough to guarantee explicit definition of the convergence interval only for $N \geq 18$, the actual convergence interval remains within the forward-invariant domain for all $N \geq 8$ supposing sufficiently small $\alpha$, under Proposition~\ref{prop:relaxed_stability}
                
                \textbf{Global Constraint Synthesis:}
                The cascading constraint structure ensures that for any $\alpha < \alpha_2$ and $N\geq 18$, both variance and mean converge as stated above.
                
                The convergence interval size scales as $|\phi_{\max}' - \phi_{\min}'| + 2\delta_{\max} = \mathcal{O}(1/N) + \mathcal{O}(\alpha/N)$. As $N \to \infty$, both terms vanish, providing arbitrarily tight convergence.
            
                \textbf{Extension to Complete Distributional System:}
                Under Proposition~\ref{prop:symmetry_conservation}, the convergence analysis for $(\mu[k], \sigma^2[k]) := (\mu_0[k], \sigma_0^2[k])$ extends naturaly to the complete system via preserved symmetry relations, establishing convergence for both weight distributions $D_0$ and $D_1$.
                
                The synthesis of variance convergence and bounded mean convergence establishes convergent learning of XOR node under the specified conditions, completing the theoretical guarantee.
            \end{proof}

    \begin{remark}
        The analysis above does not quantify the convergence rate of $\mu[k]$ toward the convergence interval. Empirically, however, all our experiments indicate convergence in a practically reasonable number of iterations, far from any asymptotic regime (see experimental results in Section~\ref{section:experimental_validation}).
        
        Two theoretical elements nevertheless provide intuition about the effective convergence speed:
        \begin{itemize}
            \item Proposition~\ref{prop:bounded_convergence_interval} shows that $\overline{\mu}_{\max}[k] < \phi'_{\max} < 1/4$ on the interval $[0,1/2]$. Hence, for $\mu[k]$ outside the convergence interval, the instantaneous fixed point $\overline{\mu}[k]$ cannot lie arbitrarily close to the identity line. In other words, the target toward which $\mu[k]$ is pulled is not arbitrarily close to its current value, which prevents arbitrarily slow drift.
            \item Theorem~\ref{thm:complete_variable_dynamics} establishes an exponential contraction of the distance to the fixed–point envelope. Hence $\mu[k]$ approach the convergence interval with exponential rate.
        \end{itemize}
        These two facts together suggest that, under our assumptions, the mean dynamics should exhibit an effective exponential convergence toward the limiting interval.
    \end{remark}

    \section{Supplementary Experimental Results}\label{app:exp_results}    

        \subsection{Oracle Weight Proportion Independence Validation}\label{appendix:pw_independence}
        
            Figure~\ref{fig:impact_pw} demonstrates convergence behavior across different oracle weight proportions. All curves perfectly overlap, confirming identical convergence dynamics regardless of $p_w$. This validates our theoretical claim that the symmetric initialization and update rules render the learning process invariant to the underlying parity function structure as derived in Proposition~\ref{prop:symmetry_conservation}.
            
            \begin{figure}[htbp]
                \centerline{\includegraphics[width=250pt]{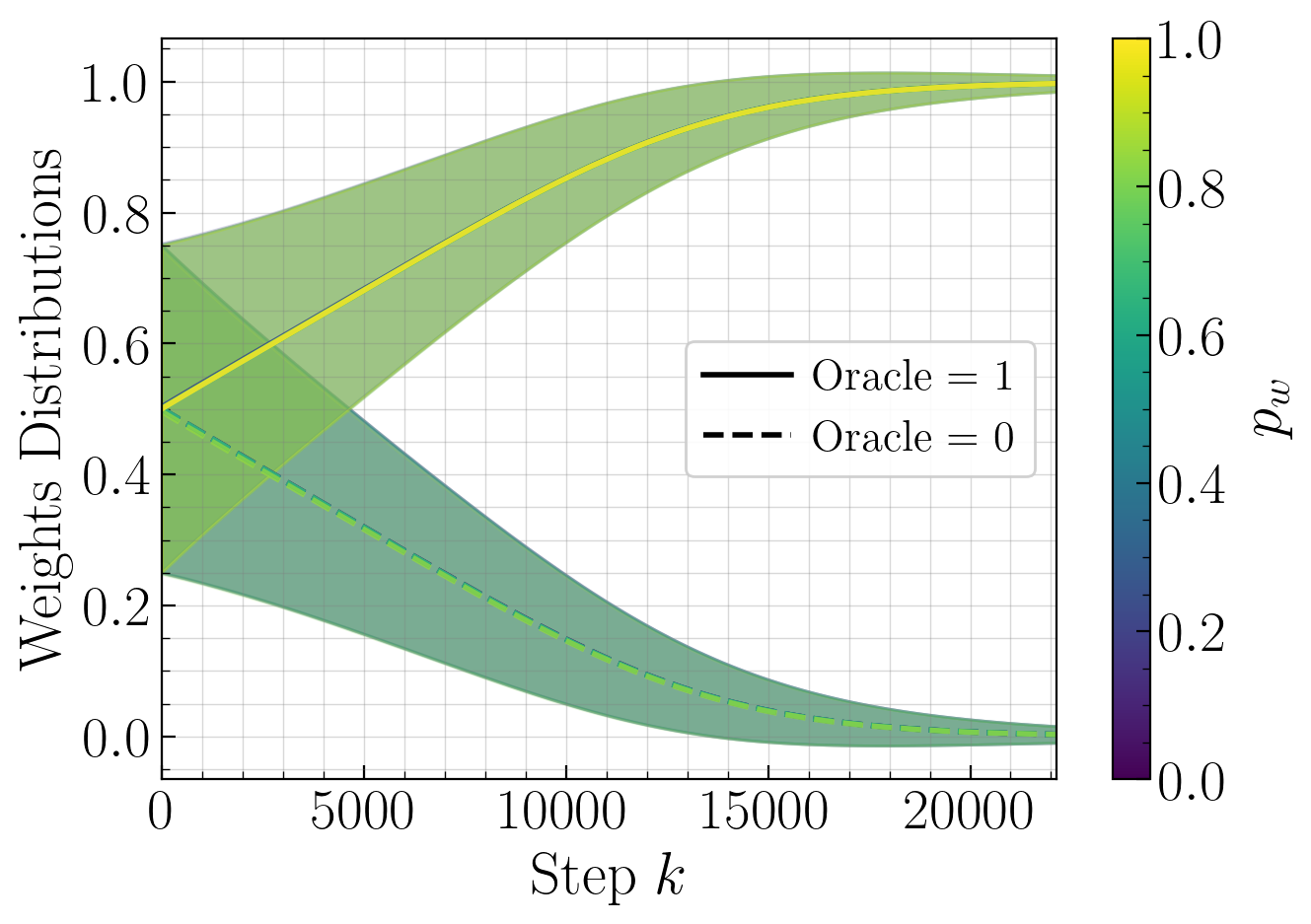}}
                \caption{Oracle weight proportion independence validation. Mean and variance evolution for different $p_w$ values (purple: $p_w=0$ to yellow: $p_w=1$). Perfect curve overlap confirms convergence behavior independence from oracle weight proportion. \textbf{Dashed}: target-0 weights; \textbf{Solid}: target-1 weights. \textit{Conditions}: $N=100{,}000$, $M=1{,}000$, $\alpha=10$, $p_e=1/N$, $P=1$, $S=1{,}000{,}000$.}
                \label{fig:impact_pw}
            \end{figure}

        \subsection{Detailed Distributional Analysis Across Oracle Weight Proportions}\label{appendix:pw_distributions}
            
            This section provides comprehensive distributional analysis validating the independence of convergence behavior from oracle weight proportion $p_w$. Figures~\ref{fig:pw_0}--\ref{fig:pw_1} show weight distributions and Q-Q plots for various $p_w$ values, mirroring the analysis of Figure~\ref{fig:gaussian_conservation} in the main text.

            These results validate Proposition~\ref{prop:symmetry_conservation} and \ref{prop:gaussian_conservation}, demonstrating that the learning dynamics remain invariant across all possible oracle weight proportion configurations, from completely sparse ($p_w=0$) to completely dense ($p_w=1$) target functions.

            \begin{figure*}[htbp]
                \centerline{\includegraphics[width=460pt]{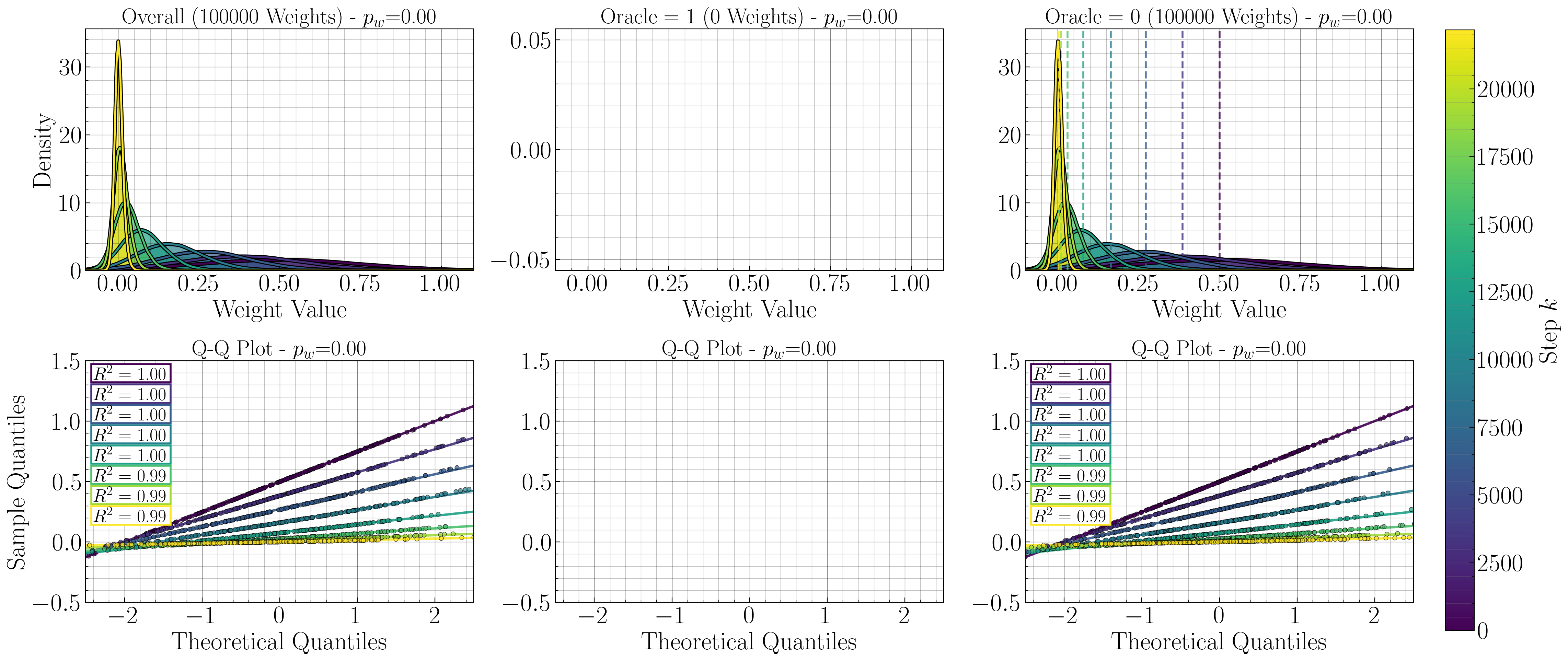}}
                \caption{Distributional analysis for oracle weight proportions $p_w = 0$. \textbf{Top row}: Weight distribution snapshots (purple: early, yellow: late) for all weights, target-0 family, and target-1 family. \textbf{Bottom row}: Q-Q plots confirming Gaussian conservation in separated families. \textit{Conditions}: $N=100{,}000$, $M=1{,}000$, $\alpha=10$, $p_e=1/N$, $S=1{,}000{,}000$.}
                \label{fig:pw_0}
            \end{figure*}

            
            \begin{figure*}[htbp]
                \centerline{\includegraphics[width=460pt]{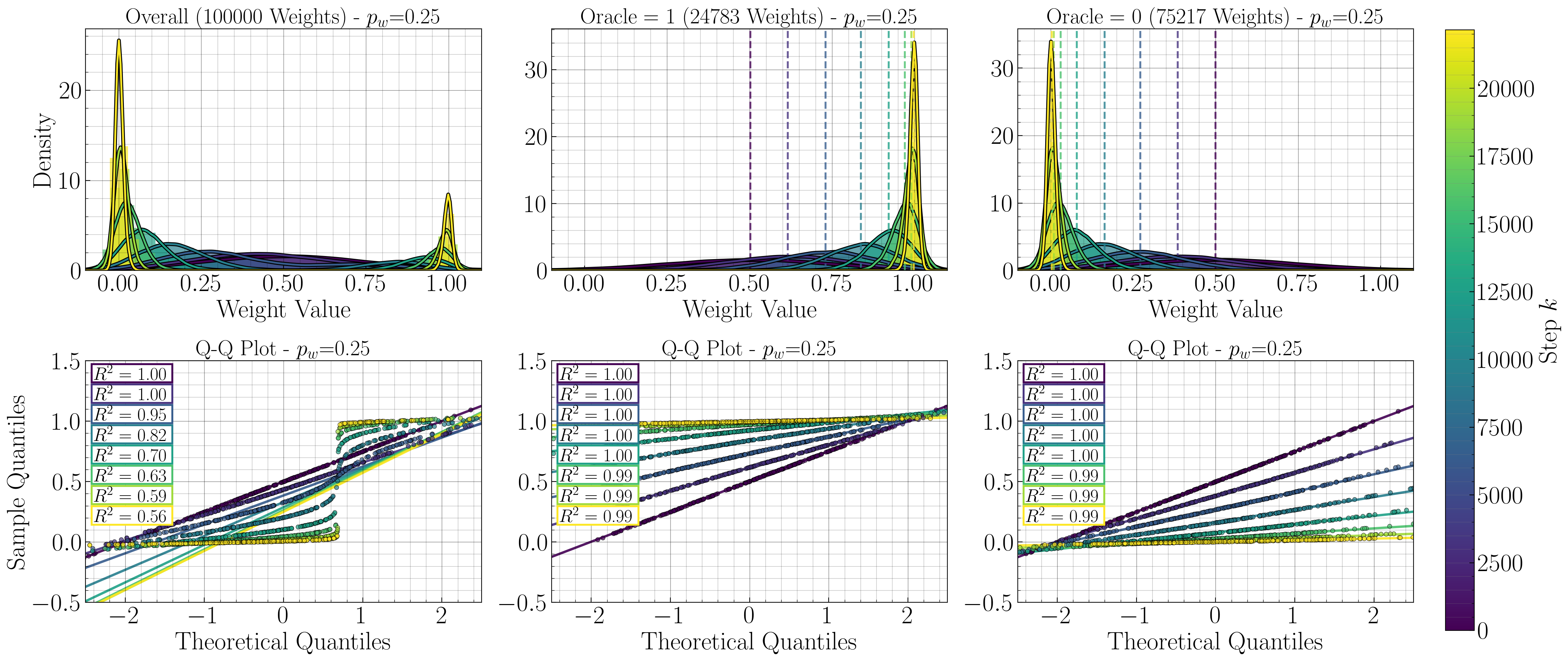}}
                \caption{Distributional analysis for oracle weight proportions $p_w = 0.25$. Layout and interpretation identical to Figure~\ref{fig:pw_0}. }
                \label{fig:pw_0_25}
            \end{figure*}
            
            \begin{figure*}[htbp]
                \centerline{\includegraphics[width=460pt]{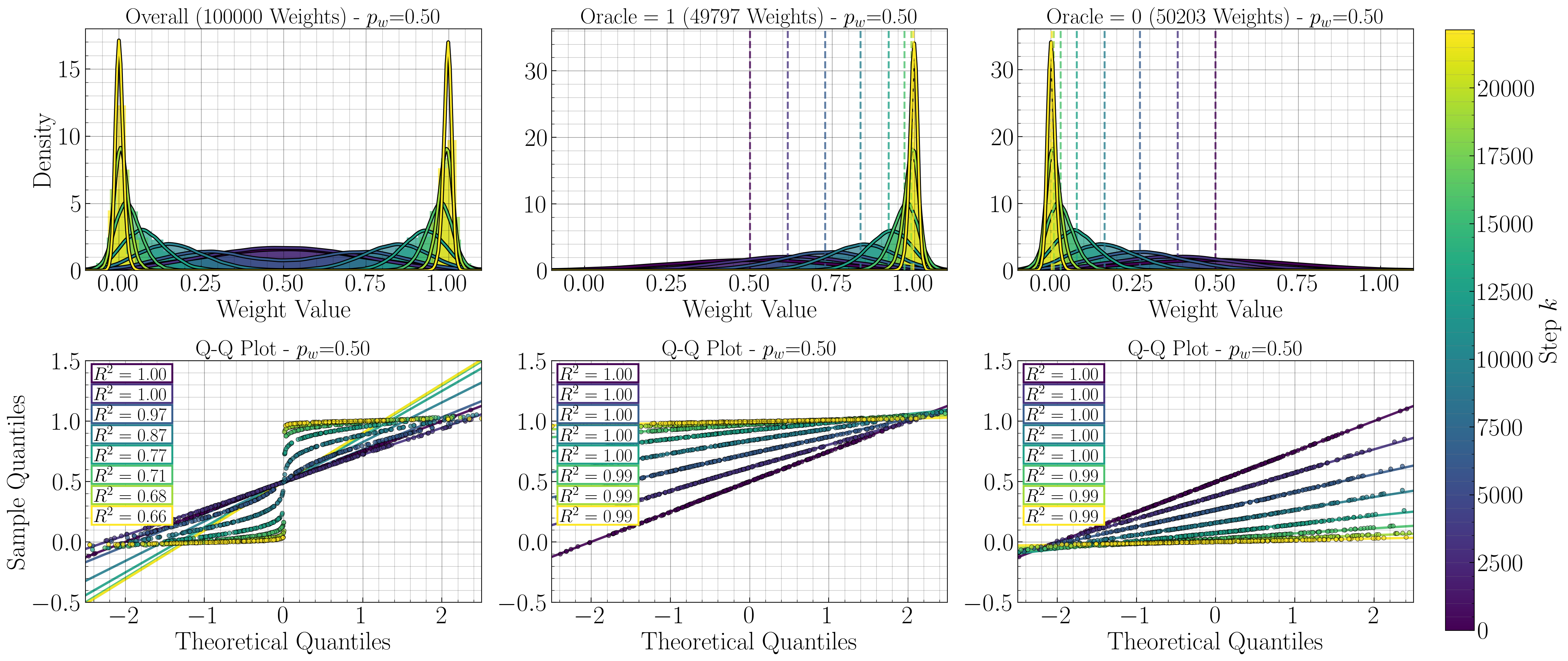}}
                \caption{Distributional analysis for oracle weight proportions $p_w = 0.5$. Layout and interpretation identical to Figure~\ref{fig:pw_0}. }
                \label{fig:pw_0_5}
            \end{figure*}

            \begin{figure*}[htbp]
                \centerline{\includegraphics[width=460pt]{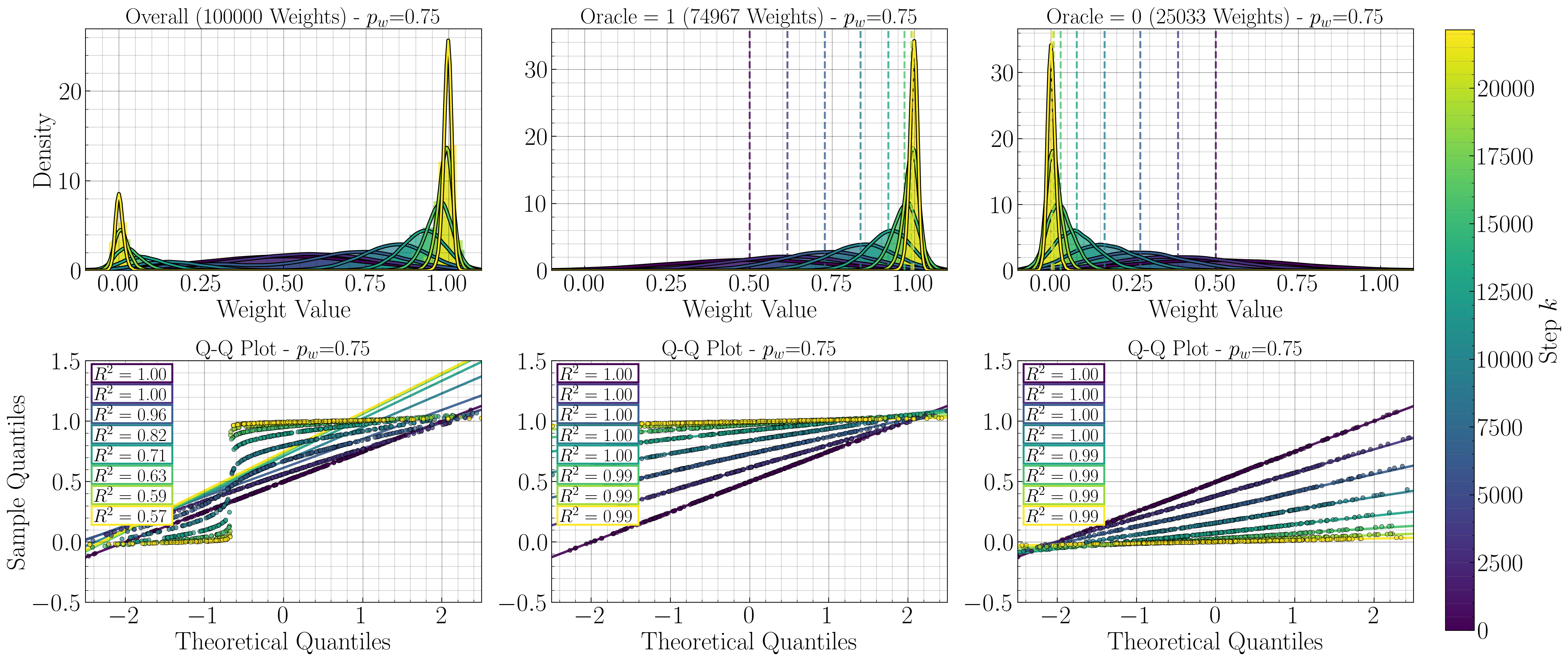}}
                \caption{Distributional analysis for oracle weight proportions $p_w = 0.75$. Layout and interpretation identical to Figure~\ref{fig:pw_0}. }
                \label{fig:pw_0_75}
            \end{figure*}

            
            \begin{figure*}[htbp]
                \centerline{\includegraphics[width=460pt]{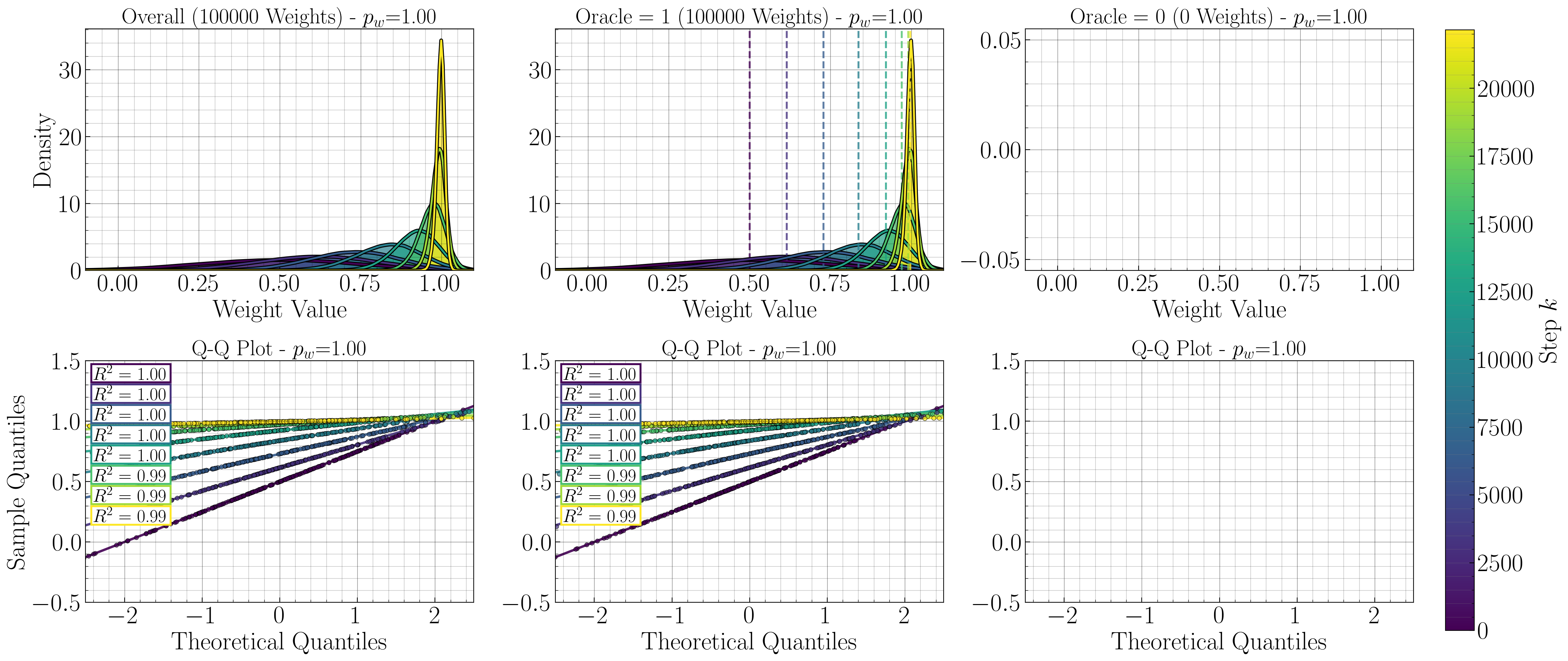}}
                \caption{Distributional analysis for oracle weight proportions $p_w = 1.0$. Layout and interpretation identical to Figure~\ref{fig:pw_0}.}
                \label{fig:pw_1}
            \end{figure*}

        \clearpage  
        \subsection{Additional Sparsity Analysis}\label{appendix:sparsity_analysis}

            Figure~\ref{fig:convergence_as_a_function_of_p_e_num_steps} provides an alternative visualization of the sparsity impact analysis, with $p_e$ on the x-axis and convergence steps on the y-axis. Iso-distance curves represent constant weight distances, confirming the same three-regime structure and optimal sparsity relationships observed in the main analysis and outlined in Proposition~\ref{prop:optimal_error_probability}. 

            \begin{figure}[htbp]
                \centerline{\includegraphics[width=230pt]{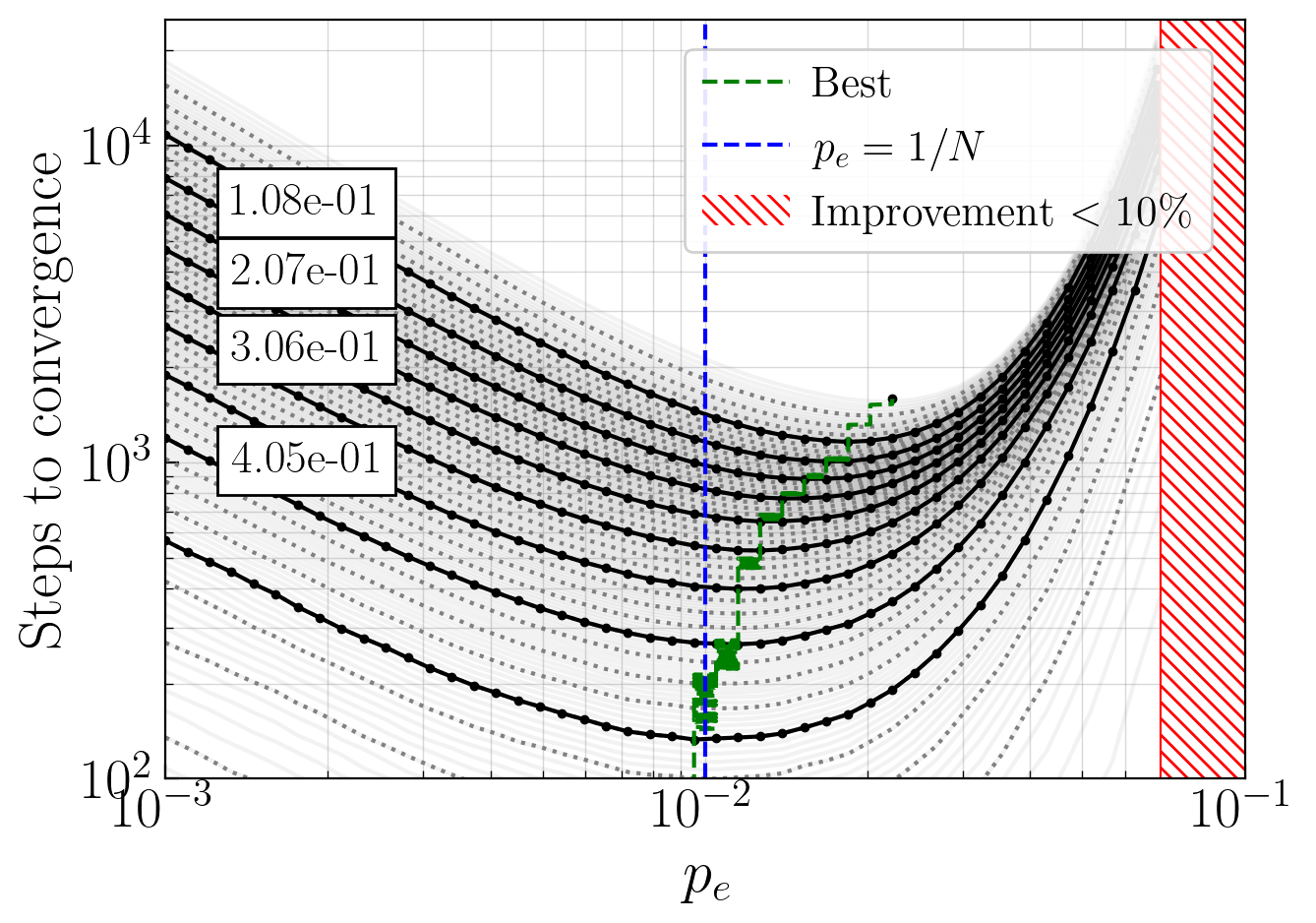}}
                \caption{Steps to convergence versus $p_e$ with iso-distance curves (constant weight distance levels) in log-log scale. Confirms three-regime structure: ultra-sparse (increasing steps with linear slope), optimal zone around $p_e^* \approx 2/N$, and dense regime with convergence failure beyond $p_e^\mathrm{lim} \approx 7/N$. The log-log scale clearly reveals that the ultra-sparse regime exhibits gradual convergence slowdown (linearly increasing slope as $p_e$ decrease from $1/N$) without complete failure, while the dense regime shows abrupt convergence breakdown beyond the critical limit. \textit{Conditions}: $N=100$, $P=1000$, $\alpha=0.1$, $M=100$, $p_w=0.5$, $S=25{,}000$.}
                \label{fig:convergence_as_a_function_of_p_e_num_steps}
            \end{figure}

        
        \subsection{Additional Learning Rate Analysis}\label{appendix:learning_rate_analysis}

            This section provides complementary analysis to Section~\ref{subsubsection:impact_of_learning_rate}. Figure~\ref{fig:convergence_as_a_function_of_alpha} presents iso-step curves showing the relationship between learning rate $\alpha$ and convergence performance, analogous to the sparsity analysis in Figure~\ref{fig:convergence_as_a_function_of_p_e}.
            
            \begin{figure}[htbp]
                \centerline{\includegraphics[width=250pt]{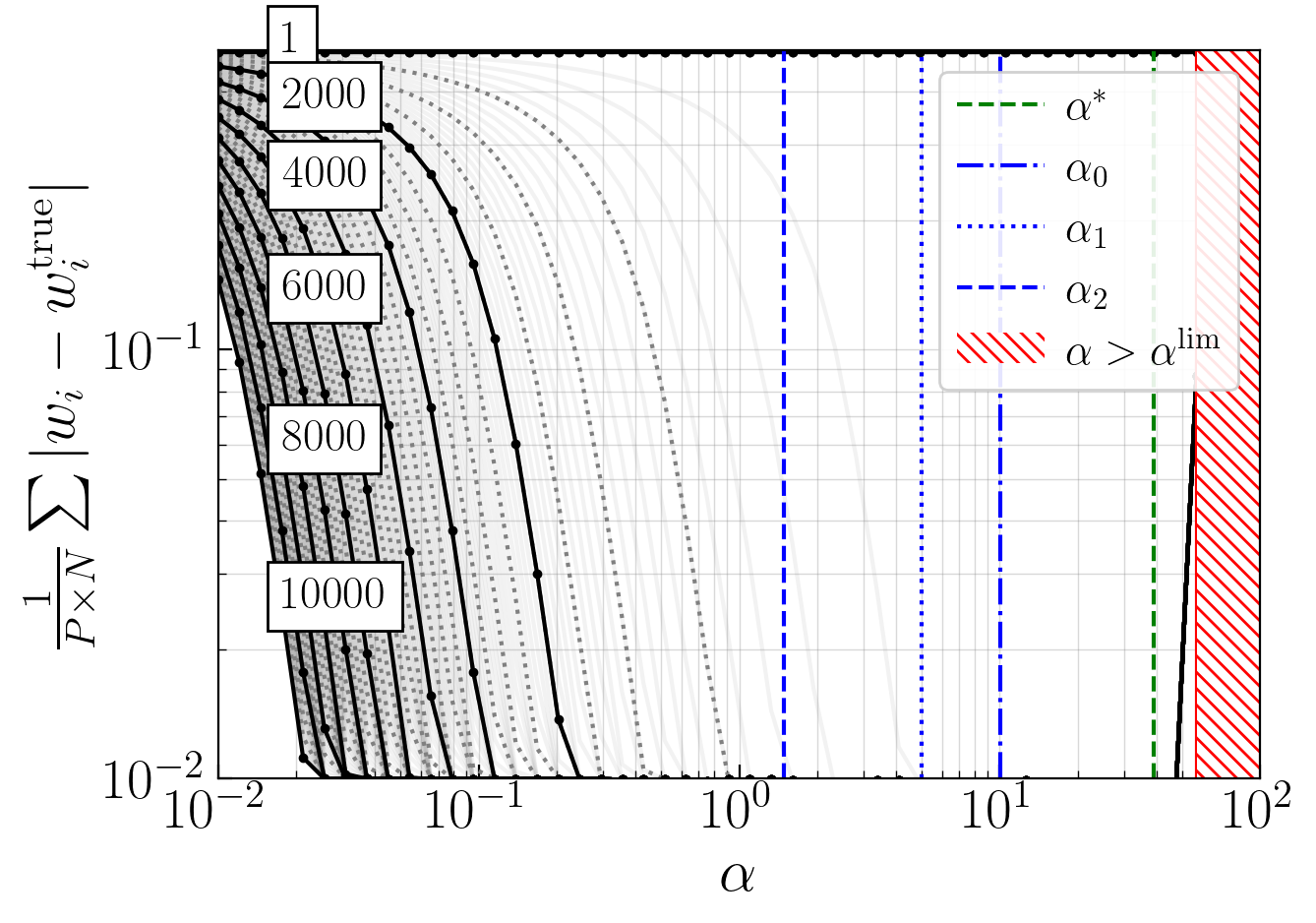}}
                \caption{Convergence analysis across learning rate regimes. Iso-step curves show training steps required to reach different convergence levels (distance $d$) for various learning rates $\alpha$. The visualization confirms that convergence speed increases monotonically with learning rate until reaching a critical breakdown threshold $\alpha^{\text{lim}}$ beyond which convergence fails entirely. Theoretical bounds $\alpha_0$, $\alpha_1$, and $\alpha_2$ are overlaid for comparison with empirical behavior. \textit{Conditions}: $N=100$, $P=1{,}000$, $p_e=1/N$, $M=50{,}000$, $p_w=0.5$, $S=10{,}000$.}
                \label{fig:convergence_as_a_function_of_alpha}
            \end{figure}
            
            The analysis reveals a simple two-regime behavior:
            
            \begin{itemize}
                \item \textbf{Convergent regime} ($\alpha < \alpha^{\text{lim}}$): Convergence speed increases with learning rate. Higher learning rates consistently lead to faster convergence without stability issues.
                
                \item \textbf{Divergent regime} ($\alpha \geq \alpha^{\text{lim}}$): Complete convergence failure occurs beyond the critical threshold, in line with the existence of finite thresholds demonstrated in Propositions~\ref{prop:sigma_convergence}, \ref{prop:monotonic_constraint} and \ref{prop:bounded_convergence_interval}.
            \end{itemize}

            \begin{figure}[htbp]
                \centerline{\includegraphics[width=235pt]{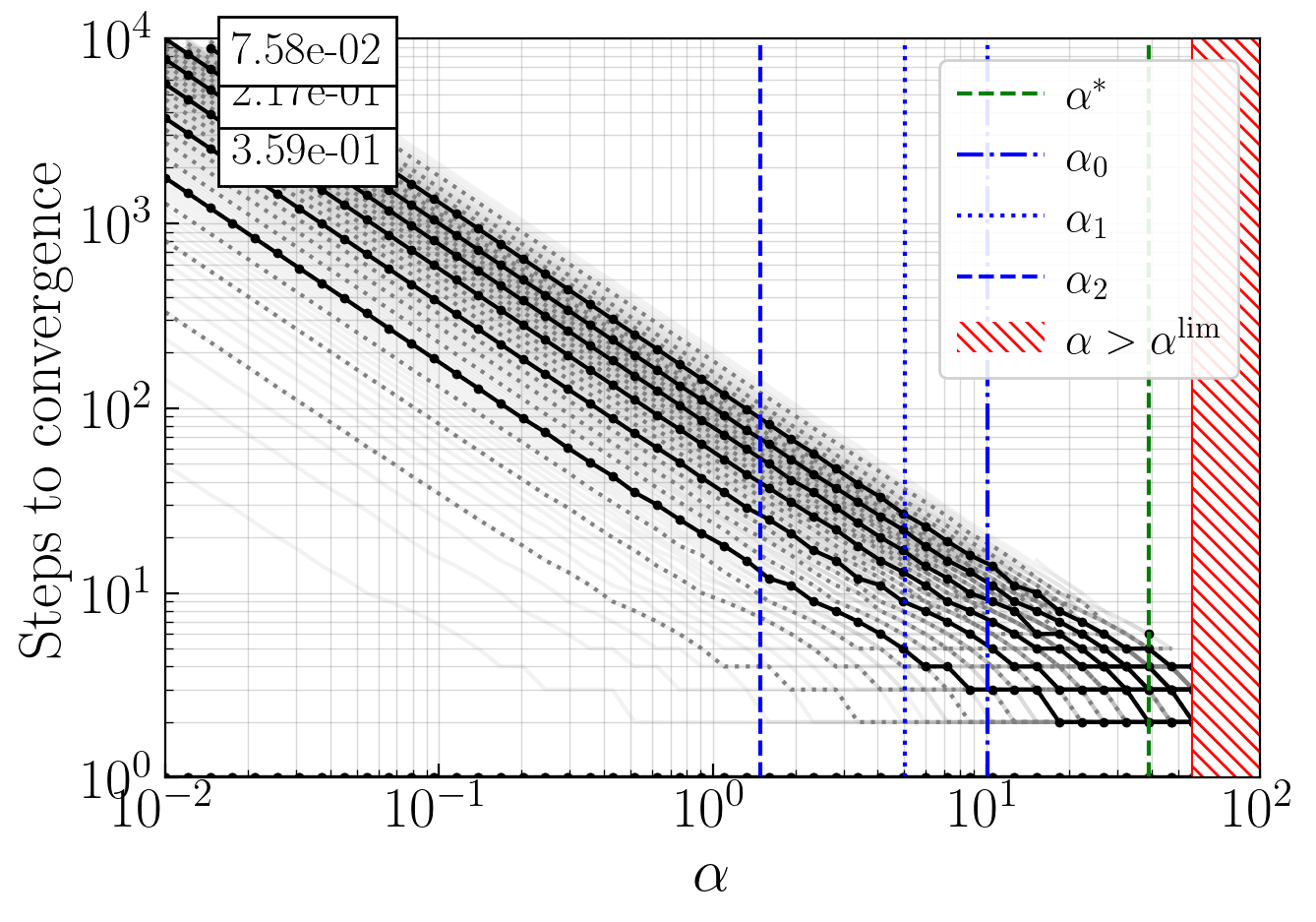}}
                \caption{Learning rate impact on convergence speed. Iso-distance curves (constant weight distance levels) show linear convergence speed until breakdown at $\alpha^{\text{lim}}$. Theoretical bounds $\alpha_0$ (Prop.~\ref{prop:sigma_convergence}), $\alpha_1$ (Prop.~\ref{prop:monotonic_constraint}), $\alpha_2$ (Prop.~\ref{prop:bounded_convergence_interval}) are overlaid, with the most strict, $\alpha_2<a^\mathrm{lim}$ providing valid convergence guarantee. \textit{Conditions}: $N=100$, $P=10$, $p_e=1/N$, $M=50,00$, $p_w=0.5$, $S=10{,}000$.}
                \label{fig:convergence_as_a_function_of_alpha_num_steps}
            \end{figure}
            Figure~\ref{fig:convergence_as_a_function_of_alpha_num_steps} shows a clear linear relationship between the number of step to convergence and the learning rate until critical threshold $\alpha^{\text{lim}}$ where training fails. The optimal rate \emph{w.r.t} convergence speed $\alpha^*$ occurs immediately below this threshold, indicating maximum stable learning rate provides optimal performance. 

        \subsubsection{Sample Complexity \& Convergence Speed}\label{app:sample_complexity_convergence_speed}

            We investigate convergence speed across all combinations of learning rate $\alpha \in \{1, N/10\}$,  sparsity $p_e \in \{0.001, 1/N\}$, and batch size $M\in \{10, 100, 1000, 10000\}$ for problem sizes $N \in [10, 1000]$.
            
            \begin{figure}[htbp]
                \centerline{\includegraphics[width=235pt]{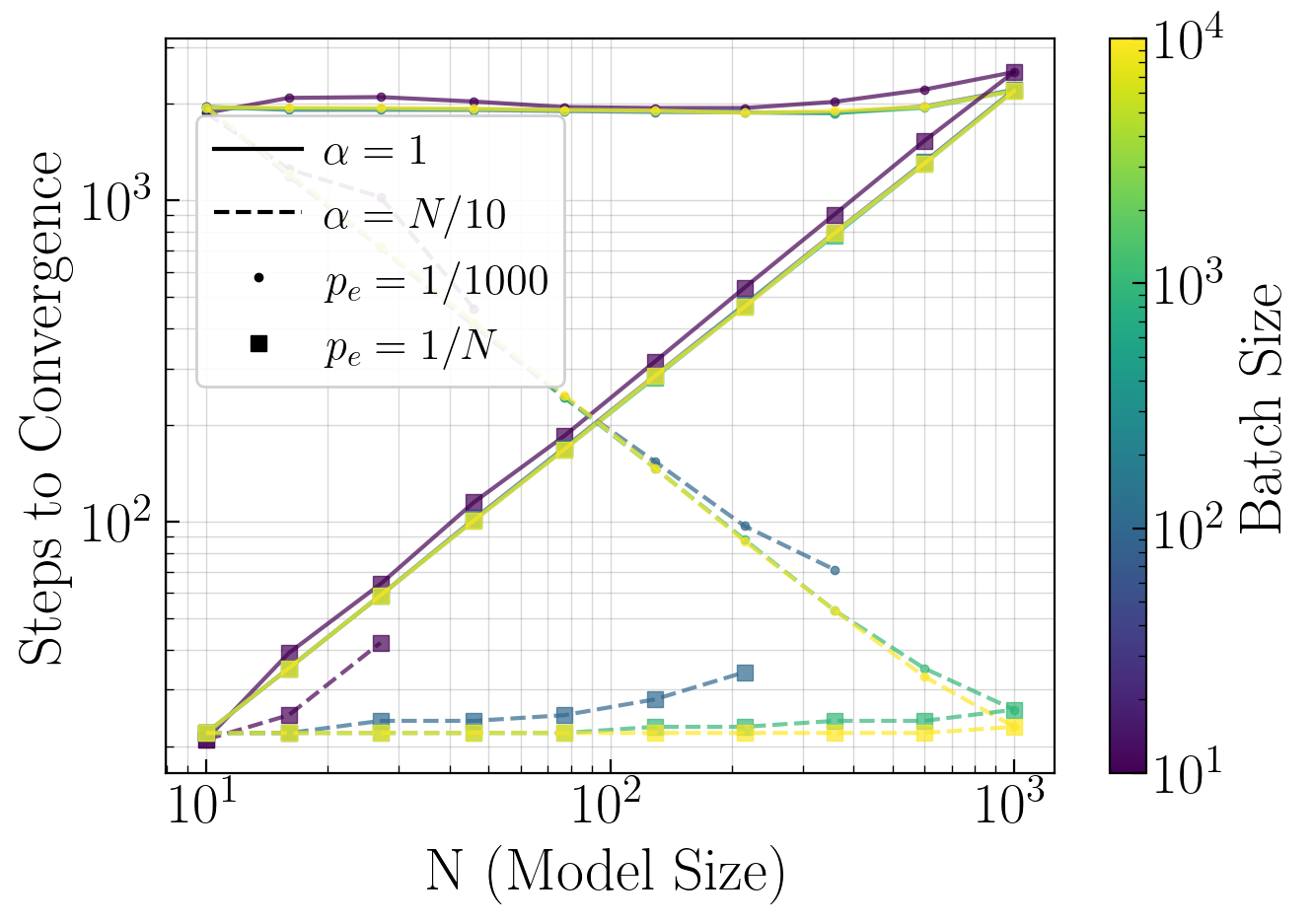}}
                \caption{Convergence analysis across hyperparameter combinations. \textbf{Color:} batch size, \textbf{lines: } learning rates, \textbf{markers:} sparsity levels. Constant effective learning rate $\alpha p_e$ yields scale-invariant convergence times. \textit{Conditions}: $P=\lfloor1000/N\rfloor$, $p_w=0.5$, $S=25{,}000$.}
                \label{fig:convergence_study}
            \end{figure}

            
            Figure~\ref{fig:convergence_study} corroborates that convergence speed is determined by the product $\alpha p_e$, consistent with theoretical results (Gradient Eq.~\eqref{eq:symmetric_gradient}). The scaling relationships $\alpha \propto N$ and $p_e \propto 1/N$ maintain constant effective learning rate $\alpha p_e$, yielding consistent convergence speed regardless of problem size. This supports that larger problems require proportionally larger learning rates to compensate for sparser inputs. 

            Batch size is not shown to affect convergence speed, except when small batches ($M \ll N$) combine with large learning rates. Consistent with our findings above, the large $\alpha \propto N$ values required to compensate for $p_e \propto 1/N$ scaling might increase sensitivity to destructive batch outliers\footnote{When $p_e = 1/N$, input vectors follow a binomial distribution for the number of active bits, which converges to a Poisson distribution for large $N$. In expectation, most inputs are one-hot vectors ($n_e = 1$), which provide informative gradients, while inputs with $n_e > 1$ can be destructive to learning dynamics. For small batch sizes, random sampling can produce batches where destructive samples ($n_e > 1$) outnumber informative samples ($n_e = 1$). We hypothesize that while such "unfavorable", outliers, batches are manageable under moderate learning rates, they become problematic near $\alpha^{\text{lim}}$ (especially for the large $\alpha^{\text{lim}}$ values permitted by larger $N$) where few destructive batches can push the model outside its stability basin, causing permanent convergence failure.}, which are more frequent with small $M$. These results demonstrate scale-invariant training duration when hyperparameters are properly scaled, suggesting polynomial sample complexity. However, this requires sufficient batch sizes to ensure statistical stability. 

    \section{Supporting Mathematical Results}\label{app:supporting_results}
        \subsection{General Approximation and Bounds}\label{app:general_approx_and_bounds}
            \begin{proposition}[Exponential Approximation with Error Bound]\label{prop:exp_approx}
                For any integer $N \geq 2$ and $|x| < 1$:
                $$(1+x)^{N-1} = e^{Nx} \cdot e^{\eta(x)}$$
                where $|\eta(x)| \leq \frac{(N-1)x^2}{2(1-|x|)}+|x|$.
            \end{proposition}
            
            \begin{proof}\label{proof:exp_approx}
                We use the series expansion of the natural logarithm:
                \begin{equation}\ln(1+x) = \sum_{k=1}^{\infty} \frac{(-1)^{k+1}x^k}{k} = x - \frac{x^2}{2} + \frac{x^3}{3} - \frac{x^4}{4} + \cdots \quad \forall |x| < 1
                \end{equation}
                
                For the first-order approximation $\ln(1+x) = x + \varepsilon_1(x)$, the truncation error is:
                $$\varepsilon_1(x) = \sum_{k=2}^{\infty} \frac{(-1)^{k+1}x^k}{k}$$
                
                We bound this error as follows:
                \begin{equation}
                |\varepsilon_1(x)| = \left|\sum_{k=2}^{\infty} \frac{(-1)^{k+1}x^k}{k}\right| \leq \sum_{k=2}^{\infty} \frac{|x|^k}{k} \leq \frac{1}{2}\sum_{k=2}^{\infty} |x|^k = \frac{1}{2} \cdot \frac{|x|^2}{1-|x|} = \frac{x^2}{2(1-|x|)}
                \end{equation}
                
                Therefore:
                \begin{equation}
                (1+x)^{N-1} = e^{(N-1)\ln(1+x)} = e^{(N-1)(x + \varepsilon_1(x))} = e^{Nx} \cdot e^{(N-1)\varepsilon_1(x)-x}=e^{Nx} \cdot e^{\eta(x)}
                \end{equation}
                
                Setting $\eta(x) = (N-1)\varepsilon_1(x)-x$, the error bound follows from the triangle inequality:
                \begin{equation}
                |\eta(x)| \leq (N-1)|\varepsilon_1(x)|+|x| \leq \frac{(N-1)x^2}{2(1-|x|)}+|x|
                \end{equation}
    
            \end{proof}

        \subsection{Properties of Variable Coefficient Dynamics}\label{app:theory_variable_coef_dynamics}

            \begin{definition}[One-Step Affine Recurrence Relations]\label{def:one_step_affine_recurrence}
                Consider the recurrence relation with \textbf{one-step affine transformations}:
                \begin{equation}
                    x[k+1] = a(x[k]) \cdot x[k] + b(x[k])
                \end{equation}
                where $a: \mathbb{R} \to \mathbb{R}$ and $b: \mathbb{R} \to \mathbb{R}$ are given functions.
                
                At each iteration, the map $x \mapsto a(x) \cdot x + b(x)$ applies an affine transformation with state-dependent coefficients.
                
                The \textbf{instantaneous fixed point} at state $x$ is:
                \begin{equation}
                    \bar{x}(x) := \frac{b(x)}{1-a(x)} \quad \text{(when } a(x) \neq 1 \text{)}
                \end{equation}

                Similarly, since $x[k]$ is a variable whose value is dependant on the step $k$, we can also note the instantaneous fixed point at step $k$ using the following shorthand notation:
                \begin{equation}
                    \bar{x}[k] := \frac{b[k]}{1-a[k]}  \quad \text{(when } a[k] \neq 1 \text{)}
                \end{equation}
                
                \textbf{Special case}: When $a(x) = a$ and $b(x) = b$ are constants, this reduces to the classical affine recurrence relation.
            \end{definition}

            \begin{figure}[htbp]
                \centering
                \begin{minipage}{0.48\textwidth}
                    \centering
                    \begin{tikzpicture}
                        \begin{axis}[
                            axis lines=middle,
                            xlabel={$x$},
                            ylabel={$y$},
                            xmin=-0.45, xmax=0.45, ymin=-0.45, ymax=0.45,
                            legend pos=north east,
                            legend style={font=\small},
                            width=\textwidth,
                            height=\textwidth,
                        ]
            
                            \addplot[thick, blue]{0.4*x-0.1};
                            \addplot[thick, black, dashed]{x};
            
                            \draw[blue, thin] (axis cs:0.4,0.4) -- (axis cs:+0.4,0.06);
                            \draw[blue, thin] (axis cs:0.4,0.06) -- (axis cs:+0.06,0.06);
                            \draw[blue, thin] (axis cs:0.06,0.06) -- (axis cs:+0.06,-0.076);
                            \draw[blue, thin] (axis cs:0.06,-0.076) -- (axis cs:-0.076,-0.076);
                            \draw[blue, thin] (axis cs:-0.076,-0.076) -- (axis cs:-0.076,-0.1304);
                            \draw[blue, thin] (axis cs:-0.076,-0.1304) -- (axis cs:-0.1304,-0.1304);
                            \draw[blue, thin] (axis cs:-0.1304,-0.1304) -- (axis cs:-0.1304,-0.15216);
            
                            \draw[red, dashed] (axis cs:0,-0.1666666666) -- (axis cs:-0.1666666666,-0.1666666666);
                            \draw[red, dashed] (axis cs:-0.1666666666,0) -- (axis cs:-0.1666666666,-0.1666666666);
                            
                            \addplot[orange, only marks, mark=*, mark size=1.5pt] coordinates {(-0.1666666666, -0.1666666666)};
                            \node[orange, below, font=\footnotesize] at (axis cs:-0.1666666666,-0.1666666666) {$\frac{b}{1-a}$};
            
                            \addlegendentry{$y=0.4x-0.1$}
                            \addlegendentry{$y=x$}
                        \end{axis}
                    \end{tikzpicture}
                    
                    \textbf{(a)}Demonstrates the monotonic convergence of Proposition~\ref{prop:constant_monotonic} with constant affine transformation coefficient when $a = 0.4 > 0$
                \end{minipage}
                \hfill
                \begin{minipage}{0.48\textwidth}
                    \centering
                    \begin{tikzpicture}
                        \begin{axis}[
                            axis lines=middle,
                            xlabel={$x$},
                            ylabel={$y$},
                            xmin=-0.45, xmax=0.45, ymin=-0.45, ymax=0.45,
                            legend pos=north east,
                            legend style={font=\small},
                            width=\textwidth,
                            height=\textwidth,
                        ]
            
                            \addplot[thick, blue]{-0.4*x-0.1};
                            \addplot[thick, black, dashed]{x};
            
                            \draw[blue, thin] (axis cs:0.3,0.3) -- (axis cs:+0.3,-0.22);
                            \draw[blue, thin] (axis cs:0.3,-0.22) -- (axis cs:-0.22,-0.22);
                            \draw[blue, thin] (axis cs:-0.22,-0.22) -- (axis cs:-0.22,-0.012);
                            \draw[blue, thin] (axis cs:-0.22,-0.012) -- (axis cs:-0.012,-0.012);
                            \draw[blue, thin] (axis cs:-0.012,-0.012) -- (axis cs:-0.012,-0.0952);
                            \draw[blue, thin] (axis cs:-0.012,-0.0952) -- (axis cs:-0.0952,-0.0952);
                            \draw[blue, thin] (axis cs:-0.0952,-0.0952) -- (axis cs:-0.0952,-0.06192);
            
                            \draw[red, dashed] (axis cs:0,-0.0714) -- (axis cs:-0.0714,-0.0714);
                            \draw[red, dashed] (axis cs:-0.0714,0) -- (axis cs:-0.0714,-0.0714);
                            
                            \addplot[orange, only marks, mark=*, mark size=1.5pt] coordinates {(-0.0714, -0.0714)};
                            \node[orange, below, font=\footnotesize] at (axis cs:-0.0714,-0.0714) {$\frac{b}{1-a}$};
            
                            \addlegendentry{$y=-0.4x-0.1$}
                            \addlegendentry{$y=x$}
                        \end{axis}
                    \end{tikzpicture}
                    
                    \textbf{(b)} Shows the oscillatory but convergent behavior of Proposition~\ref{prop:constant_oscillatory} with constant affine transformation coefficient when  $a = -0.4 < 0$.
                \end{minipage}
                \caption{Graphical interpretation of the convergence behavior. The blue line represents the update function $y = ax + b$, while the black dashed line $y = x$ helps visualize the fixed point at their intersection. The cobweb plots show the iterative convergence process}
                \label{fig:convergence_comparison}
            \end{figure}

            \begin{proposition}[Constant Coefficient Case - Monotonic Convergence]\label{prop:constant_monotonic}
                For constant coefficients $a, b \in \mathbb{R}$ with $0 < a < 1$, the recurrence $x[k+1] = ax[k] + b$ has unique fixed point $\bar{x} = \frac{b}{1-a}$ and:
                \begin{enumerate}
                    \item \textbf{Exponential convergence}: $|x[k] - \bar{x}| = a^k |x[0] - \bar{x}|$
                    \item \textbf{Monotonic convergence}: 
                        \begin{itemize}
                            \item If $x[0] < \bar{x}$: $x[0] < x[1] < x[2] < \cdots < x[k] < \bar{x}$ for all $k$, and $\lim_{k\to+\infty}x[k]=\bar{x}$
                            \item If $x[0] > \bar{x}$: $x[0] > x[1] > x[2] > \cdots > x[k] > \bar{x}$ for all $k$, and $\lim_{k\to+\infty}x[k]=\bar{x}$
                            \item If $x[0] = \bar{x}$: $x[k] = \bar{x}$ for all $k \geq 0$
                        \end{itemize}
                \end{enumerate}
            \end{proposition}
    
            \begin{proof}\label{proof:constant_monotonic}
                The fixed point satisfies $\bar{x} = a\bar{x} + b$, giving $\bar{x} = \frac{b}{1-a}$ (well-defined since $a \neq 1$).
                
                Defining the deviation $e[k] = x[k] - \bar{x}$:
                \begin{align}
                    e[k+1] &= x[k+1] - \bar{x} = ax[k] + b - \bar{x} \\
                    &= a(x[k] - \bar{x}) = ae[k]
                \end{align}
                
                Therefore $e[k] = a^k e[0]$, proving part the exponential convergence property.
                
                Since $0 < a < 1$, we have $a^k > 0$ for all $k$, so $\text{sign}(e[k]) = \text{sign}(e[0])$. Also, we have $x[k+1] - x[k] = (a-1)x[k] + b = (1-a)(\bar{x} - x[k])$.
                \begin{itemize}
                    \item If $x[0] < \bar{x}$: then $x[k] < \bar{x}$ for all $k$, and $x[k+1] - x[k] = (1-a)(\bar{x} - x[k]) > 0$
                    \item If $x[0] > \bar{x}$: then $x[k] > \bar{x}$ for all $k$, and $x[k+1] - x[k] = (1-a)(\bar{x} - x[k]) < 0$
                \end{itemize}
                Thus the sequence is monotonic and bounded, staying in the corridor between $x[0]$ and $\bar{x}$, thus proving monotonic convergence (i.e. without overshoot) property. See Figure~\ref{fig:convergence_comparison}.a.
    
            \end{proof}
            
            \begin{proposition}[Constant Coefficient Case - Oscillatory Convergence]\label{prop:constant_oscillatory}
                For constant coefficients $a, b \in \mathbb{R}$ with $-1 < a < 0$, the recurrence $x[k+1] = ax[k] + b$ has unique fixed point $\bar{x} = \frac{b}{1-a}$ and:
                \begin{enumerate}
                    \item \textbf{Exponential convergence}: $|x[k] - \bar{x}| = |a|^k |x[0] - \bar{x}|$
                    \item \textbf{Oscillatory convergence}: 
                        \begin{itemize}
                            \item If $x[0] \neq \bar{x}$: $\text{sign}(x[k] - \bar{x}) = (-1)^k \text{sign}(x[0] - \bar{x})$ for all $k$, and $\lim_{k\to+\infty}x[k]=\bar{x}$
                            \item If $x[0] = \bar{x}$: $x[k] = \bar{x}$ for all $k \geq 0$
                        \end{itemize}
                \end{enumerate}
            \end{proposition}
            
            \begin{proof}\label{proof:constant_oscillatory}
                The fixed point satisfies $\bar{x} = a\bar{x} + b$, giving $\bar{x} = \frac{b}{1-a}$ (well-defined since $a \neq 1$).
                
                Defining the deviation $e[k] = x[k] - \bar{x}$:
                \begin{align}
                    e[k+1] &= x[k+1] - \bar{x} = ax[k] + b - \bar{x} \\
                    &= a(x[k] - \bar{x}) = ae[k]
                \end{align}
                
                Therefore $e[k] = a^k e[0]$, proving exponential convergence since $|e[k]| = |a|^k |e[0]|$.
                
                Since $-1 < a < 0$, we have $a^k = |a|^k \cdot (-1)^k$, so:
                $$\text{sign}(e[k]) = \text{sign}(a^k e[0]) = (-1)^k \text{sign}(e[0])$$
                
                This proves the oscillatory convergence property. See Figure~\ref{fig:convergence_comparison}.b.
    
            \end{proof}

            \begin{theorem}[Complete Variable Coefficient Dynamics]\label{thm:complete_variable_dynamics}
                Consider the recurrence $x[k+1] = a(x[k]) \cdot x[k] + b(x[k])$ where $a, b: \mathbb{R} \to \mathbb{R}$ satisfy $0 < a(x) < 1$ for all $x \in \mathbb{R}$.
                
                Define the \textbf{instantaneous fixed point function} $\bar{x}(x) := \frac{b(x)}{1-a(x)}$ and the \textbf{fixed point envelope}:
                \begin{align}
                    \bar{x}_{\min} &:= \inf_{x \in \mathbb{R}} \bar{x}(x) = \inf_{x \in \mathbb{R}} \frac{b(x)}{1-a(x)} \\
                    \bar{x}_{\max} &:= \sup_{x \in \mathbb{R}} \bar{x}(x) = \sup_{x \in \mathbb{R}} \frac{b(x)}{1-a(x)}
                \end{align}
                
                For any $x \in \mathbb{R}$, define the distance to the envelope:
                $$d(x) := \text{dist}(x, [\bar{x}_{\min}, \bar{x}_{\max}]) = \begin{cases}
                x - \bar{x}_{\max} & \text{if } x > \bar{x}_{\max} \\
                \bar{x}_{\min} - x & \text{if } x < \bar{x}_{\min} \\
                0 & \text{if } x \in [\bar{x}_{\min}, \bar{x}_{\max}]
                \end{cases}$$
                
                Then the following fundamental properties hold:
        
                \begin{enumerate}
                    \item \textbf{Convex Interpolation}: Each step interpolates between current state and instantaneous fixed point:
                    $$x[k+1] = a(x[k]) \cdot x[k] + (1-a(x[k])) \cdot \bar{x}(x[k])$$
                    This ensures $x[k+1] \in [\min\{x[k], \bar{x}(x[k])\}, \max\{x[k], \bar{x}(x[k])\}]$, preventing overshooting beyond the instantaneous target.
                    
                    \item \textbf{Envelope Contraction}: Distance to the envelope contracts at each step:
                    $$d(x[k+1]) \leq a(x[k]) \cdot d(x[k])$$
                    where $d(x) = \text{dist}(x, [\bar{x}_{\min}, \bar{x}_{\max}])$. This implies forward invariance of the envelope:
                    $$x[k] \in [\bar{x}_{\min}, \bar{x}_{\max}] \Rightarrow x[k+1] \in [\bar{x}_{\min}, \bar{x}_{\max}]$$
                    and progressive attraction for trajectories starting outside.
        
                    \item \textbf{Exponential Convergence Rate}: If additionally $\alpha := \sup_{x \in \mathbb{R}} a(x) < 1$, then the distance to the envelope diminishes exponentially: 
                        $$d(x[k]) \leq \alpha^k d(x[0])$$
        
                \end{enumerate}
            \end{theorem}
    
            \begin{proof}\label{proof:complete_variable_dynamics}
                \textbf{Property 1 - Convex Interpolation}: Direct algebraic manipulation:
                \begin{align}
                    x[k+1] &= a(x[k]) \cdot x[k] + b(x[k]) \\
                    &= a(x[k]) \cdot x[k] + (1-a(x[k])) \cdot \frac{b(x[k])}{1-a(x[k])} \\
                    &= a(x[k]) \cdot x[k] + (1-a(x[k])) \cdot \bar{x}(x[k])
                \end{align}
                
                Since $0 < a(x[k]) < 1$, this is a convex combination of $x[k]$ and $\bar{x}(x[k])$, hence:
                $$x[k+1] \in [\min\{x[k], \bar{x}(x[k])\}, \max\{x[k], \bar{x}(x[k])\}]$$
                
                \textbf{Property 2 - Envelope Contraction}: We prove $d(x[k+1]) \leq a(x[k]) \cdot d(x[k])$ by cases.
                
                \textbf{Case 1}: $x[k] > \bar{x}_{\max}$ (outside envelope, above)
                
                Here $d(x[k]) = x[k] - \bar{x}_{\max}$ and $\bar{x}(x[k]) \leq \bar{x}_{\max}$ by definition.
                
                Since $x[k] \geq \bar{x}_{\min}$ and $\bar{x}(x[k]) \geq \bar{x}_{\min}$, we have $x[k+1] \geq \bar{x}_{\min}$ by Property 1 (no overshoot).
                
                \textit{Subcase 1a}: If $x[k+1] > \bar{x}_{\max}$:
                \begin{align}
                    d(x[k+1]) = x[k+1] - \bar{x}_{\max} &= a(x[k]) \cdot x[k] + (1-a(x[k])) \cdot \bar{x}(x[k]) - \bar{x}_{\max} \\
                    &= a(x[k])(x[k] - \bar{x}_{\max}) + (1-a(x[k]))(\bar{x}(x[k]) - \bar{x}_{\max}) \\
                    &\leq a(x[k])(x[k] - \bar{x}_{\max}) = a(x[k]) \cdot d(x[k])
                \end{align}
                where the inequality uses $\bar{x}(x[k]) - \bar{x}_{\max} \leq 0$ and $(1-a(x[k])) > 0$.

                \textit{Subcase 1b}: If $x[k+1] \in [\bar{x}_{\min}, \bar{x}_{\max}]$: Then $d(x[k+1]) = 0 < a(x[k]) \cdot d(x[k])$.
                
                \textbf{Case 2}: $x[k] < \bar{x}_{\min}$ (symmetric argument gives same result)
                
                \textbf{Case 3}: $x[k] \in [\bar{x}_{\min}, \bar{x}_{\max}]$
                
                Here $d(x[k]) = 0$ and $\bar{x}(x[k]) \in [\bar{x}_{\min}, \bar{x}_{\max}]$ by definition. Since $x[k+1]$ is a convex combination of two points in the envelope, $x[k+1] \in [\bar{x}_{\min}, \bar{x}_{\max}]$ (see Property 1), so $d(x[k+1]) = 0 = a(x[k]) \cdot d(x[k])$.
                
                Therefore $d(x[k+1]) \leq a(x[k]) \cdot d(x[k])$ in all cases, which immediately implies forward invariance of the envelope and progressive attraction for trajectories starting outside. 
                
                \textbf{Property 3 - Exponential Convergence Rate}: 
                From Property 2, we have $d(x[k+1]) \leq a(x[k]) \cdot d(x[k])$ for all $k$. 
                
                If $\alpha := \sup_{x \in \mathbb{R}} a(x) < 1$, then $a(x[k]) \leq \alpha$ for all $k$, so:
                \begin{align}
                d(x[1]) &\leq a(x[0]) \cdot d(x[0]) \leq \alpha \cdot d(x[0]) \\
                d(x[2]) &\leq a(x[1]) \cdot d(x[1]) \leq \alpha \cdot d(x[1]) \leq \alpha^2 \cdot d(x[0]) \\
                &\vdots \\
                d(x[k]) &\leq \alpha^k \cdot d(x[0])
                \end{align}
                
                By induction: if $d(x[j]) \leq \alpha^j d(x[0])$ for some $j$, then:
                $$d(x[j+1]) \leq a(x[j]) \cdot d(x[j]) \leq \alpha \cdot \alpha^j d(x[0]) = \alpha^{j+1} d(x[0])$$
                
                Since $0 < \alpha < 1$, we have $\alpha^k \to 0$ as $k \to \infty$, ensuring exponential convergence to the envelope.
    
            \end{proof}

            \begin{proposition}[Intersection Bounds for Monotonic Functions]\label{prop:intersection_bounds}
                Let $f: [a,b] \to \mathbb{R}$ be a decreasing function and $g(x) = x$ be the identity function. If $f$ intersects $g$ at some point $\phi \in [a,b]$ (i.e., $f(\phi) = \phi$), then:
                $$f(b) \leq \phi \leq f(a)$$
                with equality conditions:
                \begin{itemize}
                    \item $\phi = f(a)$ if and only if $f$ is constant on $[a,\phi]$
                    \item $\phi = f(b)$ if and only if $f$ is constant on $[\phi,b]$
                \end{itemize}
            \end{proposition}

            \begin{proof}\label{proof:intersection_bounds}   
                \textbf{Upper bound:} Since $\phi \in [a,b]$, we have $\phi \geq a$. Since $f$ is decreasing:
                $$f(\phi) \leq f(a)$$
                Since $f(\phi) = \phi$, this gives us $\phi \leq f(a)$.
                
                \textbf{Lower bound:} Since $\phi \in [a,b]$, we have $\phi \leq b$. Since $f$ is decreasing:
                $$f(\phi) \geq f(b)$$
                Since $f(\phi) = \phi$, this gives us $\phi \geq f(b)$.
                
                \textbf{Equality conditions:}
                \begin{itemize}
                    \item For $\phi = f(a)$: If $f$ is constant on $[a,\phi]$, then $f(\phi) = f(a) = \phi$. Conversely, if $\phi = f(a)$, then for any $x \in [a,\phi]$, the decreasing property gives $f(x) \geq f(\phi) = f(a)$, but since $f(a)$ is the maximum, we have $f(x) = f(a)$.
                    
                    \item For $\phi = f(b)$: If $f$ is constant on $[\phi,b]$, then $f(\phi) = f(b) = \phi$. Conversely, if $\phi = f(b)$, then for any $x \in [\phi,b]$, the decreasing property gives $f(x) \leq f(\phi) = f(b)$, but since $f(b)$ is the minimum, we have $f(x) = f(b)$.
                \end{itemize}
    
                \begin{figure}[htbp]
                    \centering
                    \begin{tikzpicture}
                        \begin{axis}[
                            axis lines=middle,
                            xlabel={$x$},
                            ylabel={$y$},
                            xmin=-0.15, xmax=1.3, ymin=-0.15, ymax=1.3,
                            legend pos=north east,
                            legend style={font=\small},
                            width=10cm,
                            height=7cm,
                            grid=major,
                            grid style={gray!30},
                        ]
                
                            \addplot[thick, black, dashed, domain=0:1] {x};
                            
                            \addplot[thick, blue, domain=0:1, samples=100] {0.2+0.6*(1-x)^(1/3)};   
        
                            \addplot[thick, blue, domain=0:1, samples=100] {0.2+0.6*(1-x)^(1)};     
        
                            \addplot[thick, blue, domain=0:1, samples=100] {0.2+0.6*(1-x)^(3)};     
                            q
                            \addplot[thick, red, domain=0:1] {0.8};
                            \addplot[thick, red, domain=0:1] {0.2};
                            
                            \addplot[only marks, mark=*, mark size=2pt, red] coordinates {(0.2, 0.2)};
                            \addplot[only marks, mark=*, mark size=2pt, blue] coordinates {(0.36, 0.36)};
                            \addplot[only marks, mark=*, mark size=2pt, blue] coordinates {(0.50, 0.50)};
                            \addplot[only marks, mark=*, mark size=2pt, blue] coordinates {(0.63, 0.63)};
                            \addplot[only marks, mark=*, mark size=2pt, red] coordinates {(0.8, 0.8)};
                            
                            \draw[red, dashed] (axis cs:0.2,0.8) -- (axis cs:0.2,0.0);
                            \draw[red, dashed] (axis cs:0.8,0.8) -- (axis cs:0.8,0.0);

                            \node[blue, below] at (axis cs:0.36,0.36) {$\phi_0$};
                            \node[blue, below] at (axis cs:0.5,0.5) {$\phi_1$};
                            \node[blue, below] at (axis cs:0.63,0.63) {$\phi_2$};
        
                            \node[red, below right] at (axis cs:0.2,0.2) {$\phi_{\min}$};
                            \node[red, above left] at (axis cs:0.8,0.8) {$\phi_{\max}$};
                            
                            \node[blue, left] at (axis cs:0,0.8) {$f(a)$};
                            \node[blue, left] at (axis cs:0,0.2) {$f(b)$};
                            \addplot[only marks, mark=*, mark size=2pt, blue] coordinates {(1.0, 0.2)};
                            \addplot[only marks, mark=*, mark size=2pt, blue] coordinates {(0.0, 0.8)};
                
                            \addlegendentry{$y = x$}
                            \addlegendentry{$f_0(x)$}
                            \addlegendentry{$f_1(x)$}
                            \addlegendentry{$f_2(x)$}
                            \addlegendentry{$g_{\max}(x)$}
                            \addlegendentry{$g_{\min}(x)$}
                        \end{axis}
                    \end{tikzpicture}
                    \caption{Intersection bounds for decreasing functions on $[a,b] = [0,1]$. Three decreasing functions $f_0(x)$, $f_1(x)$, and $f_2(x)$ (blue curves) with different curvatures all satisfy $f(0) = 0.8$ and $f(1) = 0.2$, intersecting the identity line $y = x$ at points $\phi_0 \approx 0.36$, $\phi_1 = 0.5$, and $\phi_2 \approx 0.63$ respectively. The constant functions $g_{\max}(x) = f(0) = 0.8$ and $g_{\min}(x) = f(1) = 0.2$ (red lines) provide conservative bounds, intersecting $y = x$ at $\phi_{\max} = 0.8$ and $\phi_{\min} = 0.2$. This illustrates that for any decreasing function $f$ on $[a,b]$, the intersection point $\phi$ (if any) satisfies $f(b) \leq \phi \leq f(a)$, i.e., $0.2 \leq \phi \leq 0.8$, with the bounds achieved by constant functions representing the extreme cases of the decreasing property.}
        
                    \label{fig:intersection_bounds}
                \end{figure}
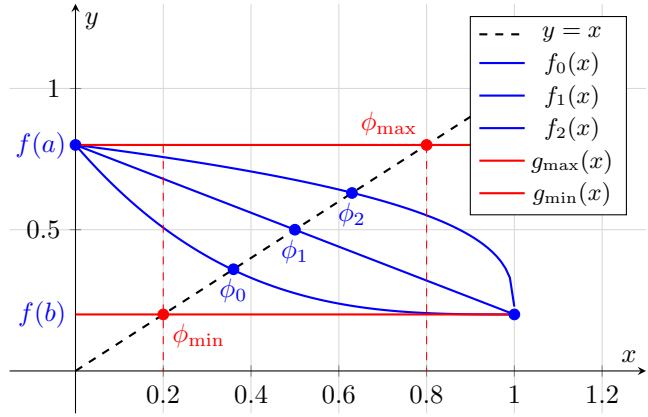
    
            \end{proof}

            \begin{remark}[Domain-Constrained Bounds]\label{rem:domain_bounds}
                The intersection bounds are further constrained by the domain $[a,b]$:
                $$\max(f(b), a) \leq \phi \leq \min(f(a), b)$$
                since $\phi \in [a,b]$ by definition. This refinement is particularly important when $f(b) < a$ or $f(a) > b$.
            \end{remark}
            
            \begin{corollary}[Intersection Existence Conditions]\label{cor:existence_conditions}
                An intersection $\phi \in [a,b]$ is guaranteed to exists when :
                $$ a \leq f(a) \leq b \quad \textrm{or} \quad a \leq f(b) \leq b$$
                
                On the contrary, no intersection can exist on $[a,b]$ when:
                $$f(a) < a \quad \textrm{or} \quad f(b) > b$$
            \end{corollary}

            \begin{corollary}[Conservative Intersection Bounds]\label{cor:conservative_bounds}
                For a decreasing function $f$ on $[a,b]$ intersecting $y = x$ at point $\phi$:
                \begin{itemize}
                    \item The constant function $g_{\text{max}}(x) = f(a)$ provides the most conservative upper bound, intersecting $y = x$ at $x = f(a)$
                    \item The constant function $g_{\text{min}}(x) = f(b)$ provides the most conservative lower bound, intersecting $y = x$ at $x = f(b)$
                    \item Any decreasing function's intersection lies in the interval $[\phi_{\min}, \phi_{\max}] = [f(b), f(a)]$
                \end{itemize}
            \end{corollary}

    \label{app:end}


\end{document}